\documentclass[nonacm]{acmart}

\setcopyright{acmlicensed}
\copyrightyear{2024}
\acmYear{2024}
\acmDOI{XXXXXXX.XXXXXXX}

\acmConference[KDD '24]{30th ACM
SIGKDD Conference on Knowledge Discovery and Data Mining}{August 25--29,
  2024}{Barcelona, Spain}
\acmISBN{978-1-4503-XXXX-X/18/06}

\usepackage{microtype}
\usepackage{graphicx}
\usepackage{subcaption}
\usepackage{booktabs} 
\usepackage{multirow}
\usepackage{makecell}

\usepackage{amsmath,bm}

\usepackage{hyperref}
\usepackage{nicefrac}

\usepackage{amsmath}
\usepackage{mathtools}
\usepackage{amsthm}

\usepackage[capitalize,noabbrev]{cleveref}

\theoremstyle{plain}
\newtheorem{theorem}{Theorem}[section]

\newtheorem{lemma}[theorem]{Lemma}

\theoremstyle{definition}
\newtheorem{definition}[theorem]{Definition}
\newtheorem{assumption}[theorem]{Assumption}

\crefname{challenge}{Challenge}{Challenges}
\theoremstyle{remark}
\newtheorem{remark}[theorem]{Remark}

\usepackage[textsize=tiny]{todonotes}
\usepackage[inline]{enumitem}

\DeclareMathOperator*{\argmin}{arg\,min}

\usepackage{siunitx}

\newcommand{\vcd}{\textnormal{VC}}

\newcommand{\var}{\textnormal{Var}}

\crefname{assumption}{Assumption}{Assumptions}

\usepackage{placeins}

\usepackage{thmtools}
\usepackage{thm-restate}

\allowdisplaybreaks

\begin{document}
\title{Evaluating the Effectiveness of Index-Based Treatment Allocation}

\author{Niclas Boehmer}
\authornote{These authors contributed equally.}
\affiliation{%
  \institution{Harvard University}
  \country{USA}}

\author{Yash Nair}
\authornotemark[1]
\affiliation{%
  \institution{Stanford University}
  \country{USA}
}

\author{Sanket Shah}
\authornotemark[1]
\affiliation{%
  \institution{Harvard University}
  \country{USA}
}

\author{Lucas Janson}
\affiliation{%
  \institution{Harvard University}
  \country{USA}
}

\author{Aparna Taneja}
\affiliation{%
  \institution{Google Research India}
  \country{India}
}

\author{Milind Tambe}
\affiliation{%
  \institution{Harvard University}
  \institution{Google Research}
  \country{USA}
}

\renewcommand{\shortauthors}{Boehmer, Nair, and Shah et al.}

\begin{abstract}
When resources are scarce, an allocation policy is needed to decide who receives a resource. 
This problem occurs, for instance, when allocating scarce medical resources and is often solved using modern ML methods. 
This paper introduces methods to evaluate index-based allocation policies---that allocate a fixed number of resources to those who need them the most---by using data from a randomized control trial. Such policies create dependencies between agents, which render the assumptions behind standard statistical tests invalid and limit the effectiveness of estimators. Addressing these challenges, we translate and extend recent ideas from the statistics literature to present an efficient estimator and methods for computing asymptotically correct confidence intervals. This enables us to effectively draw valid statistical conclusions, a critical gap in previous work. Our extensive experiments validate our methodology in practical settings, while also showcasing its statistical power. We conclude by proposing and empirically verifying extensions of our methodology that enable us to reevaluate a past randomized control trial to evaluate different ML allocation policies in the context of a mHealth program, drawing previously invisible conclusions.
\end{abstract}



\maketitle

\section{Introduction}\label{introduction}
In treatment allocation, we have a limited number of intervention resources. The challenge that arises is to devise an allocation policy that decides to whom we allocate them to maximize social welfare. 
Treatment allocation 
finds applications in various scenarios, for instance, when 
\begin{enumerate*}[label=(\roman*)]
  \item  allocating scarce medical resources such as medication \cite{ayer2019prioritizing,deo2013improving} or screening tools \cite{lasry2011model,deo2015planning,bastani2021efficient,lee2019optimal} 
  \item  scheduling maintenance or inspection visits \cite{yeter2020risk,gerum2019data,luque2019risk}, or 
  \item  allocating spots in support programs \cite{mac2019efficacy,mate2022field,verma2023restless}.
\end{enumerate*}
Accordingly, treatment allocation is a common problem in statistics and economics \cite{bhattacharya2012inferring,kitagawa2018should}. 
More recently, ML methods have started to be increasingly used for treatment allocation and thus the problem has gained popularity in the ML community \cite{DBLP:conf/kdd/KillianBST21,killian2019learning,kunzel2019metalearners,bastani2021efficient,DBLP:conf/kdd/FoucheKB19,DBLP:conf/nips/MateKXPT20,DBLP:journals/jmlr/ZhaoLNSS19}. 
In resource-limited settings, allocations are commonly made based on individualized measures of risk \cite{kent2016risk,mac2019efficacy,perdomo2023difficult} or treatment effects \cite{wager2018estimation,kunzel2019metalearners,DBLP:conf/ijcai/DanassisVKTT23,verma2023expanding}, often predicted using modern ML techniques.
Both of these and many other strategies can be captured by so-called
index-based allocation policies. 
Given a fixed number of resources, these policies first compute an index (e.g., risk) for each individual, and subsequently allocate the resources to the individuals with the lowest index.
We present and evaluate methods for causal inference for the effectiveness of index-based allocation policies using randomized control trials (RCTs), the gold standard for analyzing treatment~effects~\cite{hariton2018randomised}. 

While our work is generally applicable and relevant to a wide range of application domains, it is motivated in particular by a deployed index-based allocation policy in a mobile health program organized by the Indian NGO ARMMAN 
\cite{mate2022field,verma2023restless,verma2023expanding,DBLP:conf/kdd/Tambe22}.
ARMMAN's mMitra program provides critical preventive care information to enrolled pregnant women and mothers of infants through automated voice messages. 
To promote engagement, each week, a limited number of beneficiaries can be called by health workers to provide them with additional information and guidance.

\citet{mate2022field} and \citet{verma2023restless} conducted RCTs to evaluate the effectiveness of index-based allocation policies to allocate live service calls in mMitra based on different ML paradigms. 
Evaluating these trials turns out to be a significant research challenge. This is because whether or not an individual gets selected for treatment by a policy depends on the other beneficiaries in the population.
The resulting dependence between beneficiaries renders central assumptions behind standard statistical tests invalid and leads to the low statistical power of estimators (see \Cref{challenge}).
\citet{mate2022field} and \citet{verma2023restless} note that their methodology (to which we refer as the base estimator; see \Cref{sec:PA}) comes without rigorous empirical evidence or theoretical guarantees on the validity of computed confidence intervals and drawn statistical conclusions.
Addressing this gap, \emph{we are the first to provide the necessary tools to effectively draw reliable statistical conclusions about the quality of index-based allocation policies by describing a new estimator together with customized statistical inference techniques}.

In more detail, the contributions of the paper are as follows. 
In \textbf{\Cref{sec:prev}}, we describe the methodology used by \citet{mate2022field} and \citet{verma2023restless}.
Moreover, we describe recent work by \citet{imai2023statistical} from the statistics literature, on top of which many of our ideas and results are built.
In \textbf{\Cref{app:estimation}}, translating ideas from \citet{imai2023statistical}, we present the \emph{subgroup estimator}, which computes the average treatment effect by comparing those who are selected by the policy in the policy arm of the RCT to those the policy would have selected in the control arm of the RCT.
In \textbf{\Cref{hy:Sub}}, using results from \citet{imai2023statistical}, we conclude the asymptotic normality of the subgroup estimator and describe new methods for computing asymptotically valid confidence intervals for evaluating and comparing policies. We also argue why standard tests still produce good results for the subgroup estimator.
In \textbf{\Cref{hy:Base}}, we establish the asymptotic normality of the base estimator that allows us to compute asymptotically valid confidence intervals using a new proof strategy via empirical process theory \cite{vaart2023empirical}. 

In our experimental \textbf{\Cref{app:experiment}}, we use synthetic and real-world data to build various simulators that simulate an individual's behavior as a Markov Decision Process. 
We successfully verify that our asymptotic theoretical guarantees regarding the validity of confidence intervals for our estimators empirically extend to a variety of practical cases.
Moreover, we demonstrate that the subgroup estimator has typically a significantly higher statistical power than the base estimator, as we find that computed confidence intervals are usually more than halved.
In fact, the difference between the two can be even more pronounced, for instance when the budget is very small the difference gets as large as a factor of $8$.
This finding highlights that our methodology allows for a more flexible  study design: 
As the base estimator has a very low statistical power in case only a few treatments are allocated, \citet{verma2023restless} distributed many resources in their field trial, a strategy that is both expensive (and sometimes infeasible) and proves very challenging in the evaluation stage as it reduces the observed average treatment effect. 
These are problems that largely disappear when using the new subgroup estimator.

Lastly, in \textbf{\Cref{sec:rs}}, we turn to the field trial conducted by \citet{verma2023restless}. For this, we need to extend our methodology, e.g., accounting for the sequential allocation of resources and covariate correction, beyond the case covered by our theoretical guarantees from \Cref{sec:methods}. We empirically verify the validity of computed confidence intervals and reevaluate the field trial conducted by \citet{verma2023restless}. 
We confirm previous conclusions obtained using methods whose reliability was unclear to the authors \cite{verma2023restless}. In addition, we also identify previously hidden insights by making use of the increased flexibility and statistical power of the subgroup estimator.

\section{Preliminaries}\label{prelim:SS}

Let $A=\{0,1\}$ be the set of actions where $1$ is the active action (treatment\footnote{We use the terms \emph{treatment} and \emph{intervention} interchangeably.} given) and $0$ is the passive action (no treatment given).
An agent $i\in [n]$ is characterized by covariates $\mathbf{x}_i\in \mathcal{X}$  and a reward function $R_i:A\to \mathbb{R}$ that returns the reward generated by the agent given the action assigned to it.\footnote{This is equivalent to the Nayman-Rubin potential outcomes model \cite{imbens2015causal}.} 
Agents are drawn i.i.d. from a probability distribution $P$ defined over the space of covariates and reward functions $\mathcal{X}\times (A\to \mathbb{R})$.
We write $(\mathbf{x}_{i,n},R_{i,n})\sim P$ to denote a set $[n]$ of agents being sampled i.i.d. from the probability distribution $P$ and $\mathbf{X}_n:=(\mathbf{x}_{i,n})_{i\in [n]}$ to denote the covariates of these $n$ agents.
If not stated otherwise, expectation and probabilities in this paper are over groups of $n$ agents, i.e., $(\mathbf{x}_{i,n},R_{i,n})\sim P$.

An allocation policy $\pi$ gets as input the covariates $\mathbf{X}_n\in \mathcal{X}^n$ of $n$ agents and a treatment fraction $\alpha$ and returns $\lceil\alpha n\rceil$ agents to which the active action is applied.\footnote{As $n$ is typically fixed, we could alternatively also specify the \emph{number} of agents receiving a treatment. Both formulations are equivalent, but the fraction formulation will prove advantageous in the presentation of our theoretical analysis in \Cref{sec:methods}.}
We denote as $J^{\pi(\mathbf{X}_n,\alpha)}_i$ the indicator variable that denotes whether agent $i\in [n]$  gets assigned a treatment as per policy $\pi$, i.e., $J^{\pi(\mathbf{X}_n,\alpha)}_i=1$ if $i\in \pi\left(\mathbf{X}_n,\alpha\right)$ and $0$ otherwise.
An index-based allocation policy $\pi^\Upsilon$ is defined by a function $\Upsilon:\mathcal{X}\to \mathbb{R}$ mapping covariates to an index. 
Given  $\mathbf{X}_n\in \mathcal{X}^n$ and a treatment fraction $\alpha \in [0,1]$, $\pi^\Upsilon$ returns the $\lceil\alpha n\rceil$ agents with the lowest index $\Upsilon(\mathbf{x}_i)$. 
Moreover, given $\mathbf{X}_n\in \mathcal{X}^n$ and a threshold $\lambda\in \mathbb{R}$, let $\upsilon^\Upsilon(\mathbf{X}_n,\lambda)$ return the set of agents $i\in [n]$ with an index value $\Upsilon(\mathbf{x}_i)$ smaller or equal to $\lambda$ (note that this policy does not satisfy the definition of an allocation policy, as the number of agents that receive an active action is not fixed). To highlight this difference, we refer to the policy that acts on everyone in $\upsilon^\Upsilon$ as a \emph{threshold policy}.

\paragraph{Statistics Notation}
We now introduce terminology necessary to formalize our methodology. 
An estimand is the quantity we want to measure and an estimator is a value ``approximating'' the estimand, computed from the available observed data using some procedure.
Estimands' names will always involve a $\tau$, while estimators' names will always involve a $\theta$. 
A sequence of random variables $(A_n)_{n>0}$ with cumulative distributions $\left(G_n(a)\right)_{n>0}$ \emph{converges in distribution} to a random variable $A$ with cumulative distribution $G$ if $\lim_{n\to \infty} G_n(a)=G(a)$ for all $a\in \mathbb{R}$ at which $G$ is continuous, in which case we write $A_n \overset{d}{\rightarrow} A$ (note that in the context of this paper, $n$ will typcially be the number of samples we observe).
A sequence $(A_n)_{n>0}$ \emph{converges in probability} to $A$ ($A_n \overset{p}{\rightarrow} A$) if $\lim_{n\to \infty} \mathbb{P}(|A_n-A|\geq \epsilon)=0$ for all $\epsilon>0$. 
An estimator $\theta_n$ of an estimand $\tau_n$ is (weakly) consistent if 
$\theta_n-\tau_n \overset{p}{\rightarrow} 0$. We denote as $\mathcal{N}(\mu,\sigma^2)$ the normal distribution with mean $\mu$ and variance $\sigma^2$. 
Let $q_\alpha$ be the $\alpha$-quantile of the cumulative distribution function $F_{\Upsilon}(\lambda)=\mathbb{P}_{(\mathbf{x},R)\sim P}[\Upsilon(\mathbf{x})\leq \lambda]$  of indices, i.e., the smallest number so that an expected $\alpha$-fraction of agents have an index below $q_{\alpha}$.

\section{Challenges and Previous  Work} \label{sec:prev}
We describe previous approaches to evaluating index-based allocation policies (\Cref{sec:PA}) and why they fall short to address the problem (\Cref{challenge}). In \Cref{sec:imaili},  we describe the work of \citet{imai2023statistical}, which we will refer to throughout the rest of the paper.

\subsection{Previous Approaches and RCT Design} \label{sec:PA}
Due to resource scarcity, our basic setup which has also been used in previous work \cite{mate2022field,verma2023restless,pmlr-v202-mate23a} assumes a modified RCT design, where treatment is allocated according to the evaluated allocation policy: We have access to the results of a randomized control trial with a policy arm (p) containing $n$ agents $(\mathbf{x}^p_i,R^p_i)_{i\in [n]}$ sampled i.i.d. from $P$ on which we run our policy $\pi$. As the outcome, we observe $(\mathbf{x}^p_i,R^p_i(J^p_i))_{i\in [n]}$, where $J^p_i:=J^{\pi(\mathbf{X}_n^p,\alpha)}_i$. 
Moreover, we have access to a control arm (c) of $n$ agents $(\mathbf{x}^c_i,R^c_i)_{i\in [n]}$ sampled i.i.d. from $P$ for which we observe $(\mathbf{x}^c_i,R^c_i(0))_{i\in [n]}$. 
Note that for both the control and policy arm we naturally can only observe the agent's reward according to the action applied to them (e.g., $R(0)$ for all agents in the control arm) while the counterfactual remains unobserved.\footnote{We will occasionally also feature \emph{standard} RCTs where there is a treatment arm where everyone gets treated, in contrast to the policy arm in our setting.}

Previous work \cite{mate2022field,verma2023restless,pmlr-v202-mate23a} has evaluated these RCTs by estimating the average benefit that an agent derives from being a member of the policy arm instead of the control arm (independent of whether agents have been selected by the allocation policy or not). Accordingly, they estimate policies' effectiveness as the difference between the expected reward generated by an (arbitrary) agent from the policy arm compared to the expected reward generated by an (arbitrary) agent from the control arm:
\begin{align*}
    \tau^{\mathrm{base}}_{n,\alpha}(\pi) &=\frac{1}{n}\left(\mathbb{E} \sum_{i\in [n]} R_i(J_i^{\pi\left(\mathbf{X}_n,\alpha\right)}) - \mathbb{E} \sum_{i\in [n]} R_i(0)\right)
    =\mathbb{E}  R_{1}(J_{1}^{\pi\left(\mathbf{X}_n,\alpha\right)}) - \mathbb{E} R_1(0)
\end{align*}
To estimate $\tau^{\mathrm{base}}_{n,\alpha}(\pi)$, they compute the difference in the observed generated reward of all agents in the policy arm compared to all agents in the control arm: 
\begin{equation}\label{eq:base}
    \tilde{\theta}^{\mathrm{base}}_{n,\alpha}(\pi)=\frac{1}{n}\left(\sum_{i\in [n]} R^p_i(J^p_i)-\sum_{i\in [n]} R^c_i(0)\right)
\end{equation}
 For the sake of consistency with the next section, we rescale $\tilde{\theta}^{\mathrm{base}}_{n,\alpha}$ and let $\theta^{\mathrm{base}}_{n,\alpha}(\pi):=\frac{n}{\lceil\alpha n\rceil}\tilde{\theta}^{\mathrm{base}}_{n,\alpha}(\pi)$ be the \emph{base estimator}.

\subsection{Shortcomings and Challenges} \label{challenge}
The methodology used in previous work has two main shortcomings: First, the base estimator $\theta^{\mathrm{base}}$ suffers from low statistical power, i.e., the estimator is quite ``noisy'' leading to large confidence intervals and problems with distinguishing policies. Second, as acknowledged by \citet{mate2022field} and \citet{verma2023restless} there are no theoretical guarantees or empirical evidence that computed confidence intervals and drawn statistical conclusions are valid.  

To understand why these problems occur, let us consider the class of threshold policies introduced in \Cref{prelim:SS}, which make independent decisions for every individual.
For these threshold policies standard methods for statistical inference, which rely on the central limit theorem (CLT), can be used. The CLT says that the sample mean of independent observations drawn from some (arbitrary) distribution (as generated, e.g., by a threshold policy in an RCT) converges to a normal distribution. Estimates of this normal distribution's mean $\mu$ and variance $\sigma$ can then be used for estimating the variance of the estimator and for instance to construct valid confidence intervals.  However, for resource allocation policies, the samples that we observe in the policy arm are no longer independent because an agent's treatment and thereby its observed reward depends on the index values of other agents. This renders the standard central limit theorem inapplicable. 
Consequently, statistical tests such as Welch's z-test, which rely on the CLT, are no longer guaranteed to produce accurate statistical conclusions.
Thus, the challenge arises of how to compute valid confidence intervals and p-values for policy evaluation.

Another consequence of the dependence between agents is that if we apply an allocation policy to a group of $n$ agents, we only observe a \textit{single} independent group sample; slightly changing the composition of the group could change the treatment allocation and thereby also the observed rewards (in contrast, for threshold policies, we would derive $n$ fully independent samples).
In light of the resulting lack of independent samples, we face the challenge of constructing estimators that do not suffer from low statistical power needed to draw statistically significant conclusions.

\subsection{Work by \citet{imai2023statistical}} \label{sec:imaili}
We discuss recent work by \citet{imai2023statistical}. The work is positioned differently and does not make any explicit connections to allocation policies and treatment allocation, but upon closer inspection turns out to be closely related.

There is a growing body of work on estimating conditional heterogeneous treatment effects (CATEs) of individuals based on their covariates \cite{wager2018estimation,kunzel2019metalearners,kennedy2023towards} with wide applications ranging from making decisions on patients in precision medicine to making predictions how a treatment performs in a population with a different covariate distribution than observed ones.  
While most statistics works in this direction have focused on designing policies to decide which CATE value should be sufficient to receive treatment \cite{athey2021policy,kitagawa2018should,zhao2012estimating,sun2021empirical,luedtke2016optimal}, few also consider inference and estimation \cite{sun2021treatment,yadlowsky2021evaluating,imai2023statistical}. 
However, from this rich body of works, only the recent work of \citet{imai2023statistical} is upon closer inspection closely related to our problem, as they in contrast to a majority of other works consider average treatment effects in groups of agents (and not only for individuals).

Specifically, \citet{imai2023statistical} analyze how to estimate the average treatment effect in groups of agents with similar CATEs. 
They assume access to a standard RCT where everyone in the treatment arm receives treatment.
Translated to our setting, their methodology applies to estimating the average effect a treatment has on agents with an index value below $q_{\alpha}$, i.e., those agents who belong to the expected $\alpha$-fraction of agents with the lowest index: $$\tau^{\mathrm{q}}_{\alpha}(\Upsilon):= \mathbb{E}_{(\mathbf{x},R)\sim P}
    [R(1)-R(0)\mid \Upsilon(\mathbf{x})\leq q_{\alpha}]$$
To measure this estimand, they 
take the difference between the summed reward of the $\alpha$-fraction of agents in the treatment arm with the lowest indices and the summed reward of the $\alpha$-fraction of agents in the control arm with the lowest indices. 
Using our notation, their estimator, which we call the \emph{subgroup estimator}, is equivalent to the following: 
\begin{align}
    \theta^{\mathrm{SG}}_{n,\alpha}(\pi^\Upsilon)=\frac{1}{\lceil\alpha n\rceil}\left(\sum_{i\in \pi^{\Upsilon}(\mathbf{X}^p_n,\alpha)}\!\!\!\!\! R^p_i(1)- \!\!\!\!\!\sum_{i\in \pi^{\Upsilon}(\mathbf{X}^c_n,\alpha)} \!\!\!\!\! R^c_i(0)\right)
    \label{eq:simestimator1}
\end{align}
They show in Lemma S2 appearing in Appendix S2 of \citet{imai2023statistical} that $\theta^{\mathrm{SG}}_{n,\alpha}(\pi^\Upsilon)$ converges in expectation at a $\sqrt{n}$-rate to $\tau^{\mathrm{q}}_{\alpha}(\Upsilon)$:

\begin{lemma}[informal corollary of  Lemma S2 in \cite{imai2023statistical}]\label{imai1} 
Under very mild assumptions, 
     $\lim_{n\to \infty}\sqrt{n}\left( \tau^{\mathrm{q}}_{\alpha}(\Upsilon)- \mathbb{E}[\theta^{\mathrm{SG}}_{n,\alpha}(\pi^{\Upsilon})] \right) 
    = 0$.
\end{lemma}
Moreover, they also show how to reason about the asymptotic variance of their estimator using the following result: 
\begin{theorem}[informal corollary of  Theorem 2 in \cite{imai2023statistical}]\label{imai2} 
    Under very mild assumptions,
    $$
    \sqrt{n}\left(\theta^{\mathrm{SG}}_{n,\alpha}(\pi^\Upsilon)-\tau^{\mathrm{q}}_{\alpha}(\Upsilon)\right)
   \overset{d}{\rightarrow} \mathcal{N}(0, \sigma^2_{\text{asym}})$$ for some $\sigma^2_{\text{asym}}\geq 0$. We can consistently estimate $\sigma^2_{\text{asym}}$ as $\hat{\sigma}^2_{\text{asym}}$ from the results of a standard RCT. 
\end{theorem}

\section{Methodology} \label{sec:methods}
We present the subgroup estimator for our setting (\Cref{app:estimation}) and describe how we can compute asymptotically valid confidence intervals for it (\Cref{hy:Sub}).
Lastly, in \Cref{hy:Base}, we use an alternative proof to derive analogous results for the base estimator~$\theta^{\mathrm{base}}$.

\subsection{Subgroup Estimator}\label{app:estimation}
We describe how and why a variant of the estimator used by \citet{imai2023statistical} to evaluate CATEs, which we call the subgroup estimator (see \Cref{eq:simestimator1}), can be used in our setting. 

We propose a new estimand that quantifies the effectiveness of a policy by measuring the average effect of a treatment as prescribed by the policy. This estimand will turn out to be equivalent---up to rescaling---to the base estimand $\tau^{\mathrm{base}}$. Our estimand makes it clear how our task connects to \Cref{eq:simestimator1} and explicitly quantifies the effect of \emph{one} treatment, as compared to the base estimand $\tau^{\mathrm{base}}$ that quantifies the effect of being an agent in the policy group (that might or might not receive treatment). 
More concretely, for an allocation policy $\pi$, a treatment fraction $\alpha$, and a group size $n\in \mathbb{N}$, we define $\tau^{\mathrm{new}}_{n,\alpha}(\pi)$ to be the expected additional reward generated by an intervention allocated according to policy~$\pi$:
\begin{align}
    & \tau^{\mathrm{new}}_{n,\alpha}(\pi):= \frac{1}{\lceil\alpha n\rceil} \mathbb{E} \!\!
    \sum_{i\in \pi\left(\mathbf{X}_n,\alpha\right)} (R_i(1)-R_i(0)) \label{eq:est}
\end{align}
$\tau^{\mathrm{new}}_{n,\alpha}(\pi)$ is---up to rescaling---equivalent to the estimand $\tau^{\mathrm{base}}_{n,\alpha}(\pi)$ used in previous work: \begin{align*}
 & \tau^{\mathrm{base}}_{n,\alpha}(\pi)  =\frac{1}{n}\left(\mathbb{E}[ \sum_{i\in [n]} R_i(J_i^{\pi\left(\mathbf{X}_n,\alpha\right)}) - \sum_{i\in [n]} R_i(0)]\right) = \frac{1}{n}\; \mathbb{E}\; [ \sum_{\mathclap{i\in \pi\left(\mathbf{X}_n,\alpha\right)}} (R_i(1)-R_i(0)) +  \sum_{\mathclap{i\notin \pi\left(\mathbf{X}_n,\alpha\right)}} (R_i(0)-R_i(0)) ]   = \frac{\lceil\alpha n\rceil}{n} \tau^{\mathrm{new}}_{n,\alpha}(\pi)
\end{align*}
The reason for this equivalence is that---in expectation---
agents on which we did not act in the policy arm cancel out with agents in the control arm. 
Nevertheless, in the base estimator $\theta^{\mathrm{base}}$ (which simply drops the expectation from $\tau^{\mathrm{base}}$), these agents will introduce noise, as they will influence the observed summed reward of the two arms, i.e., the two sums in $\theta^{\mathrm{base}}$ (cf. \Cref{eq:base}), differently. 
This motivates us to ``remove'' them for the estimation.
The \emph{subgroup estimator}  allows us to do  this. 
We separately estimate the expected reward of agents selected by the policy when treated and when not treated. 
For the first part, we can use the agents selected by our policy in the policy arm (for which we observe $R_i(1)$) and for the second part the agents that \textit{would have been} selected by our policy in the control arm (for which we observe $R_i(0)$). This results in the subgroup estimator $\theta^{\mathrm{SG}}$ from \Cref{eq:simestimator1}: 
\begin{align}
    \theta^{\mathrm{SG}}_{n,\alpha}(\pi)=\frac{1}{\lceil\alpha n\rceil}\left(\sum_{i\in \pi(\mathbf{X}^p_n,\alpha)}\!\!\!\!\! R^p_i(1)- \!\!\!\!\!\sum_{i\in \pi(\mathbf{X}^c_n,\alpha)} \!\!\!\!\! R^c_i(0)\right)
    \label{eq:simestimator}
\end{align}
In fact, it is easy to see that the expected value of the subgroup estimator $\theta^{\mathrm{SG}}$ is equal to our estimand $\tau^{\mathrm{new}}$: 
\begin{small}
    \begin{align}
  \mathbb{E}[\theta^{\mathrm{SG}}_{n,\alpha}(\pi)] =\frac{1}{\lceil\alpha n\rceil}\left(\mathbb{E}\!\!\!\! \sum_{i\in \pi(\mathbf{X}_n,\alpha)} R_i(1)- \mathbb{E}\!\!\!\!\! \sum_{i\in \pi(\mathbf{X}_n,\alpha)} \!\!\!\!\! R_i(0)\right)= \tau^{\mathrm{new}}_{n,\alpha}(\pi) 
   \label{eq:epectation}
\end{align}
\end{small}

 \paragraph{Intuitive Differences between Base and Subgroup Estimator}\footnote{Note that the work of \citet{imai2023statistical} does not discuss any intuition behind the subgroup estimator  $\theta^{\mathrm{SG}}$. Moreover, the base estimator  $\theta^{\mathrm{base}}$ naturally does not appear in their work, as it cannot be used for the task studied by them.}
The base estimator $\theta^{\mathrm{base}}$ treats the RCT arms as indecomposable units and estimates the effect of treatments (allocated according to policy $\pi$) on \emph{the complete policy arm} through a  comparison with the complete control arm. 
In contrast, the idea of the subgroup estimator $\theta^{\mathrm{SG}}$ is to estimate the effect of treatments \emph{on the treated agents} by approximating their unobserved counterfactual behavior (when they did not receive treatment) using the control arm.
For this, we view each agent as an individual sample and compare the agents that received treatment in the policy arm to the agents that would have been assigned treatment by the policy in the control arm.
Thus, in contrast to the base estimator $\theta^{\mathrm{base}}$, the subgroup estimator $\theta^{\mathrm{SG}}$ only takes into account the agents that are ``relevant'' to our policy.
Specifically, $\theta^{\mathrm{SG}}$ ignores the difference $\sum_{i\notin \pi(\mathbf{X}^p_n,\alpha)} R^p_i(0)- \sum_{i\notin \pi(\mathbf{X}^c_n,\alpha)} R^c_i(0)$ that intuitively does not provide us with any insights regarding the policy and only adds noise to the estimator.

\paragraph{Base, Subgroup, and Hybrid Estimator}
The subgroup estimator has a significantly lower variance than the base estimator in our experiments from \Cref{app:experiment}.
However, as we will explain in \Cref{app:cc} this is not a formal guarantee, as there are corner cases where the situation is reversed. 
If one wants to be extra careful to avoid these situations, we present in \Cref{app:hybrid} a hybrid estimator that combines the two, thereby blending their strengths. 
In our experiments from \Cref{app:experiment}, the hybrid estimator performs always extremely similarly to the subgroup estimator.

\paragraph{RCT Design}
Recall that the work of \citet{imai2023statistical} assumed a standard RCT design where everyone in the treatment arm receives a treatment.
However, if resources are scarce this might not be feasible. 
This is why previous work \cite{mate2022field,verma2023restless}  has used customized RCTs where only an $\alpha$ fraction of the agents in the policy arm---as determined by the policy---get treated.
Notably, the base estimator $\theta^{\mathrm{base}}$ can only be applied after such a customized RCT has been conducted.  
The subgroup estimator $\theta^{\mathrm{SG}}$ offers a much more flexible approach: 
We can use it in both settings and in fact for any RCT where all agents that would have been selected by the policy in the treatment group received treatment.
This allows us for instance to run a standard RCT where everyone in the treatment arm gets treated and only specify afterward the index policy whose effectiveness we want to evaluate. We can even use one standard RCT to get an idea of the effectiveness of different index-based allocation policies or different treatment fractions $\alpha$.  

\paragraph{Policy comparison} While the estimand and estimator presented in this section quantify the effectiveness of a single policy, they can also be used to compare two policies against each other. 
A naive approach is to use our machinery presented in \Cref{hy:Sub} to compute $(100-\beta)\%$-confidence intervals for both policies. In case they do not overlap, we can conclude that one policy outperforms the other with probability $(100-2\beta)\%$  by union bound. However, there is also a better approach described in \Cref{hy:Sub}. 

\paragraph{Relation to \citet{pmlr-v202-mate23a}}
To the best of our knowledge, the work of \citet{pmlr-v202-mate23a} is the only other paper explicitly dealing with casual inference for index-based allocation policies. They provide techniques for reducing the variance of estimators; however, their methods do not allow for the computation of confidence intervals that are necessary for hypothesis testing.
In \Cref{sec:comp}, we present a detailed discussion of how the subgroup estimator $\theta^{\mathrm{SG}}$ relates to the estimator of \citet{pmlr-v202-mate23a}, essentially arguing that both lead to a similar variance reduction in our setting, while in contrast to their work, our estimator admits a much simpler formulation and comes with (valid) confidence intervals.

\subsection{Inference for Subgroup Estimator}\label{hy:Sub}
We describe how we can do asymptotically correct inference for the subgroup estimator $\theta^{\mathrm{SG}}$. This section and the next mostly describe ideas, with details and full proofs appearing in \Cref{app:Sub,app:Base}. 

The main ingredient for doing inference with the subgroup estimator $\theta^{\mathrm{SG}}$ is to establish that it is asymptotically normal with respect to our estimand $\tau^{\mathrm{new}}$, i.e., the difference between the estimator and estimand is distributed according to a normal distribution. Estimating the variance of this normal distribution then allows us, for instance, to reason about the probability that the error of the estimator is above a certain threshold (this cannot be achieved by merely knowing that the estimator in expectation converges to the estimand, see \Cref{eq:epectation}). 
To establish this result, the general proof idea is to first show that the subgroup estimator $\theta^{\mathrm{SG}}_{n,\alpha}(\pi^\Upsilon)$ is asymptotically normal with respect to $\tau^{\mathrm{q}}_{\alpha}(\Upsilon)$, i.e., the average intervention effect of treatments when prescribed to agents with index smaller equal $q_{\alpha}$. Subsequently, one can show that $\tau^{\mathrm{q}}_{\alpha}(\Upsilon)$ converges ``fast''  to our estimand $\tau^{\mathrm{new}}_{n,\alpha}(\pi^\Upsilon)$ to conclude the result.

To implement this strategy, we use the results from \citet{imai2023statistical} as discussed in \Cref{sec:imaili}.
It is sufficient to combine \Cref{imai1,imai2} with the simple observation  from \Cref{eq:epectation} that $\mathbb{E}[\theta^{\mathrm{SG}}_{n,\alpha}(\pi)]=\tau^{\mathrm{new}}_{n,\alpha}(\pi)$:
Essentially, \Cref{imai1} implies that we can ``replace'' $\tau^{\mathrm{q}}_{\alpha}(\pi)$ by $\mathbb{E}[\theta^{\mathrm{SG}}_{n,\alpha}(\pi)]=\tau^{\mathrm{new}}_{n,\alpha}(\pi)$ in \Cref{imai2}. 
Formally, using Section 2.2 of \citet{imai2023statistical}, we can conclude that: 
\begin{theorem}\label{theorem} 
    Under very mild assumptions,
    $$
    \sqrt{n}\left(\theta^{\mathrm{SG}}_{n,\alpha}(\pi)-\tau^{\mathrm{new}}_{n,\alpha}(\pi)\right)
   \overset{d}{\rightarrow} \mathcal{N}(0, \sigma^2_{\text{SG}})$$ for some $\sigma^2_{\text{SG}}\geq 0$. $\sigma^2_{\text{SG}}$ can be consistently estimated as
   \begin{align*}
 \hat{\sigma}^2_{\text{SG}}=& \frac{1}{\alpha^2(n-1)}\!\!\sum_{i\in \pi(\mathbf{X}^p_n,\alpha)}  \left(R_i^p(1)- \!\!\!\!\!\!\sum_{i\in \pi(\mathbf{X}^p_n,\alpha)} \!\!\frac{R_i^p(1)}{n}\right)^2\\
 & +\frac{1}{\alpha^2(n-1)}\!\!\sum_{i\in \pi(\mathbf{X}^c_n,\alpha)}  \left(R_i^c(0)- \!\!\!\!\!\!\sum_{i\in \pi(\mathbf{X}^c_n,\alpha)} \!\!\frac{R_i^c(0)}{n}\right)^2 \\
 & - \frac{(1-\alpha)n}{\alpha(2n-1)\lceil\alpha n\rceil^2} \left(\sum_{i\in \pi(\mathbf{X}^p_n,\alpha)} R_i^p(1)- \sum_{i\in \pi(\mathbf{X}^c_n,\alpha)} R_i^c(0)\right)^2
 \end{align*} 

\end{theorem}
\noindent Using Slutky's theorem, we can conclude from \Cref{theorem} that
\begin{equation}\label{normal}
    \sqrt{n}(\nicefrac{\theta^{\mathrm{SG}}_{n,\alpha}(\pi)-\tau^{\mathrm{new}}_{n,\alpha}(\pi))}{\sqrt{\hat{\sigma^2_{\text{SG}}}}}\overset{d}{\rightarrow} \mathcal{N}(0,1).
\end{equation}

\paragraph{Confidence Intervals} From \Cref{normal}, we can derive a formula for asymptotically correct $\beta$-confidence interval of $\tau^{\mathrm{new}}_{n,\alpha}(\pi)$ as \begin{equation}
    I=[\theta^{\mathrm{SG}}_{n,\alpha}(\pi)-Z_{1-\frac{\beta}{2}}\sqrt{\nicefrac{\hat{\sigma}^2_{\text{SG}}}{n}},\theta^{\mathrm{SG}}_{n,\alpha}(\pi)+Z_{1-\frac{\beta}{2}}\sqrt{\nicefrac{\hat{\sigma}^2_{\text{SG}}}{n}}]
\end{equation} where $Z_{\gamma}$ is the $\gamma$ quantile of $\mathcal{N}(0,1)$. 
Asymptotically correct here means that $\mathbb{P}(\tau^{\mathrm{new}}_{n,\alpha}(\pi)\in I) {\rightarrow} 1-\beta$.
Note that the rate-$\sqrt{n}$ convergence established in \Cref{theorem} indicates that $\mathbb{P}(\tau^{\mathrm{new}}_{n,\alpha}(\pi)\in I)$ should ``quickly'' converge to $1-\beta$ and indeed our experiments confirm that the confidence interval is approximately valid already for a limited number of samples in different realistic settings. 

\paragraph{P-Values} \Cref{theorem} and \Cref{normal} also allow us to compute asymptotically valid p-values. Let $\Phi(x)$ be the cumulative distribution of $\mathcal{N}(0,1)$, i.e.,  $\Phi(x)$ is the probability that a sample from $\mathcal{N}(0,1)$ is smaller equal to $x$.
Assume for instance that we wanted to test the null hypothesis $\tau^{\mathrm{new}}_{n,\alpha}(\pi)\leq 0$, then $p=1-\Phi\left(\nicefrac{\sqrt{n}\theta^{\mathrm{SG}}_{n,\alpha}(\pi)}{\sqrt{\hat{\sigma}^2_{\text{SG}}}}\right)$ will be an asymptotically valid p-value, i.e., $\mathbb{P}(p\leq \beta) {\rightarrow} \beta$.  

\paragraph{Welch's $z$-test} Our variance estimator $\hat{\sigma}^2_{\text{SG}}$ and the above-derived confidence intervals are similar to the results of the standard Welch's $z$-test: We would recover the result produced by Welch's $z$-test if we deleted the third term in $\hat{\sigma}^2_{\text{SG}}$, which is always negative. 
In line with this, Welch's $z$-test outputs conservative confidence intervals that are approximately valid in our experiments. 

\paragraph{Policy Comparison (cont'd).} We can use our results from this section to more effectively compare the effectiveness of two policies $\pi_1$ and $\pi_2$ with respective variance estimates $\hat{\sigma}_{1,\mathrm{SG}}^2$ and $\hat{\sigma}_{2,\mathrm{SG}}^2$ from \Cref{theorem}. Assuming that both policies were evaluated in fully independent RCT, the asymptotically correct $\beta$-confidence interval of $\tau^{\mathrm{new}}_{n,\alpha}(\pi_1)-\tau^{\mathrm{new}}_{n,\alpha}(\pi_2)$ is 
$$
    I= [\left(\theta^{\mathrm{SG}}_{n,\alpha}(\pi_1)-\theta^{\mathrm{SG}}_{n,\alpha}(\pi_2)\right)-Z_{1-\frac{\beta}{2}}\sqrt{\nicefrac{(\hat{\sigma}^2_{1,\mathrm{SG}}+\hat{\sigma}^2_{2,\mathrm{SG}})}{n}},
     \left(\theta^{\mathrm{SG}}_{n,\alpha}(\pi_1)-\theta^{\mathrm{SG}}_{n,\alpha}(\pi_2)\right)+Z_{1-\frac{\beta}{2}}\sqrt{\nicefrac{(\hat{\sigma}^2_{1,\mathrm{SG}}+\hat{\sigma}^2_{2,\mathrm{SG}})}{n}}]
$$

\begin{table*}[t]
    \centering
    \begin{subtable}[t]{0.6\linewidth}
        \centering
        \resizebox{0.8\linewidth}{!}{%
        \begin{tabular}{ccccccc}
        \toprule
        \multirow{4}{*}{Domain} & \multicolumn{6}{c}{Estimator} \\
        \cmidrule{2-7}
         & \multicolumn{3}{c}{Base} & \multicolumn{3}{c}{Subgroup} \\
         \cmidrule(lr){2-4} \cmidrule(lr){5-7}
         & $<$ CI & in CI & $>$ CI & $<$ CI & in CI & $>$ CI \\
         \midrule
         Synthetic & 0.027 & 0.952 & 0.021 & 0.036 & 0.935 & 0.029 \\
         TB & 0.024 & 0.946 & 0.030 & 0.031 & 0.947 & 0.022 \\
         mMitra & 0.018 & 0.956 & 0.026 & 0.039 & 0.938 & 0.023 \\
         \bottomrule
        \end{tabular}%
    }
    \caption{Fraction of times the estimand is in, below ($<$~CI), or above ($>$~CI) an estimator's 95\% confidence interval (over 1000 different RCTs).}
    \label{tab:validity}
    \end{subtable}
    \hfil
    \begin{subtable}[t]{0.35\linewidth}
        \centering
        \resizebox{0.8\linewidth}{!}{\begin{tabular}{ccc}
        \toprule
        \multirow{2}{*}{Domain} & \multicolumn{2}{c}{Estimator} \\
        \cmidrule{2-3}
         & Base & Subgroup \\
         \midrule
         Synthetic & 0.426 & 0.178  \\
         TB & 0.778 & 0.293 \\
         mMitra & 0.668 & 0.221  \\
         \bottomrule
        \end{tabular}}
    \caption{Half-width of confidence intervals (averaged over 1000 RCTs).}
    \label{tab:power}
    \end{subtable}
    \caption{Empirical comparison of the confidence intervals produced by different estimators. \textnormal{Both the base and subgroup estimator produce approximately valid confidence intervals; however, the subgroup estimator's confidence intervals are consistently smaller.}}\vspace{-0.5cm}
\end{table*}

\subsection{Inference for Base Estimator}\label{hy:Base}
The proof of \citet{imai2023statistical} cannot be applied to prove the asymptotic normality of the base estimator $\theta^{\mathrm{base}}$. 
Thus, we come up 
with an alternative, more generally applicable proof via empirical process theory \cite{vaart2023empirical} that allows us to prove the asymptotic normality of the base and subgroup estimator (as well as the hybrid estimator featured in \Cref{app:estimation}).
As described in \Cref{hy:Sub}, we can use the asymptotic normality of the base estimator $\theta^{\mathrm{base}}$ to construct asymptotically correct confidence intervals and p-values for it.
In short, we prove the following for the base estimator: 
\begin{theorem}[informal]\label{ourtheorem} 
     Under very mild assumptions 
    $
    \sqrt{n}\left(\theta^{\mathrm{base}}_{n,\alpha}(\pi)-\\\tau^{\mathrm{new}}_{n,\alpha}(\pi)\right)
   \overset{d}{\rightarrow} \mathcal{N}(0, \sigma^2_{\text{base}})$ for some $\sigma^2_{\text{base}}\geq 0$. We can compute a consistent estimate $\hat{\sigma}^2_{\text{base}}$ of~$\sigma^2_{\text{base}}$. 
\end{theorem}

\section{Experiments}\label{app:experiment}
We empirically analyze the base and subgroup estimator using the provably asymptotically correct variance estimation techniques described in \Cref{sec:methods}.
We are interested in 
\begin{enumerate*}[label=(\roman*)]
  \item checking whether the asymptotically valid confidence intervals remain valid in realistic settings, and
  \item comparing the statistical power of the estimators. 
\end{enumerate*}

\paragraph{Setup}\label{setup}
As commonly done in previous work \cite{lee2019optimal,ayer2019prioritizing,mate2022field,verma2023restless,verma2023expanding,DBLP:journals/corr/abs-2308-09726}, we assume that the behavior of each agent is modeled by a Markov Decision Process. 
We focus on adherence settings, where there are two states (`good'$=1$ or `bad'$=0$) and two actions (`intervene'$=1$ or `do not intervene'$=0$), and we obtain a reward of $1$ for every timestep a beneficiary is in the good state. 
Accordingly, the policy's goal is to use interventions to keep agents in the good state. 
By default, our RCT arms consist of $n=5000$ agents and we can intervene on $20\%$ of them ($\alpha=0.2$). 
Agents transition between states according to a transition matrix $T$, where an entry $T^a_{s, s'}$
specifies the probability of transitioning from state $s\in \{0,1\}$ to $s'\in \{0,1\}$ when taking action $a\in \{0,1\}$.
We allocate the interventions in the first time step using the respective transition probabilities to move to the next state. Subsequently, we let agents transition between states using $T^0$ and collect rewards for another $9$ time steps, i.e., the reward generated by an agent is the number of timesteps in which the agent is in the good state.
We consider three domains, differing in how transition matrices are built or learned (see \Cref{app:setup}): 
\begin{description}[topsep=0pt,leftmargin=2em]
\setlength\itemsep{0em}
    \item[Synthetic] Transition probabilities are chosen uniformly at random subject to the constraint that the probability of going to a good state when you act minus when you don't act lies in a certain range, i.e., $T^1_{s,1} - T^0_{s, 1} \in [0, 0.2]$ for each state $s \in \{0,1\}$.
    \item[Medication Adherence (TB)] This domain uses real-world Tuberculosis medication adherence data from~\citet{killian2019learning}. For each agent, the data is used to fit their transition probabilities under the passive action. We then simulate the treatment effect, i.e., $T^1_{s,1} - T^0_{s, 1}$, in each state $s\in \{0,1\}$ by sampling uniformly at random from $[0, 0.2]$.
    \item[Mobile Health (mMitra)] We use real-world data from a field trial conducted by~\citet{mate2022field} to evaluate the effectiveness of service calls to improve engagements in a mobile health information program. Agent's transition probabilities both under the active and passive action are learned from the data.  
\end{description}

To choose which agents to intervene on, we calculate the classic Whittle index~\cite{weber1990index} quantifying an agent's action effect from their transition matrix. We want to estimate the effectiveness of this ``Whittle Index'' policy, which is the standard method to solve the popular restless multi-armed bandits problem, focusing on computing $95\%$ confidence intervals.
To evaluate our estimators, we compute our estimand $\tau^{\mathrm{new}}$ via Monte Carlo simulation.

\begin{figure*}[t]
    \centering
    \includegraphics[width=\linewidth]{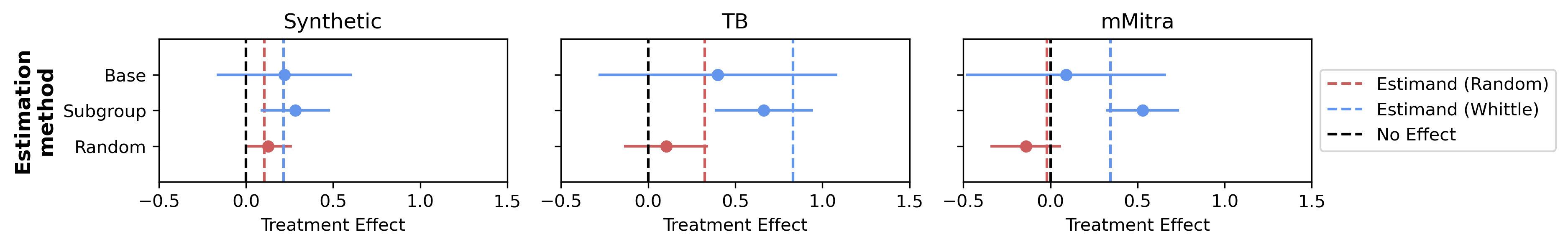}
    \vspace{-1.5em}
    \caption{A representative example of the size of confidence intervals. \textnormal{We compare different estimators for the effectiveness of the Whittle policy (blue) and the random policy (orange). The $x$-axis shows the average effect of a treatment. Vertical lines show the estimand and a zero treatment effect.
    For each estimator, we show their point estimate as a dot and their confidence interval as a line.}}
    \label{fig:cis_1step}
\end{figure*}

\paragraph{Validity (\cref{tab:validity})} We check whether the confidence intervals produced by our estimators are valid, i.e., whether the computed $95\%$ confidence intervals (which differ between runs of an RCT for one simulator) truly contain the estimand (which is constant for each simulator) $95\%$ of the times.  \cref{tab:validity} confirms that both the base and subgroup estimators produce approximately valid confidence intervals, with the error (i.e., $|95\%-\text{in CI}|$) being less than $1.5\%$ in all three domains. 

\paragraph{Power (\Cref{tab:power})} Given that both the base and subgroup estimators are valid, we can compare their power. We do this in \Cref{tab:power} by comparing the (half-)width of their confidence interval. We find that our subgroup estimator always produces tighter confidence intervals, with their width being usually \emph{around a third} of the base estimator's confidence interval across all three domains.

\paragraph{A Representative Example (\Cref{fig:cis_1step})}
It is hard to appreciate the difference between estimators in the abstract. To make the difference more concrete, we picked one representative RCT and show in~\cref{fig:cis_1step} the confidence intervals computed by our estimators for this RCT. As an example of how to read these figures, note that the fact that the confidence interval of the base estimator crosses the black vertical zero line in all three domains implies that we cannot conclude that interventions had a statistically significant positive effect using the base estimator. \Cref{fig:cis_1step} also includes in orange the random allocation policy that assigns treatments uniformly at random to $20\%$ of the agents (its confidence intervals can be correctly computed using a standard Welch's $z$-test).
We find that the subgroup estimator allows us to draw otherwise impossible statistical conclusions. 
In particular, for all three domains, based on the results of the base estimator, we cannot conclude that there is a statistically significant difference between the random and Whittle policy (their confidence intervals overlap). 
In contrast, for the subgroup estimator confidence intervals for the TB and mMitra simulator are disjoint from the random ones. 
Using the approach described at the end of \Cref{hy:Sub}, with high (i.e., $97.5\%+$) probability the expected effect of treatments allocated according to the Whittle policy is $0.2$ (resp.\ $0.4$) higher than of treatments allocated by the random policy in TB (resp.\ mMitra). 

\paragraph{Changing Hyperparameters} In \Cref{app:cHs} (\Cref{hype:1,hype:2,hype:3}), we analyze the influence of different hyperparameters. 
We vary the treatment fraction, the number of agents, the number of observed timesteps, the intervention effect (for TB and synthetic), and the confidence level. 
In all the considered variations, computed confidence intervals remain approximately valid: The error for both estimators is always less than $3\%$ and typically around $1\%$.
Regarding the power of our estimators, the subgroup estimator outperforms the base one in all considered settings, yet the extent varies: The difference is particularly large (up to a factor of $8$) if treatment resources are extremely scarce, there are only a few agents, agents are observed over a long period, or the confidence level is high (see \Cref{fig:budget}). 
An illustrative observation of the discrepancy between the two estimators is that the base estimator can require group sizes up to ten times larger than the subgroup estimator to achieve confidence intervals of a similar size (see \Cref{power:n}). 

\section{Real-World Study} \label{sec:rs}
We reevaluate the field study by \citet{verma2023restless} and start by extending our methodology in different directions to deal with the increased complexity of the real-world field trial.  

\subsection{Extended Methodology}\label{sec:extMe}
We describe various extensions of our estimators to deal with the field trial by \citet{verma2023restless}. 
Our extensions are no longer covered by the variance estimation techniques and theoretical guarantees from \Cref{sec:methods}. 
Thus, in this section, we use the standard Welch's $z$-test to compute confidence intervals. 
As argued at the end of \Cref{hy:Sub}, Welch's $z$-test produces approximately valid confidence intervals in our basic setting, and our empirical results from this section conducted using the same setup as in \Cref{app:experiment} indicate that it continues to do so for our extensions.
Details for the methods and experiments presented in this section appear in~\Cref{app:extensions-meth}.

\subsubsection{Covariates} \label{sec:cov}
To correct for imbalances between agents' covariates in the RCT arms \cite{senn2008statistical,kahan2014risks}, \citet{mate2022field} and \citet{verma2023restless} used linear regression.
The idea is to learn a linear function of covariates and a treatment indicator variable to capture the agent's reward. To correct the subgroup estimator for covariates, we do the following:
For some agent $i$ from the RCT we let $J_i$ be the action that the agent received and $x_{i,1}, \dots, x_{i,m}\in \mathbb{R}$ be the agent's numerical covariates.
We can write the regression as 
$R_i(J_i)=k+\beta J_i+ \sum_{t=1}^m \gamma_{t}x_{i,t}+\epsilon_i$, where  the coefficient $\beta$ presents the average treatment effect $\tau^{\mathrm{new}}$. 
We fit the regression over the $\alpha$-fraction of agents from the policy and control arm with the lowest indices, i.e., 
$\pi(\mathbf{X}^p_n,\alpha)\cup \pi(\mathbf{X}^c_n,\alpha)$.
Note that previous work has used the agent's arm membership as the indicator variable, i.e., they replaced $J_i$ on the right side with the agent's group membership and fitted the regression over all agents.
In our experiments (see \Cref{tab:validity_ext,tab:power_ext} in \Cref{app:extensions}), correcting for covariances can have both positive and negative effects on the size of confidence intervals depending on the correlation between covariates and the reward. 
For the subgroup estimator confidence intervals remain approximately valid, whereas the confidence intervals produced by the base estimator exhibit a $10\%$ error for one of our~simulators. 

\subsubsection{Timestep Truncation}\label{tt}
A common scenario in treatment allocation is to observe agents' behavior for $T$ timesteps after treatments are allocated (and use their combined behavior as the total reward).
Choosing this $T$ is an impactful design decision of the trial.
If we use a small $T$ but intervention effects last for more than $T$ steps, we underestimate the additional reward generated by an intervention leading to a conservative estimate. 
Conversely, if we pick large values of $T$, then the variance in agents' behavior will increase, leading to a larger variance in our estimators:
Decreasing $T$ shrinks confidence intervals while simultaneously shifting them down. 
We find in our experiments that the former effect can be significantly stronger in some cases, leading to higher lower bounds of confidence intervals (see \Cref{tab:power_ext} in \Cref{app:extensions}). 

\subsubsection{Sequential Allocation}
In mMitra interventions are allocated over multiple timesteps with the constraint that each agent only receives a resource in one timestep.
Our subgroup estimator admits a natural extension to this setting: We take the difference between the summed reward of the agents from the policy arm that received treatment and the summed reward of the agents from the control arm that would have been allocated treatment by the policy in one of the steps.  
We find in our experiments that the validity and size of confidence intervals remain unaffected by the number of timesteps over which resources are distributed (see \Cref{fig:multi} in \Cref{app:extensions}).

\begin{figure*}
	\centering
	\includegraphics[width=0.24\linewidth]{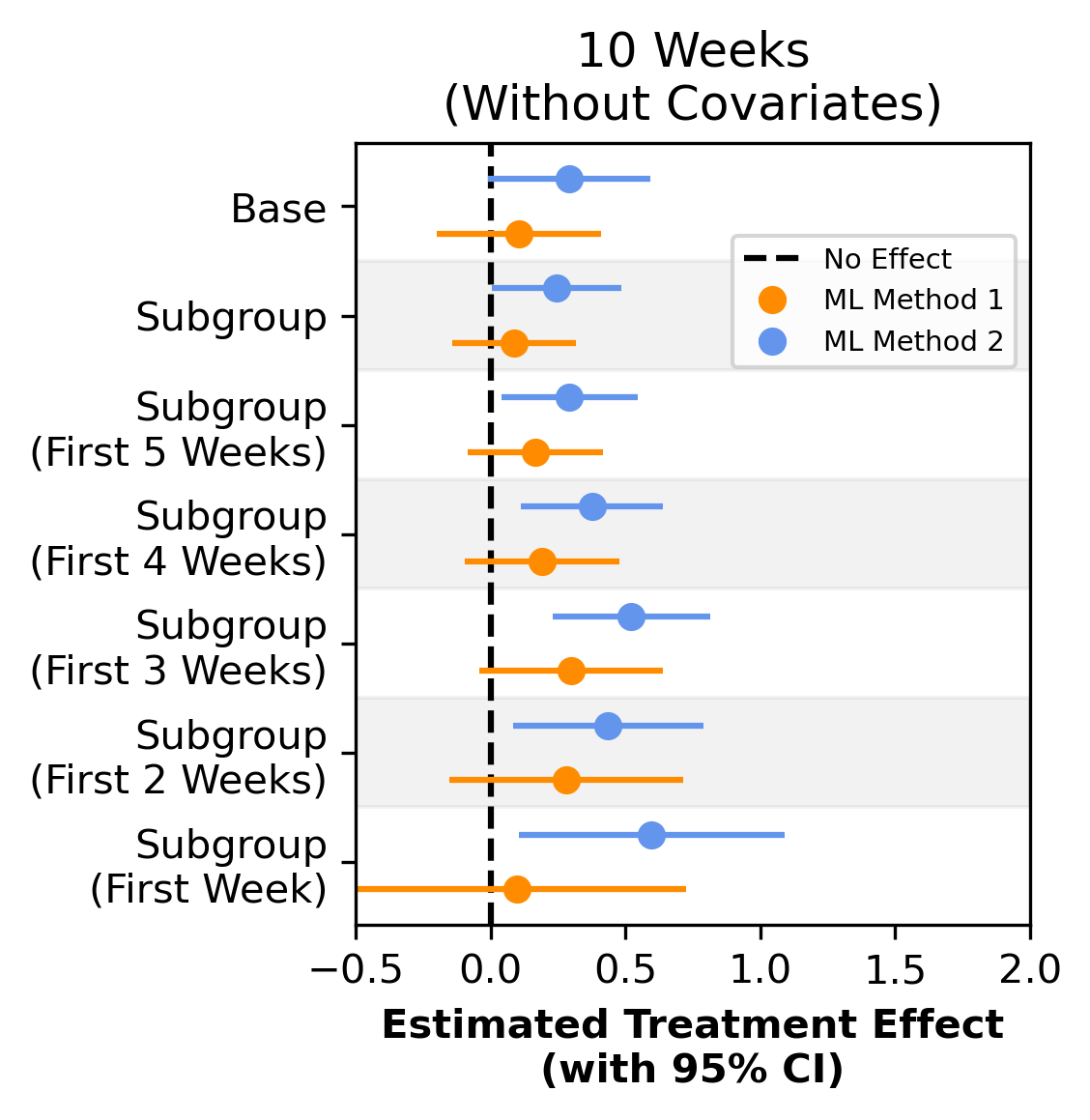} \hfill
	\includegraphics[width=0.24\linewidth]{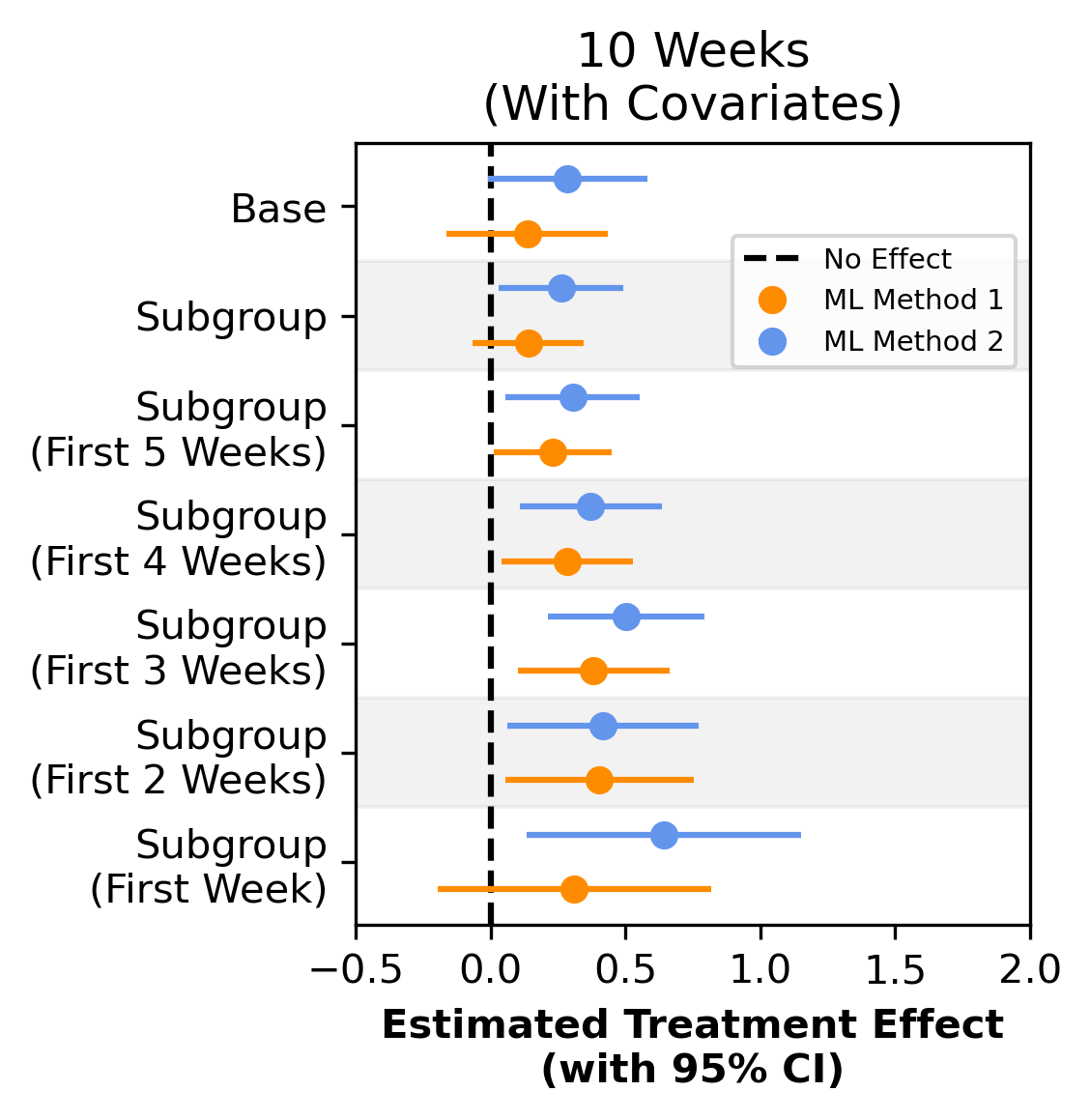}\hfill
	\includegraphics[width=0.24\linewidth]{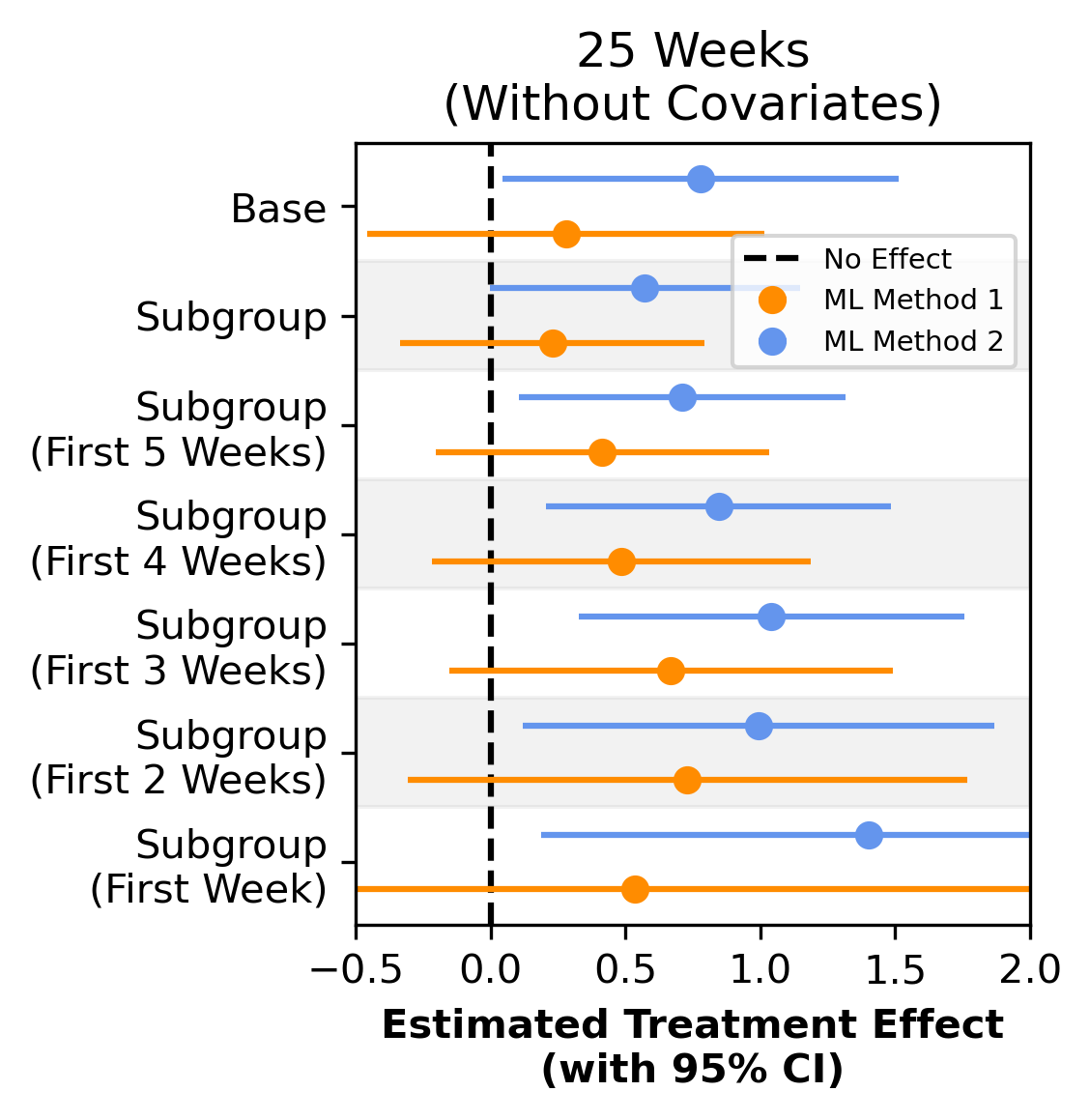}\hfill
	\includegraphics[width=0.24\linewidth]{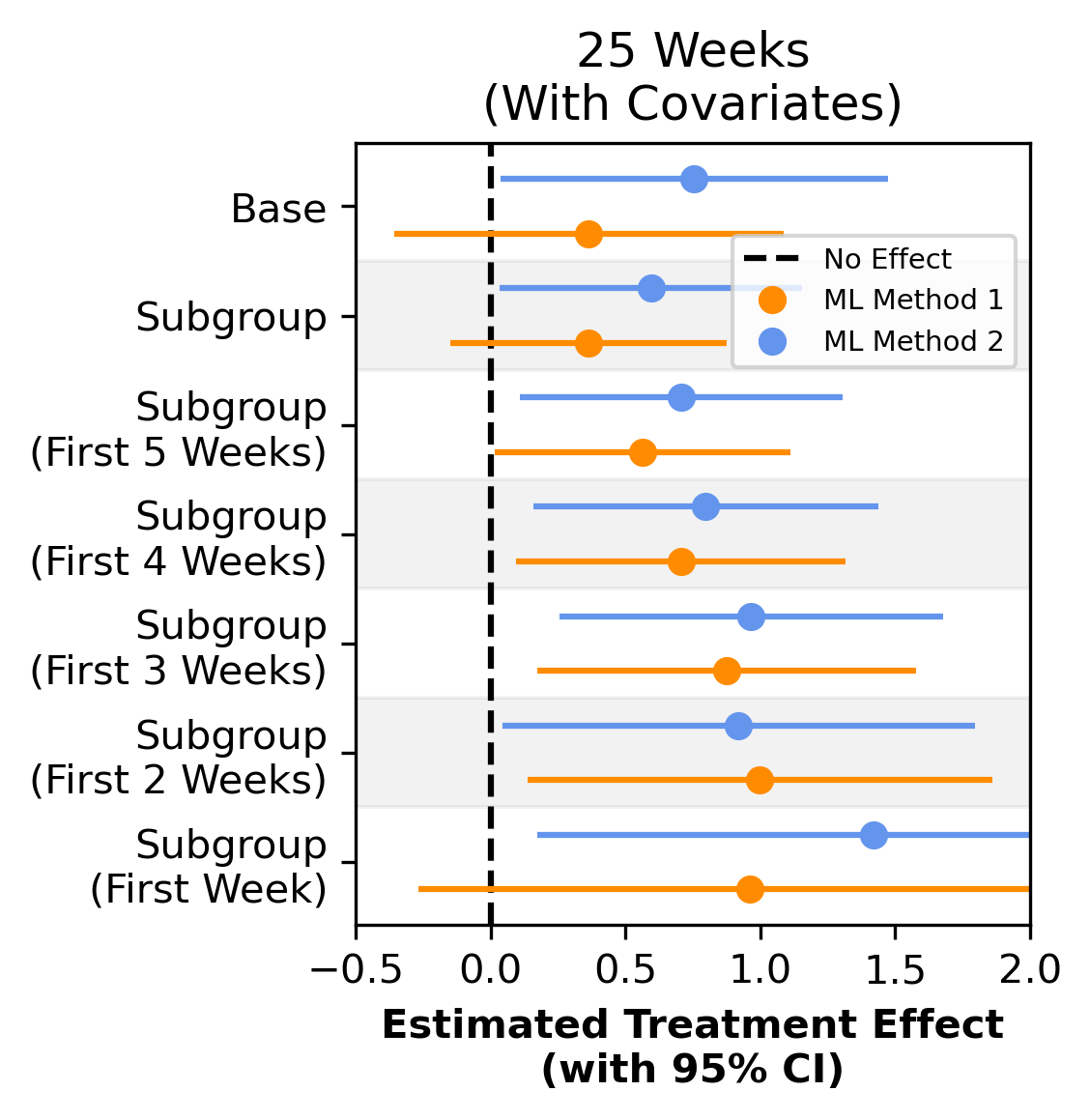}
	\caption{Evaluation of RCT from~\citet{verma2023restless}. \textnormal{We show estimators' point estimates as a dot and $95\%$-confidence intervals as a line for different evaluation horizons with and without correcting for covariates. ``Subgroup (First $x$ weeks)'' refers to our subgroup estimator applied to all agents that (would) have been allocated a treatment up until week~$x$.}}
	\label{fig:rw}
\end{figure*}

\subsection{Results}
We conclude by re-evaluating a real-world RCT conducted by~\citet{verma2023restless}.
The goal of their study was to evaluate the effectiveness of different sequential index-based allocation policies to allocate live service calls to boost participation in ARMMAN's mMitra program (see \Cref{introduction}).
They follow a restless multi-armed bandits approach and use the classic Whittle index \cite{weber1990index}.
Each RCT arm contains $3000$ agents, and $1800$ of them are allocated service calls over $6$ weeks ($300$ calls per week).
The reward generated by an agent is the agent's engagement in the program, i.e., the number of weeks in which they listen to a substantial part of the week's automated voice message. 
\citet{verma2023restless} chose to end the evaluation of their field trial after $10$ weeks, i.e., the reward captures the agents' behavior for $10$ weeks (including the $6$ weeks where treatments are assigned); however, their data also covers the following weeks. Two index-based allocation policies are studied:
``ML Method 1'' is the baseline and  ``ML Method 2'' is their improved approach for index computation.\footnote{ML Method 1 uses past data to learn a model for each beneficiary. Subsequently, in a separate step, the Whittle index is computed for each beneficiary, and based on this resources are allocated. In contrast, ML Method 2 follows a so-called decision-focused learning approach and combines these two steps into one.}
\citet{verma2023restless} used the base estimator with covariate correction (\Cref{sec:cov}). 

\paragraph{Basic Results}
We first focus on the $10$ weeks case as used in the original study (i.e., the two leftmost plots for the case with and without covariate correction).
The first two rows in each subfigure of \Cref{fig:rw} show the results for the base and subgroup estimator.  
We find that using the subgroup instead of the base estimator and correcting for covariates leads to smaller confidence intervals. 
However, none of the four methods is able to establish a positive average effect for interventions as allocated by ML method 1 (the lower bound of the confidence interval is always smaller than $0$), while for ML method 2 the lower bound for all four methods is around $0$. 

\paragraph{Fine-Grained Analysis}
As $60\%$ of agents received a call throughout the trial, establishing a positive average intervention effect (on this large subpopulation) can be quite challenging, since the effect of cleverly assigned treatments decreases in the number of allocated treatments.
This raises the question of whether service calls significantly positively affect at least \emph{some} of the $1800$ agents receiving them, which turns out to be the case.  
To answer this question, we make use of the flexibility of the subgroup estimator. For some $x\in [1,5]$, we estimate the average effect of a service call on agents called in one of the first $x$ weeks by comparing their reward to the reward of agents that would have been called in the control arm in one of the first $x$~weeks. 

Turning to the results (rows three to seven in each subfigure of \Cref{fig:rw}), 
we find that for ML method 2 this view allows us to conclude statistically significant (large) positive intervention effects for agents called in one of the first $x$ for each $x\in [1,5]$, irrespective of whether we correct for covariates or not. 
Note that for $x=1$ and no covariate correction, we recover the single-round allocation setting and the standard subgroup estimator discussed in \Cref{sec:methods}. 
Thus, our conclusion that interventions---as prescribed by ML method 2---have a statistically significant effect on agents called in the first week is theoretically backed by our proofs from \Cref{sec:methods}. 
Moreover, when correcting for covariates, we can even establish that the service calls allocated in the first weeks by the ML method 1 have a statistically significant effect. 
Looking into the fine-grained structure of intervention effects was impossible under the base estimator, as it treats the policy arm as one indecomposable unit.

\paragraph{RCT Budget}
The reason why \citet{verma2023restless} allocated treatment to so many agents in their RCT is because the base estimator has an enormous variance and suffers from extremely low statistical power when the treatment fraction is low (see \Cref{fig:budget} in \Cref{app:cHs}). 
Thus, trying to establish a positive average intervention effect on the $1800$ agents is in some sense the best one can do with the base estimator. However, the subgroup estimator shows a much better performance when the budget is small. As a result, it allows us to run RCTs with much lower costs for which average intervention effects can even be established more easily.

\paragraph{$25$ Weeks of Evaluation Horizon} Moving from observing beneficiaries for a total of $10$ weeks to a total of $25$ weeks has signifcant consequences. As featured in \Cref{tt}, this increases the value of the estimator while concurrently leading to (much) larger confidence intervals. However, despite this increase in the size of the confidence intervals, it turns out that this leads to an increase in the lower bounds of confidence intervals here. As a result, for a confidence level of $95\%$, using $25$ instead of $10$ weeks allows us to establish up to $50\%$ larger effect sizes, e.g., for the first three weeks for ML method 2 without covariate correction.
A relevant side conclusion of this analysis is that intervention effects in mMitra seem to be long-lasting.

\bibliographystyle{ACM-Reference-Format}

\newpage

\newpage
\appendix

\onecolumn

\section*{Appendix}
\tableofcontents

\section{Additional Material for Section \ref{app:estimation}}

\subsection{Corner Case: Base Estimator outperforms Subgroup Estimator}\label{app:cc}
Intuitively, the base estimator performs advantageously in cases where agents that are at the boundary of getting treated generate a much higher reward than other agents. The subgroup estimator might include some of these ``noisy'' agents in the control arm while not selecting them in the policy arm, leading to noisy estimates. The base estimator is better equipped to handle such scenarios, as it takes all agents into account so that such effects can cancel out.

To make this intuition more concrete, we state a result that we will later prove in \Cref{app:Base,app:Hyb}: 
\begin{theorem}[informal]\label{context} 
    Under mild assumptions, it holds that $
    \sqrt{n}\left(\theta^{\mathrm{SG}}_{n,\alpha}(\pi)-\tau^{\mathrm{new}}_{n,\alpha}(\pi)\right)
   \overset{d}{\rightarrow} \mathcal{N}(0, \sigma^2_{\text{SG}})$ and 
   $\sqrt{n}\left(\theta^{\mathrm{base}}_{n,\alpha}(\pi)-\tau^{\mathrm{new}}_{n,\alpha}(\pi)\right)
   \overset{d}{\rightarrow} \mathcal{N}(0, \sigma^2_{\text{base}})$ where 
   \begin{align*}
 \sigma^2_{\text{SG}}  & =\frac{1}{\alpha^2}\bigg(\alpha(1-\alpha)(\rho_1^2+  \rho_0^2)  - 2(1-\alpha)(\rho_1 \mu_1  +\rho_0 \mu_0)+ \sigma^2_1+\sigma^2_0\bigg).\\
        \sigma^2_{\text{base}} &
    = \frac{1}{\alpha^2}\left(\alpha(1-\alpha)(\rho_1-\rho_0)^2  
    +(2\alpha\check{\mu}_0- 2(1-\alpha)\mu_1 )(\rho_1-\rho_0)+\sigma^2_1 + \check{\sigma^2_0} -2\mu_1\check{\mu}_0+ \mathrm{Var}(R(0))\right)
   \end{align*}   
    with $\mu_{i}=\mathbb{E}[R(i))I[\Upsilon(\mathbf{x})\leq q_{\alpha}]$, $\check{\mu_{i}}=\mathbb{E}[R(i))I[\Upsilon(\mathbf{x})> q_{\alpha}]$, $\rho_{i}=\mathbb{E}[R(i)|\Upsilon(\mathbf{x})= q_{\alpha}]$, $\sigma^2_i=\mathrm{Var}[R(i)I[\Upsilon(\mathbf{x})\leq q_\alpha]]$ and $\check{\sigma^2_i}=\mathrm{Var}[R(i)I[\Upsilon(\mathbf{x})> q_\alpha]]$ for $i\in \{0,1\}$ where $\mathbb{E}$ is taken over $(\mathbf{x},R)\sim P$.
\end{theorem}

We present one specific example in the following, whose crucial ingredient is that we perturb the reward of an agent with covariates $\mathbf{x}$ with $f(\Upsilon(\mathbf{x})-\alpha,0,0.05)$, where $f(a,\mu,\sigma)$ is the pdf of $\mathcal{N}(\mu,\sigma)$ evaluated at $a$. This means that the reward of agents whose index is close to the $\alpha$-quantile $q_{\alpha}$ will get a large ``boost'' in their reward. 
To understand why the estimators react differently to this, we turn to the variance expressions from \Cref{context}. 
The above-described ``boost'' will increase terms $\rho_1$ and $\rho_0$ drastically at the same rate, while only marginally affecting all other terms. 
For the base estimator, the increase in  $\rho_1$ and $\rho_0$ will approximately cancel out each other (as only the difference $\rho_1-\rho_0$ between the terms appear).  In contrast, in the variance of the subgroup estimator, the sum $\rho_1+\rho_0$ of the two terms appears, implying that no such effect takes place and the variance substantially increases. 

Formally, in our example, we use $n=500$ and $\alpha=0.5$. Each agent $i$ has a single covariate $x_i\sim \mathcal{N}(0,1)$ and the index function is the identity function, i.e., the index of agent $i$ is $x_i$. To generate the reward of an agent, we sample some noise $y_i\sim \mathcal{N}(0,1)$ for every agent. Moreover, for each agent we add an index-dependent ``boost'' $z_i=f(x_i-\alpha,0,0.05)$, where $f(a,\mu,\sigma)$ is the pdf of $\mathcal{N}(\mu,\sigma)$ evaluated at $a$. 
We set $R_i(0)=x_i+y_i+z_i$ and $R_i(1)=R_i(0)+1$, i.e., we have a constant intervention effect of $1$. 
\Cref{ex1} shows the distribution of the value of the estimators in $100000$ simulated RCTs (note that we also include the hybrid estimator which we present in the next section).
We see here that the variance of the subgroup estimator is higher than that of the base one leading to around $20\%$ larger confidence intervals. 

\begin{figure*}[t!]
    \centering
    \begin{subfigure}[t]{0.5\textwidth}
        \centering
        \includegraphics[width=\textwidth]{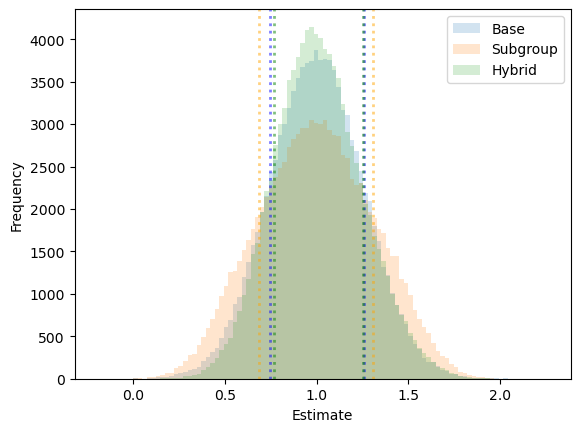}
        \caption{Example from \Cref{app:cc}.\label{ex1}}
    \end{subfigure}%
    ~ 
    \begin{subfigure}[t]{0.5\textwidth}
        \centering
        \includegraphics[width=\textwidth]{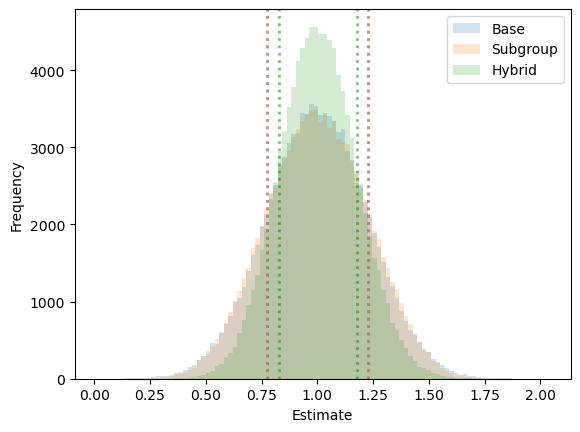}
        \caption{Example from \Cref{app:hybrid}\label{ex2}}
    \end{subfigure}
    \caption{Distribution of the value of different estimators for $100000$ RCTs. The estimand is $1$ by construction. Horizontal lines indicate one standard deviation below and above the mean.}
\end{figure*}

\subsection{Hybrid Estimator}\label{app:hybrid}
Motivated by \Cref{app:cc}, we propose a hybrid estimator that linearly combines the base and subgroup estimators, thereby blending their strengths. 
Specifically, for any sequence $\hat{w}_n$ for which $\hat{w}_n \overset{p}{\rightarrow} w$ for some fixed $w \in \mathbb{R}$, we define the \emph{hybrid estimator} as  $$\theta^{\mathrm{hyb}}_{n,\alpha,\hat{w}}(\pi):=(1-\hat{w}_n)\cdot \theta^{\mathrm{SG}}_{n,\alpha}(\pi)+\hat{w}_n\cdot \theta^{\mathrm{base}}_{n,\alpha}(\pi).$$
In \Cref{hy:Base}, we discuss formulas for computing the ``optimal'' value $w^*$ of $w$  and for computing the (asymptotically valid) confidence intervals of the induced estimator. However, in our experiments from \Cref{app:experiment} we find that the optimal hybrid estimator performs always extremely similarly to the subgroup one. 
However, there are some cases where the hybrid estimator with weight $w^*$ performs better than the other two. 
Specifically, we can slightly adjust the example described in \Cref{app:cc} by setting $z_i=f(x_i-\alpha,0,0.08)$ and observe in this case that the variance of the hybrid estimator is smaller than of the other two. $95\%$ confidence intervals produced by the hybrid estimator are around $20\%$ smaller than those of computed by the other two.  

\subsection{Threshold Estimator}
Note that an alternative view on the subgroup estimator is that it compares the average behavior of agents receiving treatment to a proxy for their average behavior in case we did not act on them. 
In the context of the subgroup estimator, this proxy is obtained by examining the agents from the control arm that would have been selected by the policy. However, there are also other approaches:
Let $\lambda$ be the largest index value of an agent on which we acted in the policy arm. 
One alternative approach is to estimate the expected intervention effect of the threshold policy $\upsilon^\Upsilon(\cdot,\lambda)$ as a proxy of $\tau^{\mathrm{new}}_{n,\alpha}(\pi^\Upsilon)$ (note that, up to ties in the index values, $\upsilon^\Upsilon(\cdot,\lambda)$ would have selected the same agents in the policy arm as $\pi^{\Upsilon}$).
Recall that the expected intervention effect of $\upsilon^\Upsilon(\cdot,\lambda)$ is much easier to deal with from a statistical point of view, as the behavior of different agents is no longer linked to each other and they can be viewed as fully independent again. 
We arrive at the following estimator:
$$\theta^{\mathrm{TE}}_{n,\alpha}(\pi)=\frac{1}{\lceil\alpha n\rceil} \!\!\!\!\! \sum_{i\in \pi(\mathbf{X}^p_n,\alpha)} \!\!\!\!\!
R^p_i(1)- \frac{1}{|\upsilon^\Upsilon(\mathbf{X}^c_n,\lambda)|}
  \!\!\!\!\! \sum_{i\in \upsilon^\Upsilon(\mathbf{X}^c_n,\lambda)}  \!\!\!\!\!  R^c_i(0).$$
However, we found in our experiments that the subgroup and threshold estimators behave very similarly in practice.

\subsection{Relation to \citet{pmlr-v202-mate23a}}\label{sec:comp}

The main idea of \citet{pmlr-v202-mate23a} is to reduce the variance of the estimator $\theta^{\mathrm{base}}$  by reshuffling individuals across experimental arms after the end of the trial. The idea is that the data we observe in our RCT also gives us access to the results of hypothetical RCTs with different partitions into policy and control arms where the treated set of agents does not change.
One of the main limitations of the work of \citet{pmlr-v202-mate23a} is that their algorithm only produces a point-estimate (and no confidence intervals), probably also partly since they could not provide a closed-form expression of the estimator.
However, in the setting considered by us the latter problem can be fixed.

In the single-step setting, their algorithm reduces to the following: Let $\lambda$ be the largest index of an agent in the policy arm that gets a treatment, let $N^c$ be the agents in the control arm with an index above $\lambda$, and let $N^p$ be the agents in the policy arm which we did not treat. Note that if we had exchanged some agent from $N^c$ with some agent from $N^p$ before the start of the trial, both of them would continue to not receive any action, so we have access to the outcome of this hypothetical trial.  
Let $r=\frac{1}{|N^c\cup N^p|} \sum_{i\in N^c\cup N^p} R_i(0)$ be the average reward of agents from $N^c\cup N^p$. 
The idea of \citet{pmlr-v202-mate23a} is now to replace the reward $R_i(0)$ of agents from $N^c\cup N^p$ with $r$ in the definition of $\theta^{\mathrm{base}}$. As a result of this, non-treated agents from the control and policy arm will partly cancel out each other, resulting in: 
$$\frac{1}{n}\left(\sum_{i\in [n]\setminus N^p}\!\!\!\!\! R_i(1)+(|N^p|-|N^c|)r-\!\!\!\!\!\!\! \sum_{i\in [n]\setminus N^c} \!\!\!\!\!\!\! R_i(0)\right).$$
This estimator can be interpreted as a rescaled and perturbed version of the threshold estimator, where in case that $|N^p|\neq |N^c|$ the smaller of the two groups ($[n]\setminus N^p$ vs. $[n]\setminus N^c$) gets ``filled'' with agents whose reward we estimate as $r$ to result in equal-sized groups. 
The same analogy also holds in the sequential setting.

\section{Additional Material for Section \ref{hy:Sub}}\label{app:Sub}
It remains to describe the assumptions under which \Cref{theorem} holds.
Recall that $F_{\Upsilon}(\lambda)=\mathbb{P}_{(\mathbf{x},R)\sim P}[\Upsilon(\mathbf{x})\leq \lambda]$  is the cumulative distribution function of indices and let 
$F^{-1}(p)=\inf\{\lambda\mid F_{\Upsilon}(\lambda)\geq p\}$ be the quantile function of $F_{\Upsilon}$. 
In addition to the assumptions made in our setup from \Cref{prelim:SS}, the additional assumptions (which are Assumptions 4 and 5 in the work of \citet{imai2023statistical}) are: 

\begin{assumption}\label{third-moment}
    $\mathbb{E}_{(\mathbf{x},R)\sim P} |R_i(i)|^3< \infty$ for $i\in \{0,1\}$.
\end{assumption}

\begin{assumption}\label{pos-var}
    $\var_{(\mathbf{x},R)\sim P} R_i(i)>0$ for $i\in \{0,1\}$.
\end{assumption}

\begin{assumption}\label{regularity}
 $\mathbb{E}_{(\mathbf{x},R)\sim P} [R_i(1)-R_i(0)\mid \Upsilon(\mathbf{x})=F^{-1}(p)]$ is left-continuous (in $p$) with bounded variation on any interval $(\gamma,1-\gamma)$ with $\gamma>0$ and continuous at $\alpha$.
\end{assumption}

\section{Additional Material for Section \ref{app:experiment}}

\subsection{Details on Setup}\label{app:setup}

To compute a high-quality approximation of our estimand, we take the average over $\num{1000}$ RCTs. In each of these RCTs, we sample $n$ agents (as characterized by their transition matrix) uniformly at random. Subsequently, using their transition probabilities, we analytically compute for the  $\lceil\alpha n\rceil$ agents with the lowest index the average difference in expected reward when they are intervened on or not.

\vspace{0.5em}\noindent We describe the simulation domains in more detail below:

\paragraph{Synthetic} Transition probabilities in the absence of an intervention are chosen uniformly at random:
$$T^0_{0, 1}, T^0_{1, 1} \sim U[0, 1]\text{ and, } T^0_{s, 0} = 1 - T^0_{s, 1}\text{ for }s=\{0,1\}$$
where $1$ is the good state and $0$ is the bad state. The probabilities for when we \textit{do} intervene (active transitions $T^1$) are chosen uniformly at random, with the constraint that you are more likely to transition to the good state when you act: 
$$T^1_{0, 1}, T^1_{1, 1} \sim U[0, 1]\text{ s.t., } (T^1_{s, 1} - T^0_{s, 1}) \in [0, 0.2] \text{ for }s=\{0,1\}$$

\paragraph{Medication Adherence (TB)} We use data from~\citet{killian2019learning} to learn the passive transition probabilities for different agents. We then sample the effect of acting, i.e., $T^1_{s,1} - T^0_{s, 1}$, in each state $s\in \{0,1\}$ uniformly at random from $[0, 0.2]$ for every agent. 

\paragraph{Mobile Health (mMitra)} We use the data from the `random' arm of the field trial in~\citet{mate2022field} to generate transition probabilities from the observed data. We do this by first discretizing engagement into 2 states---an engaging beneficiary listens to the weekly automated voice message (average length 60 seconds) for more than 30 seconds---and sequencing them to create an array $(s_0, a_0, s_1, \ldots)$. Then, to get the transition matrix for beneficiary $i$, we combine the observed transitions with $P_{\text{pop}}$, a prior created by pooling all the beneficiaries' trajectories together such that for each beneficiary:
\begin{align*}
    T^a_{s,s'} = P(s' | s, a) = \frac{\alpha P_{\text{pop}}(s' | s, a) + N(s, a, s')}{\sum_{x \in \mathcal{S}} \alpha P_{\text{pop}}(x | s, a) + N(s, a, x)}
\end{align*}
where $N(s, a, s')$ is the number of times the sub-sequence $(s, a, s')$ occurs in the trajectory of that beneficiary, and $\alpha = 5$ is the strength of the prior.

\subsection{\Cref{fig:consistency}}
In \Cref{fig:consistency}, we give examples for the confidence intervals produced by the subgroup estimator for $100$ RCTs generated uniformly at random for the three different simulators. 
In blue, we see the estimand and we mark the confidence interval in red if the estimand lies outside of it (which should ideally happen $5\%$ of the time). We see that the size of confidence intervals stays roughly the same across different RCTs, while the position is slightly changing sometimes pushing the estimand outside of the interval.

\begin{figure}[t]
    \centering
    \includegraphics[width=\linewidth]{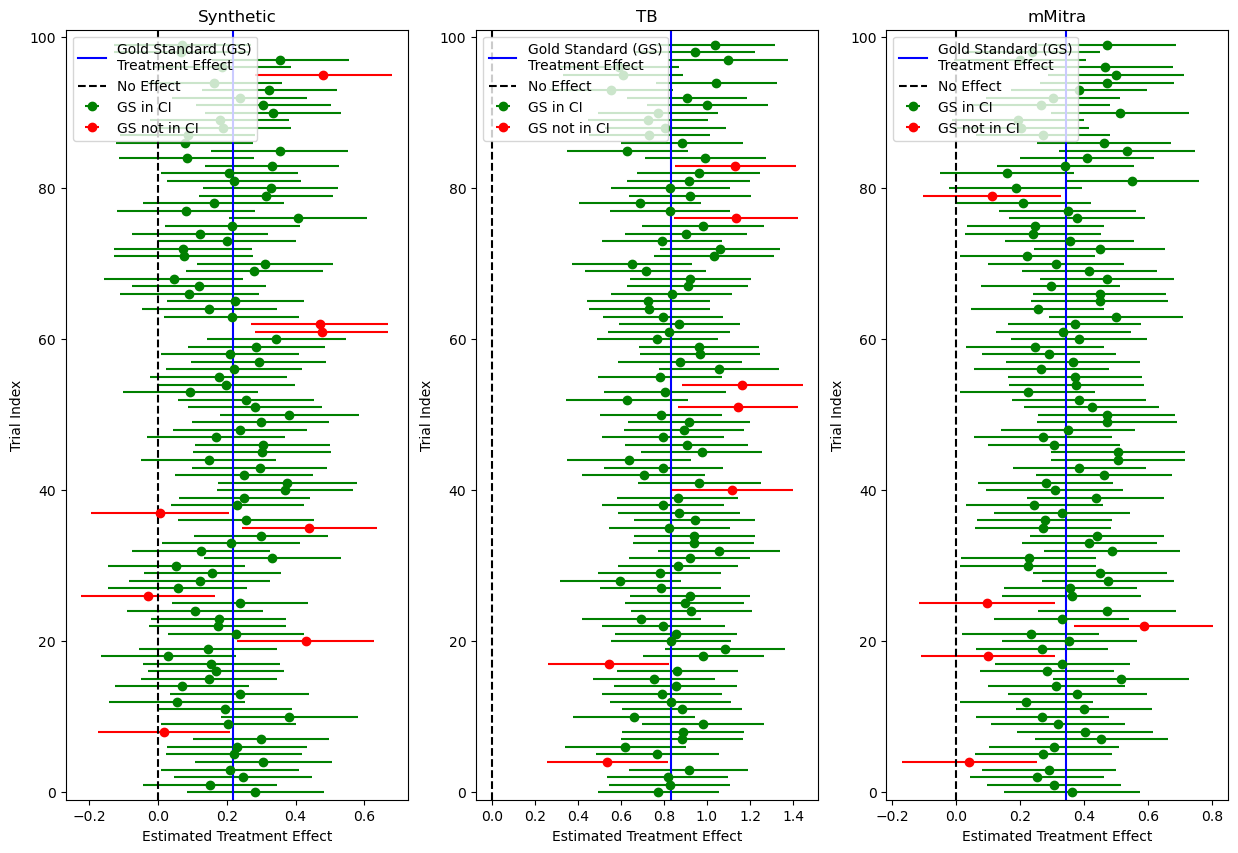}
    \vspace{-2em}
    \caption{Confidence Intervals created by the Subgroup Estimator for 100 different simulations.}
    \label{fig:consistency}
\end{figure}

\subsection{Changing Hyperparameters} \label{app:cHs}

We analyze the influence of the hyperparameters of our simulation. Our default configuration is $n=5000$ agents, $\alpha=0.2$, $10$ observed timesteps, a maximum intervention effect of $0.2$ (for synthetic and TB), and a $95\%$ confidence level. Then, in each experiment, we vary one parameter while keeping the others constant. We give a summary of the insights from these experiments below:
\begin{description}
    \item[Treatment Fraction (\Cref{fig:budget1,fig:budget})] We vary the treatment fraction $\alpha$. We observe the natural trend that the smaller the treatment fraction, the larger the confidence intervals. However, the strength of this effect is very different for the two estimators with the base estimator producing very large intervals as soon as $\alpha$ drops below $0.1$. Moreover, the confidence intervals output by the base estimator exhibit errors up to $3\%$ here (which is much higher than what we observe elsewhere). For the subgroup estimator, the error is small $\leq 1\%$ in almost all cases.
    \item[Number of Agents (\Cref{power:n1,power:n})] We vary the number of agents while keeping the treatment fraction constant. This does not seem to have any clear effect on the validity of confidence intervals. For the size of confidence intervals, we observe the natural trend that if we have more agents, the size of confidence intervals naturally shrinks. The strength of the effect is roughly similar for the two estimators. However, notably, even with $20000$ agents, the base estimator still produces intervals of considerable size.  
    \item[Number of Observed Timestep (\Cref{fig:O,fig:O2})] We vary the number of timesteps we observe after the allocation of treatment, i.e., the number of timesteps over which agents collect reward if they are in the good state. This does not have a clear impact on the validity of confidence intervals. Clearly, the longer we observe agents, the higher will be the variance in their behavior. Thus, it is unsurprising that for both estimators confidence intervals get larger when more steps are observed. Notably, this increase is particularly pronounced for the base estimator on the mMitra and TB domains.
    \item[Intervention Effect (\Cref{hype:2})] We analyze what happens if we change the intervention effect. Recall that, for both the synthetic and TB domains, we sample the transition probabilities such that the probability of the active action going to a good state exceeds that of the passive action by a maximum of $0.2$, i.e., $T^1_{s,1} - T^0_{s, 1} \in [0, 0.2]$. In this set of experiments, we vary this "upper bound" of $0.1$ to $0.5$. We find that the effect size does not seem to have a strong influence on the validity and size of confidence intervals. 
    \item[Confidence Level (\Cref{hype:3})] So far, we focused on $95\%$ confidence intervals. Here, we examine the performance of our estimators for $90\%$ and $99\%$ confidence intervals. In terms of validity, we find that the error is roughly similar independent of the confidence level. We observe that most often the confidence intervals are slightly under-covering, i.e., the estimand does not fall into the confidence interval sufficiently often.  In terms of sizes, we unsurprisingly find that we have larger intervals when increasing the confidence level. What is more surprising is that for the subgroup estimator the difference between $90\%$ and $95\%$ is roughly similar to the difference between $95\%$ and $99\%$, while for the base estimator the latter difference is larger. 
\end{description}

\begin{figure*}[t]
    \begin{subfigure}[t]{0.5\textwidth}
        \centering
        \includegraphics[width=\textwidth]{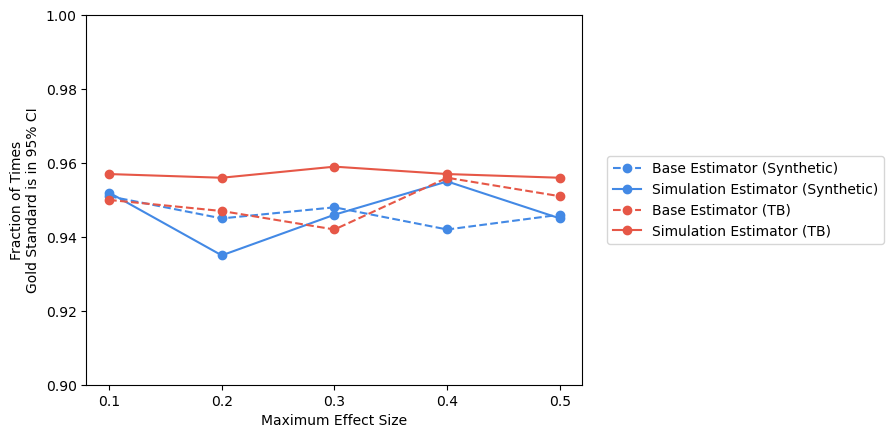}
        \caption{Validity of confidence interval when varying the intervention effect.}
    \end{subfigure}%
    ~ 
    \begin{subfigure}[t]{0.5\textwidth}
        \centering
        \includegraphics[width=\textwidth]{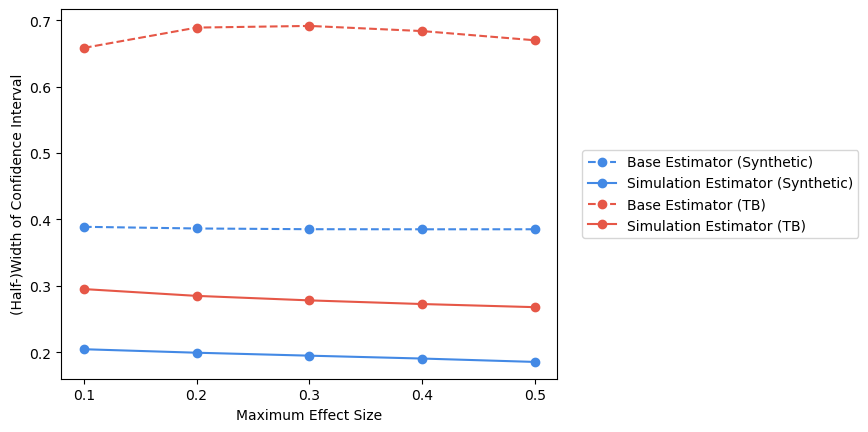}
        \caption{Power of estimators when varying the intervention effect.}
    \end{subfigure}
    \caption{Empirical comparison of the confidence intervals produced by different estimators when varying the intervention effect for the synthetic and TB domain, where we generate intervention effects randomly. In particular, for both domains, we adjust the sampling so that the maximum intervention effect, i.e., the difference between the transition probability under passive and active action, is at most the value depicted on the $x$-axis. On the left, we analyze validity by showing the fraction of times the estimand falls in an estimator's $95\%$ confidence interval (the closer to $95\%$ the better). On the right, we analyze the power of estimators by depicting the half-width of computed confidence intervals (the smaller the better).}\label{hype:2}
\end{figure*}

\begin{figure*}[t]
    \begin{subfigure}[t]{0.5\textwidth}
        \centering
        \includegraphics[width=0.9\textwidth]{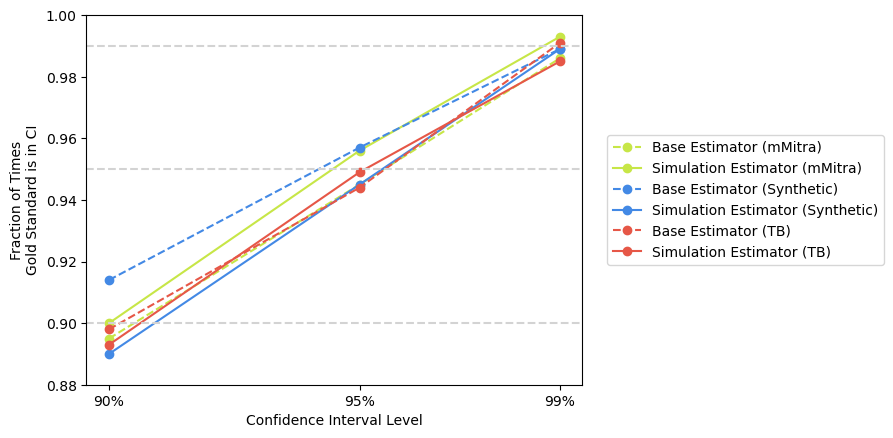}
        \caption{Validity of confidence interval for different confidence levels.}
    \end{subfigure}%
    ~ 
    \begin{subfigure}[t]{0.5\textwidth}
        \centering
        \includegraphics[width=0.9\textwidth]{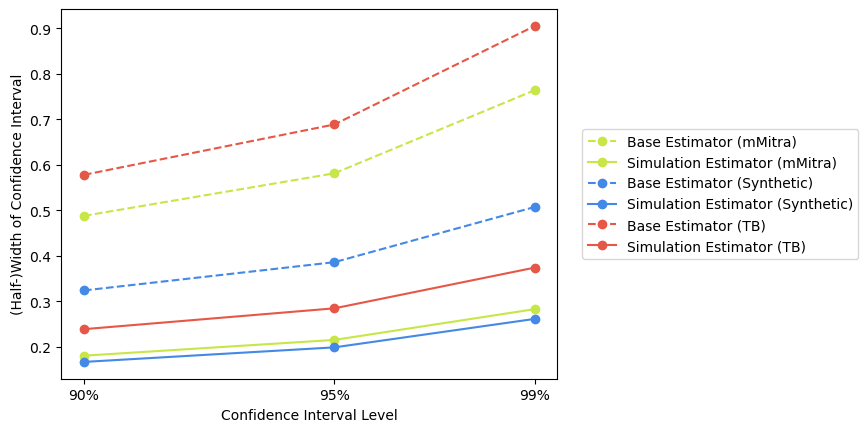}
        \caption{Power of estimators for different confidence levels.}
    \end{subfigure}
    \caption{Empirical comparison of the confidence intervals produced by different estimators for different confidence levels. On the left, we analyze validity by showing the fraction of times the estimand falls in an estimator's confidence interval (the closer the $x$ value is to the $y$ value, the better). On the right, we analyze the power of estimators by depicting the half-width of computed confidence intervals (the smaller the better).}\label{hype:3}
\end{figure*}

\begin{figure*}[t]
    \centering
    \begin{subfigure}[t]{0.48\textwidth}
        \centering
        \includegraphics[width=0.9\textwidth]{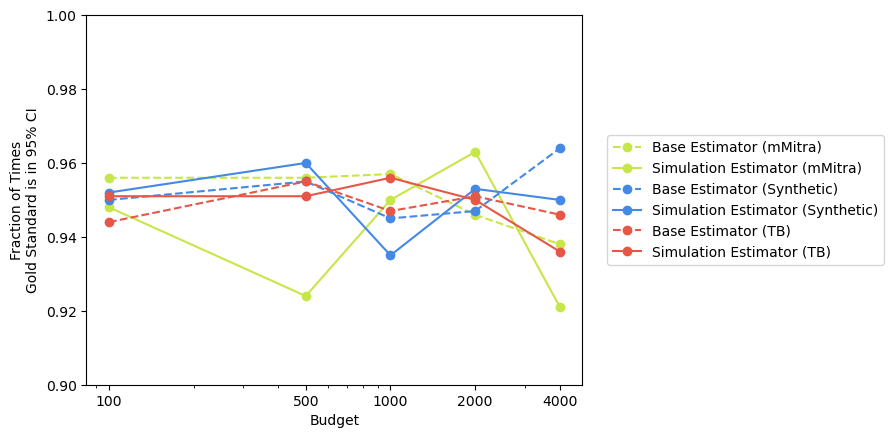}
        \caption{Validity of confidence interval when varying the budget, i.e., the number of allocated treatments.}\label{fig:budget1}
    \end{subfigure}%
    \hfill 
    \begin{subfigure}[t]{0.5\textwidth}
        \centering
        \includegraphics[width=0.9\textwidth]{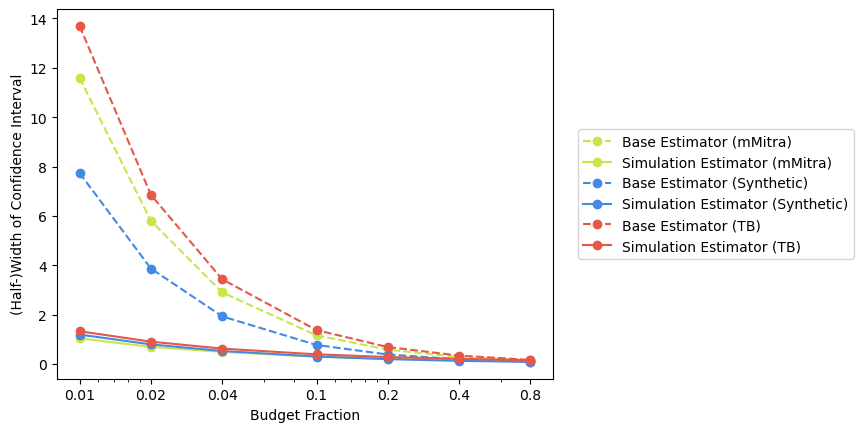}
        \caption{Power of estimators when varying the budget, i.e., the number of allocated treatments.}\label{fig:budget}
    \end{subfigure}\\
    ~
    \begin{subfigure}[t]{0.5\textwidth}
        \centering
        \includegraphics[width=0.9\textwidth]{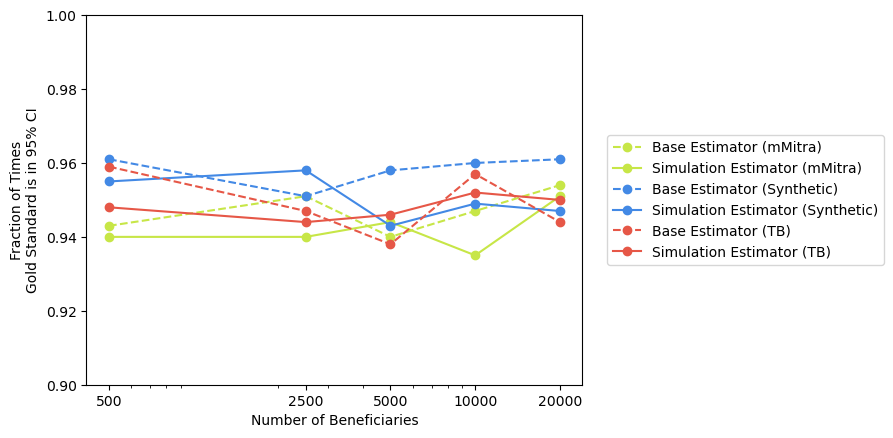}
        \caption{Validity of confidence interval when varying the number $n$ of agents.}\label{power:n1}
    \end{subfigure}%
    ~ 
    \begin{subfigure}[t]{0.5\textwidth}
        \centering
        \includegraphics[width=0.9\textwidth]{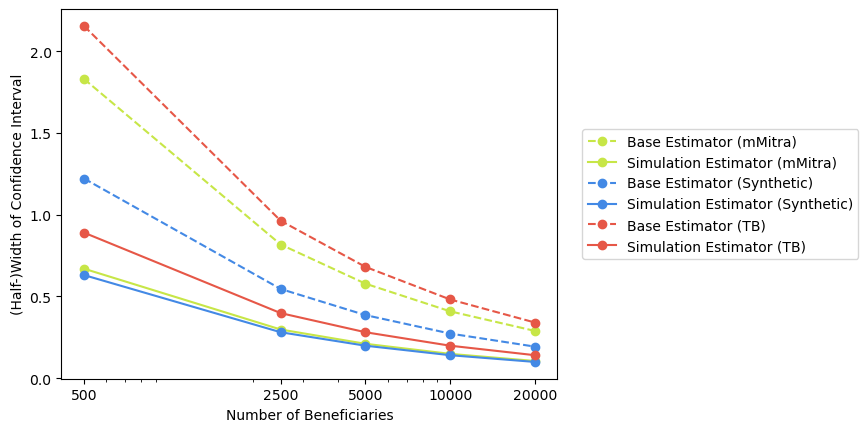}
        \caption{Power of estimators when varying the number $n$ of agents.}\label{power:n}
    \end{subfigure}\\
    ~
    \begin{subfigure}[t]{0.5\textwidth}
        \centering
        \includegraphics[width=0.9\textwidth]{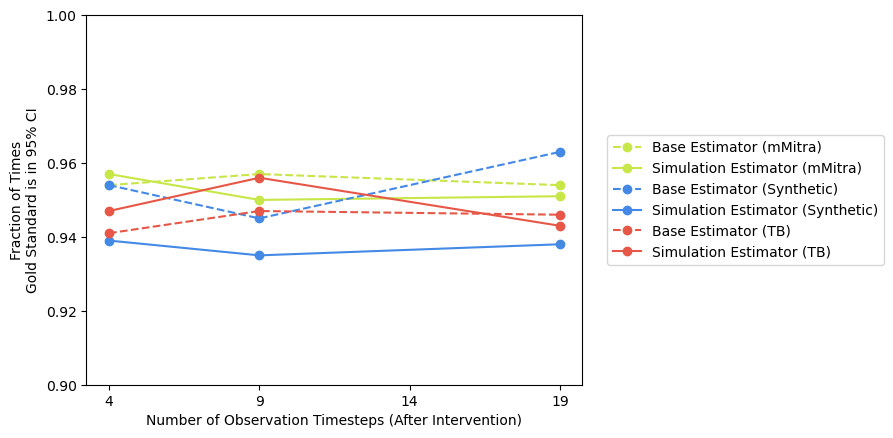}
        \caption{Validity of confidence interval when varying the observation horizon, i.e., the number of timesteps over which we observe agents and accumulate after the initial treatment allocation.}\label{fig:O}
    \end{subfigure}%
    ~ 
    \begin{subfigure}[t]{0.5\textwidth}
        \centering
        \includegraphics[width=0.9\textwidth]{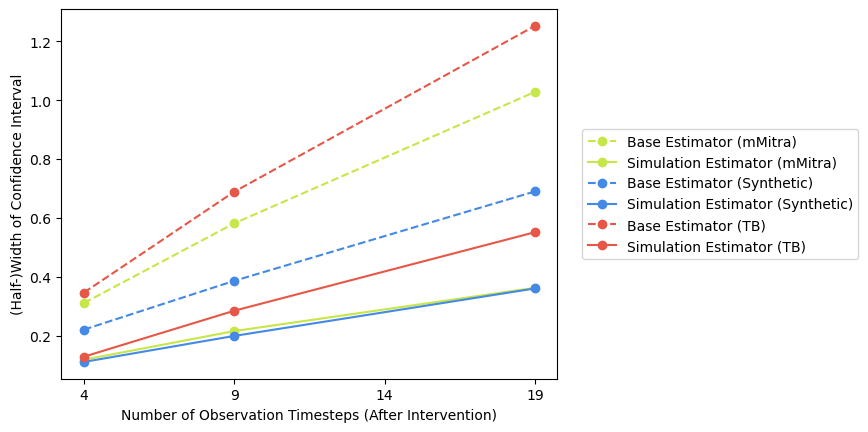}
        \caption{Power of estimators when varying the observation horizon.}\label{fig:O2}
    \end{subfigure}
    \caption{Empirical comparison of the confidence intervals produced by different estimators when varying hyperparameters. On the left, we analyze validity by showing the fraction of times the estimand falls in an estimator's $95\%$ confidence interval (the closer to $95\%$ the better). On the right, we analyze the power of estimators by depicting the half-width of computed confidence intervals (the smaller the better).}\label{hype:1}
\end{figure*}

\FloatBarrier

\section{Additional Material for Section  \ref{sec:extMe}} \label{app:extensions-meth}\label{app:extensions}

\begin{table}[H]
    \centering
    \begin{tabular}{ccccccccccc}
\toprule
\multirow{2}{*}{Category} & \multirow{2}{*}{Estimator} & \multicolumn{3}{c}{TB} & \multicolumn{3}{c}{Synthetic} & \multicolumn{3}{c}{mMitra} \\
\cmidrule(lr){3-5} \cmidrule(lr){6-8} \cmidrule(lr){9-11}
& & $<$CI & in CI & $>$CI & $<$CI & in CI & $>$CI & $<$CI & in CI & $>$CI \\
\midrule
\multirow{2}{*}{Basic} & Base & 0.024 & 0.944 & 0.032 & 0.023 & 0.957 & 0.020 & 0.028 & 0.945 & 0.027 \\
& Subgroup & 0.018 & 0.949 & 0.033 & 0.026 & 0.945 & 0.029 & 0.022 & 0.956 & 0.022 \\
\midrule
\multirow{2}{*}{\makecell{Timestep\\Truncation}} & \makecell{6 Timesteps} & 0.000 & 0.620 & 0.380 & 0.025 & 0.951 & 0.024 & 0.013 & 0.932 & 0.055 \\
& \makecell{2 Timesteps} & 0.000 & 0.000 & 1.000 & 0.007 & 0.886 & 0.107 & 0.000 & 0.013 & 0.987 \\
\midrule
\multirow{2}{*}{\makecell{Covariate Correction\\(Linear Regression)}} & Base & - & - & - & 0.068 & 0.850 & 0.082 & 0.030 & 0.940 & 0.030 \\
& \makecell{Subgroup} & - & - & - & 0.033 & 0.939 & 0.028 & 0.029 & 0.945 & 0.026 \\
\bottomrule
\end{tabular}
    \caption{Validity of Confidence Intervals. \textnormal{We measure the fraction of times that the estimand is in an estimator's 95\% confidence interval (over 1000 different simulations). We find that the base estimator, the subgroup estimator, and the subgroup estimator with covariate correction are all approximately valid. The "timestep truncation" estimators only produce estimates that are lower than the confidence intervals $\approx 0.025$ fraction of the times, and are hence valid estimators of the lower bound of the intervention effect. The base estimator with covariate correction performs poorly in both the mMitra and synthetic domains because the covariates are not strongly linearly correlated with the treatment effect.}}
    \label{tab:validity_ext}
\end{table}

\begin{table}[H]
    \centering\begin{tabular}{cccccccc}
\toprule
\multirow{2}{*}{Category} & \multirow{2}{*}{Estimator} & \multicolumn{3}{c}{Lower Bound of Estimate} & \multicolumn{3}{c}{(Half-)Width of CI} \\
\cmidrule(lr){3-5} \cmidrule(lr){6-8} 
& & TB & Synthetic & mMitra & TB & Synthetic & mMitra \\
\midrule
\multirow{2}{*}{Basic} & Base & 0.122 & -0.174 & -0.225 & 0.689 & 0.386 & 0.581 \\
& Subgroup & 0.540 & 0.015 & 0.135 & 0.284 & 0.199 & 0.215 \\
\midrule
\multirow{2}{*}{\makecell{Timestep\\Truncation}} & \makecell{6 Timesteps} & 0.477 & 0.067 & \textbf{0.165} & 0.190 & 0.147 & 0.155 \\
& \makecell{2 Timesteps} & 0.237 & \textbf{0.125} & 0.136 & 0.063 & 0.068 & 0.066 \\
\midrule
\multirow{2}{*}{\makecell{Covariate Correction\\(Linear Regression)}} & Base & - & -20.002 & -0.221 & - & 19.468 & 0.576 \\
& \makecell{Subgroup} & - & -9.917 & 0.140 & - & 10.101 & 0.211 \\
\bottomrule
\end{tabular}
    \caption{Power of Estimators. \textnormal{Timestep truncation can drastically reduce the size of the confidence intervals at the cost of underestimating intervention effects. However, we find that for two of our domains, there is typically a trade-off point where the variance reduces faster than the bias, leading to larger estimates of the lower bound of the treatment effect.}}
    \label{tab:power_ext}
\end{table}

\subsection{Covariates}

\subsubsection{Methodology} 
While the formulation of the linear regression from \Cref{sec:cov} is straightforward it is slightly less clear on which set of agents (say $N'$) to fit the regression on to recover the subgroup estimator. 
To decide this, let us consider  what happens in the absence of covariates (i.e., $m=0$): 
In this case, $k$ will be the average reward of non-treated agents from $N'$, and $\beta$ will be the average difference between the rewards of treated and non-treated agents from $N'$ \cite{wooldridge2019introductory}. 
Thus, to recover our subgroup estimator in this degenerated case, we need to fit over  the $\alpha$-fraction of agents from the policy and control arm with the lowest indices, i.e., 
$N'=\pi(\mathbf{X}^p_n,\alpha)\cup \pi(\mathbf{X}^c_n,\alpha)$.
Importantly, the intuitive approach of using the full set of available agents (i.e., $N'=N$) should not be pursued, as it leads to wrong results. In this case, $\beta$ would become the average reward difference between all agents we treated and all agents we did not treat in our RCT. This estimator does not measure our estimand anymore, since even in case our intervention had no effect, treating the agents with the highest reward under the passive action would result in a non-trivial $\beta$ value. 

\subsubsection{Experiments: Setup and Results}
For the synthetic domain, we create $|X| = 50$ covariates for an agent by left multiplying their flattened 8-dimensional transition matrix (2 start states $\times$ 2 end states $\times$ 2 actions) by a $50 \times 8$ dimensional matrix whose entries are sampled from the standard normal distribution $\mathcal{N}(0, 1)$. For the mMitra dataset, we use the actual set of covariates (e.g., age, income level, education level) associated with each beneficiary from the field trial. The list of covariates, along with summary statistics for each, is discussed in detail in the appendix of~\citet{wang2023scalable}). We find that correcting for covariates using linear regression yields slight benefits in power in the mMitra domain for the subgroup estimator. However, for the base estimator it is quite bad in the synthetic domain where the confidence intervals are no longer valid (confidence intervals that should contain the estimand 95\% of times, only cover with 85\% probability). This is because, while the underlying relationship between covariates and \textit{probabilities} is linear in this domain, there is a non-linear relationship between the probabilities and the actual rewards.

\subsection{Timestep Truncation}

Recall that the rewards of agents are determined by observing their behavior for $10$ steps. In this experiment, we still compute our estimand using this procedure. However, for the computation of our estimators, we perturb the reward function by just observing agents' behavior for $2$ (or $6$) steps. As discussed in the main body this will lead to an underestimation of intervention effects while hopefully reducing the size of confidence intervals due to reduced noise. 
The results of this experiment can be found in the middle rows of \Cref{tab:validity_ext,tab:power_ext}. 
As expected, in \Cref{tab:validity_ext}, we find that timestep truncation leads to conservative confidence intervals, which will oftentimes underestimate the intervention effect, i.e., the estimated lies above the upper bound of the confidence interval. 
On the other hand, in the right part of  \Cref{tab:power_ext} we see that shortening the observation horizon decreases the size of confidence intervals substantially. For the synthetic domain and mMitra this leads to an increase in the lower bounds of confidence intervals, implying that timestep truncation allows us to establish larger statistically significant effect sizes. For the TB domain, this turns out to be not possible.  

\subsection{Sequential Allocation} \label{app:seq}

\subsubsection{Definition}
To formally speak about the sequential setting, we need to extend our notation. 
Given a treatment fraction $\alpha$, a group size $n$, and a time horizon $T$, we assign the active action to (at most) $\lceil\alpha n\rceil$ agents in every timestep $t\in [T]$. 
We focus on the case where each agent can receive the treatment at most once. 
Accordingly, an agent $i$ is now characterized by a set of time-step dependent covariates $\mathbf{x}^t\in \mathcal{X}^T$ and a reward function $R_i: \{0,\dots, T\}\to \mathbb{R}$ that returns the total reward generated by the agent given the timestep in which we assigned them the active action ($0$ corresponds to never acting). 
We denote as $Q$ the probability function over $\mathcal{X}^T\times ( \{0,\dots, T\}\to \mathbb{R})$ from which agents are sampled i.i.d.  
At timestep $t\in [T]$, given a treatment fraction $\alpha$ and agent's covariates $(\mathbf{x}^{[1,t]}_i)_{i\in [n]}$ up until step $t$, an index based policy $\pi^{\Upsilon}$ returns the $\alpha$-fraction of agents with lowest index $\Upsilon(\mathbf{x}^{[1,t]}_i)$ to which the policy has not assigned an active action in one of the previous timesteps.

To evaluate such a sequential policy, we assume that we have access to an RCT where agents in the policy arm are assigned treatment according to the policy that is tested. 
We again have a policy arm (p) containing $n$ agents $(\check{\mathbf{x}}^{[1,T]}_i,\check{R}_i)_{i\in [n]}$ sampled i.i.d. from $Q$ on which we run our policy $\pi$. As the outcome, we observe $(\check{\mathbf{x}}^{[1,T]}_i,\check{R}_i(\check{J}_i))_{i\in [n]}$, where $\check{J}_i$ is the time step in which the policy $\pi$ assigns $i$ an active action given the covariates $(\check{\mathbf{x}}^{[1,T]}_i)_{i\in [n]}$ of all agents (and $0$ if the policy never assigns an action to the agent). 
Moreover, we have access to a control arm (c) of $n$ agents $(\hat{\mathbf{x}}^{[1,T]}_i,\hat{R}_i)_{i\in [n]}$ sampled i.i.d. from $Q$ for which we observe $(\hat{\mathbf{x}}^{[1,T]}_i,\hat{R}_i(0))_{i\in [n]}$.

\subsubsection{Methodology}

The definition of our estimand $\tau^{\mathrm{new}}_{n,\alpha}$ changes in the sequential setting to: 
\begin{align}
    & \tau^{T}_{n,\alpha}(\pi):= \frac{1}{T\lceil\alpha n\rceil}\mathbb{E} \!\!\!\!\!\!\!\! \sum_{\substack{t\in [T]\\i\in \pi\left((\mathbf{x}^{[1,t]}_i)_{i\in [n]},\alpha\right)}}  \!\!\!\!\!\!\!\![R_i(t)-R_i(0)], 
\end{align}
where the expectation ranges over ${(\mathbf{x}^[1,T]_i,R_i)_{i\in [n]}\sim Q}$.
Moreover, the base estimator in the sequential setting becomes: 
$$\frac{1}{T\lceil\alpha n\rceil}\left(\sum_{i\in [n]} \check{R}_i(\check{J}_i)-\sum_{i\in [n]} \hat{R}_i(0)\right).$$
The subgroup estimator is: 
$$\frac{1}{T\lceil\alpha n\rceil}\left(\sum_{\substack{t\in [T]\\i\in \pi\left((\check{\mathbf{x}}^{[1,t]}_i)_{i\in [n]},\alpha\right)}}\!\!\!\!\! \check{R}_i(\check{J}_i)- \!\!\!\!\!\sum_{\substack{t\in [T]\\i\in \pi\left((\hat{\mathbf{x}}^{[1,t]}_i)_{i\in [n]},\alpha\right)}} \!\!\!\!\! \hat{R}_i(0)\right).$$
The linear regression approach that corrects for covariates extends in a straightforward fashion. For the subgroup estimator, we get: 
$$R_i(J_i)=k+\beta J_i+ \sum_{t=1}^m \gamma_{t}x_{i,t}+\epsilon_i,$$ where $J_i$ indicates whether agent $i$ has received a treatment in one of the timesteps and fit it over $\{\pi((\check{\mathbf{x}}^{[1,t]}_i)_{i\in [n]},\alpha)\mid t\in [T]\}\cup \{\pi((\hat{\mathbf{x}}^{[1,t]}_i)_{i\in [n]},\alpha)\mid t\in [T]\}$. 
For the base estimator, we use 
$$R_i(J_i)=k+\beta I_i+ \sum_{t=1}^m \gamma_{t}x_{i,t}+\epsilon_i,$$ where $I_i$ is one if agent $i$ belongs to the policy arm and zero if it belongs to the treatment arm and fit it over the full agent set.

\subsubsection{Experiments}
Note that the setup of our experiment for the sequential setting is very similar to the one for the experiments in the main body. 
The main difference is that while we still observe agents for ten time steps, here we allocate resources over the first $x$ of these steps. 
Thus, the reward of an agent is still their summed behavior over ten timesteps, but they might receive treatment in say the second or third of these ten steps. 

\Cref{fig:multi} shows results for the sequential setting where we vary the number of timesteps over which $500$ treatment resources are distributed. For $x=1$, all treatments are allocated in one timestep, whereas for $x=5$, we allocate $100$ resources in each of the first five timesteps.
We find that the validity of confidence intervals remains largely unaffected when we allocate resources over multiple (instead of just one) rounds. 
In terms of statistical power, the half-width of the confidence interval of the base estimator does not change when distributing resources over multiple rounds. 
For the subgroup estimator, the size slowly decreases, which is in line with the estimand that also decreases because interventions that are allocated in later rounds will have less of an effect. 
Consequently, the subgroup estimator outperforms the base estimator even more in the sequential setting than in the single-round one.

\begin{figure*}
    \begin{subfigure}[t]{0.49\textwidth}
        \centering
        \includegraphics[width=\textwidth]{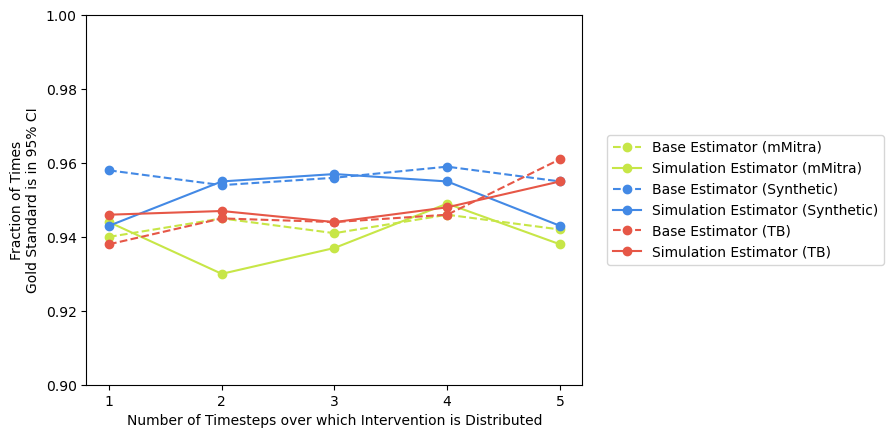}
        \caption{Validity of confidence interval for sequential allocation with varying planning horizons.}
    \end{subfigure}%
    \hfill
    \begin{subfigure}[t]{0.49\textwidth}
        \centering
        \includegraphics[width=\textwidth]{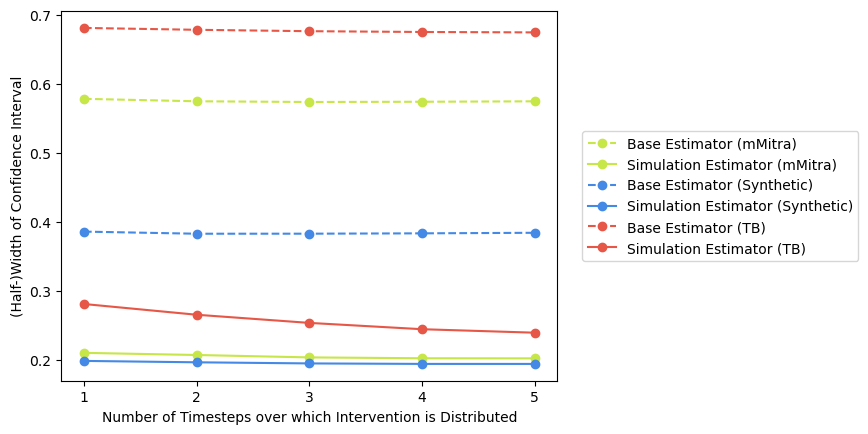}
        \caption{Power of estimators for sequential allocation with varying planning horizons.}
    \end{subfigure}
    \caption{Empirical comparison of the confidence intervals produced by different estimators if resources are allocated over multiple rounds. The $x$-axis value denotes the number $T$ of rounds over which treatment is allocated. In each round we allocate treatment to $\lceil\frac{\alpha}{T} n\rceil$ agents for $\alpha=0.1$ and $n=5000$.   On the left, we analyze validity by showing the fraction of times the estimand falls in an estimator's $95\%$ confidence interval (the closer to $95\%$ the better). On the right, we analyze the power of estimators by depicting the half-width of computed confidence intervals (the smaller the better).}\label{fig:multi}
\end{figure*}

\section{Additional Material for Section \ref{hy:Base}: General Results on Bivariate Distributions} \label{proofsec:1}
In this and the next two sections, we will prove \Cref{ourtheorem}. 
In this section, we prove a result that works for general bivariate distribution (independent of our notation of indices and rewards). 
Thus, we simplify our notation as follows:
For each $n \in \mathbb{N}$, we have \[(W_{i,n}, Z_{i,n}) \overset{iid}{\sim} P\] for some bivariate probability distribution $P$ over $\mathcal{W} \times \mathcal{Z}$.
Let $F_W(x)=\mathbb{P}(W\leq x)$ denote $W$'s marginal cdf and let $F_W^{-1}$ denote $W$'s quantile function.
Let $q_\alpha:=F_W^{-1}(\alpha)$ denote $F_W$'s $\alpha$-quantile and define the event \[E_{i,n} := \{W_{i,n} \leq q_\alpha\}.\] Similarly, define \[F_{i,n} := \{Q_{\mathcal{W}_n}(W_{i,n})\leq \lceil\alpha n\rceil\},\] where $Q_{\mathcal{W}_n}(W_{i,n})$ denotes the rank of $W_{i,n}$ among $\mathcal{W}_n := \{W_{1,n}, \ldots, W_{n,n}\}$.
Further, let \[f(t) := \mathbb{E}[Z_{1,n}|W_{1,n} \leq t] \]  $Z$'s expected value conditioned on $W\leq t$.
Moreover, define \[\psi(t) := f(F_W^{-1}(t))\] to be $Z$'s expected value conditioned on $W$ being in the $t$-quantile.

For any integrable (measurable) function $g: \mathcal{W} \times \mathcal{Z} \rightarrow \mathbb{R}$, let $\mathbb{P}_n g := \frac{1}{n}\sum_{i=1}^n g(W_{i,n}, Z_{i,n})$ and $Pg := \int_{\mathbb{R}} g(w, z)dP(w,z)$.
We let $f_t(w,z) := zI[w \leq t]$ and $\varphi: t \mapsto Pf_t$, i.e.,  $\varphi(t) = \mathbb{E}[ZI[W \leq t]]$.

Let \[(W_{(1),n}, Z_{(1),n}), \ldots, (W_{(n),n}, Z_{(n),n})\] denote the sequence of pairs $(W_{i,n}, Z_{i,n})$ sorted in increasing order of the $W_{i,n}$ (i.e., so that $W_{(i),n} \leq W_{(j), n}$ for $i \leq j$). 

Our results from this section rely on the following two assumptions:

\begin{assumption}\label{as:positive-derivative}
    $F_W$ has a positive derivative at $q_\alpha$. 
\end{assumption}

\begin{assumption}\label{second-moment}
    $Z$ has a second moment.
\end{assumption}

We show the following: 
\begin{theorem}\label{theorem4}
Let $\tilde{\mu} = \mathbb{E}[Z_{i,n}I_{E_{i,n}}]$ and $\tilde{\sigma}^2 = \var(Z_iI[W_i \leq q_\alpha])$.
Under \Cref{as:positive-derivative,second-moment},
    \begin{equation}
    \sqrt{n}\left(\frac{1}{n}\sum_{i=1}^n Z_{i,n}I_{F_{i,n}} -  \tilde{\mu} \right) \overset{d}{\rightarrow} \mathcal{N}(0, \tilde{\sigma}^2 - 2\mathbb{E}[Z|W=q_\alpha]\tilde{\mu}(1-\alpha) + \mathbb{E}[Z|W=q_\alpha]^2\alpha(1-\alpha)).
\end{equation}
\end{theorem}

This section (and its notation) follows very closely the notes in Example 1.5 of \cite{sen2018gentle}.
Recall that for any integrable (measurable) function $f: \mathcal{W} \times \mathcal{Z} \rightarrow \mathbb{R}$, let $\mathbb{P}_n f := \frac{1}{n}\sum_{i=1}^n f(W_{i,n}, Z_{i,n})$ and $Pf := \int_{\mathbb{R}} f(w, z)dP(w,z)$.
We let $f_t(w,z) := zI[w \leq t]$ and $\varphi: t \mapsto Pf_t$, i.e.,  $\varphi(t) = \mathbb{E}[ZI[W \leq t]]$.
We observe that \begin{align}\label{main}
    \sqrt{n}\left(\frac{1}{n}\sum_{i=1}^n Z_{i,n}I_{F_{i,n}} - \tilde{\mu}\right) &= \sqrt{n}(\mathbb{P}_nf_{W_{(\lceil \alpha n\rceil),n}} - Pf_{q_\alpha}) \nonumber \\ 
    &= \sqrt{n}(\mathbb{P}_n - P)f_{q_\alpha} + \sqrt{n}(\mathbb{P}_n f_{W_{(\lceil \alpha n\rceil),n}} - \mathbb{P}_n f_{q_\alpha}) \nonumber\\ 
    &= \mathbb{G}_n[f_{q_\alpha}] +\mathbb{G}_n[f_{W_{(\lceil \alpha n\rceil),n}}- f_{q_\alpha}] + \sqrt{n}(P f_{W_{(\lceil \alpha n\rceil),n}} - P f_{q_\alpha}) \nonumber \\ 
    &= \mathbb{G}_n[f_{q_\alpha}] +\mathbb{G}_n[f_{W_{(\lceil \alpha n\rceil),n}}- f_{q_\alpha}] + \sqrt{n}(\varphi(W_{(\lceil \alpha n\rceil),n}) - \varphi(q_\alpha))
\end{align} where $\mathbb{G}_n$ denotes the empirical process (indexed by functions $f \in \mathcal{F} := \{f_t: t \in \mathbb{R}\}$) equal to $\sqrt{n}(\mathbb{P}_n - P)$.

The delta method allows us to easily handle the third term:
\begin{lemma}\label{second-term}
We have \[\sqrt{n}(\varphi(W_{(\lceil \alpha n\rceil),n}) - \varphi(q_\alpha)) = \varphi'(q_\alpha)\sqrt{n}(W_{(\lceil \alpha n\rceil),n} - q_\alpha) + o_p(1).\]
\end{lemma}
\begin{proof}
    The delta method simply requires that $\varphi$ is differentiable at $q_\alpha$, which we verify at the end of this subsection (using \Cref{as:positive-derivative}), just before the beginning of \Cref{sub:zero-cov}.
\end{proof}

Using empirical process theory in \Cref{sub:zero-cov} we show that the second term goes to $0$ in probability:

\begin{restatable}{lemma}{zeroconv}\label{zero-conv}
    \[\mathbb{G}_n[f_{W_{(\lceil \alpha n\rceil),n}}- f_{q_\alpha}] \overset{p}{\rightarrow} 0\]
\end{restatable}

By \eqref{main} as well as \Cref{second-term,zero-conv}, we have that \begin{equation}\label{emp-done}\sqrt{n}\left(\frac{1}{n}\sum_{i=1}^n Z_{i,n}I_{F_{i,n}} - \tilde{\mu}\right) = \mathbb{G}_n[f_{q_\alpha}] + \varphi'(q_\alpha)\sqrt{n}(W_{(\lceil \alpha n\rceil),n} - q_\alpha)  + o_p(1)\end{equation}
\[= \sqrt{n}(\frac{1}{n}\sum_{i=1}^nZ_{i,n} I[W_i \leq q_\alpha] - \tilde{\mu}) + \varphi'(q_\alpha)\sqrt{n}(W_{(\lceil \alpha n\rceil),n} - q_\alpha) + o_p(1)\]

Furthermore, under  \Cref{as:positive-derivative} standard results (c.f.~ \cite{van2000asymptotic} Corollary 21.5) give the following asymptotic expansion for the second term of \Cref{emp-done}: \[\sqrt{n}(W_{(\lceil \alpha n\rceil),n} - q_\alpha) = -\frac{1}{\sqrt{n}}\sum_{i=1}^n \frac{I[W_i \leq q_\alpha]-\alpha}{F_W'(q_\alpha)} + o_p(1).\]

Using multidimensional CLT we have that \[\sqrt{n}\begin{pmatrix}
\frac{1}{n}\sum_{i=1}^nZ_i I[W_i \leq q_\alpha] - \tilde{\mu}\\
-\frac{1}{n}\sum_{i=1}^n \varphi'(q_\alpha)\frac{I[W_i \leq q_\alpha]-\alpha}{F_W'(q_\alpha)}

\end{pmatrix} \overset{d}{\rightarrow} \mathcal{N}\left(\begin{pmatrix}
    0\\
    0
\end{pmatrix}, \begin{pmatrix}
    \tilde{\sigma}^2 & -\varphi'(q_\alpha)\tilde{\mu}(1-\alpha)/F'_W(q_\alpha)\\
    -\varphi'(q_\alpha)\tilde{\mu}(1-\alpha)/F'_W(q_\alpha) & \varphi'(q_\alpha)^2\alpha(1-\alpha)/F'_W(q_\alpha){}^2
\end{pmatrix}\right)\] 

So, using continuous mapping theorem, we can conclude that \[\sqrt{n}(\frac{1}{n}\sum_{i=1}^nZ_i I[W_i \leq q_\alpha] - \tilde{\mu}) -\frac{1}{\sqrt{n}}\sum_{i=1}^n \frac{I[W_i \leq q_\alpha]-\alpha}{F_W'(q_\alpha)}) \overset{d}{\rightarrow} \mathcal{N}(0, \tilde{\sigma}^2 - 2\varphi'(q_\alpha)\tilde{\mu}(1-\alpha)/F'_W(q_\alpha) + \varphi'(q_\alpha)^2\alpha(1-\alpha)/F'_W(q_\alpha){}^2\] and hence 

\begin{equation}\label{eq:4}
   \sqrt{n}\left(\frac{1}{n}\sum_{i=1}^n Z_{i,n}I_{F_{i,n}} - \tilde{\mu}\right) \overset{d}{\rightarrow} \mathcal{N}(0, \tilde{\sigma}^2 - 2\varphi'(q_\alpha)\tilde{\mu}(1-\alpha)/F'_W(q_\alpha) + \varphi'(q_\alpha)^2\alpha(1-\alpha)/F'_W(q_\alpha){}^2) 
\end{equation}

Finally, let us obtain a reasonable form for $\varphi'(q_\alpha)$. Recall that $\varphi(t) := \mathbb{E}[ZI[W \leq t]]$. Writing the expectation as a Riemann-Stieltjies integral, we find that \[\mathbb{E}[ZI[W \leq t]] = \mathbb{E}[I[W \leq t]\mathbb{E}[Z|W]] = \int_{-\infty}^t \mathbb{E}[Z|W=w]dF_W(w).\] The Fundamental Theorem of Calculus for Riemann-Stieltjies integrals (cf.~Theorem 7.32 (iii) of \cite{apostol1974mathematical}) combined with \Cref{as:positive-derivative} then implies that the derivative of the above, at $q_\alpha$, is $\mathbb{E}[Z|W=q_\alpha]F_W'(q_\alpha)$. Plugging this into \Cref{eq:4}, \Cref{theorem4} follows.
It remains to prove \Cref{zero-conv}. 

\subsection{Proof of \Cref{zero-conv}} \label{sub:zero-cov}

To prove \Cref{zero-conv}, we first recall some basic terminology from \cite{vaart2023empirical} for ease of readability.

\begin{definition}[VC dimension]
    Let $\mathcal{C}$ be a collection of subsets of $\mathcal{X}$. The VC dimension of $\mathcal{C}$, is defined as \[\vcd(\mathcal{C}) := \max\{n \in \mathbb{N}: \exists S \subseteq \mathcal{X} \text{ with } |S| = n \text{ such that }S \text{ is shattered by } \mathcal{C}\}\] where we say that a set $S$ is shattered by $\mathcal{C}$ if the power set of $S$ is contained in $\{C \cap S: C \in \mathcal{C}\}$.
\end{definition}

\begin{definition}[Subgraph of a function; cf Page 141 of \cite{vaart2023empirical}]
    Let $f: \mathcal{X} \rightarrow \mathbb{R}$. The subgraph of $f$ is defined as the set \[\{(x,s) \in \mathcal{X} \times \mathbb{R}: s< f(x)\}.\]
\end{definition}

\begin{definition}[VC subgraph; cf Page 141 of \cite{vaart2023empirical}]
    Let $\mathcal{F}$ be a class of functions from $\mathcal{X} \rightarrow \mathbb{R}$ and let $\mathcal{C}$ be the associated class of subgraphs of elements of $\mathcal{F}$. The class $\mathcal{F}$ is said to be a VC-subgraph class if $\vcd(\mathcal{C}) < \infty$.
\end{definition}

We now show a standard fact:

\begin{lemma}\label{vc1}
    Let $g_t: \mathcal{W} \times \mathcal{Z} \rightarrow \mathbb{R}$ given by $g_t(w,z) = I[w \leq t]$. Then $\mathcal{G} := \{g_t: t \in \mathbb{R}\}$ is a VC-subgraph class.
\end{lemma}
\begin{proof}
    Let $\mathcal{C}$ be the class of subgraphs of elements of $\mathcal{G}$. We claim that no three points \[S = \{(w_1, z_1, s_1), (w_2, z_2, s_2), (w_3, z_3, s_3)\}\] can be shattered by $\mathcal{C}$ (and hence $\vcd(\mathcal{C}) < 3$). To see why, suppose, without loss of generality, that $w_1 \leq w_2 \leq w_3$. Also, observe that we need only consider $s_i \in [0,1)$ since otherwise the point $(w_i, z_i, s_i)$ could only be labeled in one way. But in this case, notice that there exists no  $C \in \mathcal{C}$ for which $C \cap S = \{(w_1, z_1, s_1), (w_3, z_3, s_3)\}$ since otherwise there would exist some $t$ for which $w_1 \leq t, w_3 \leq t$ but $w_2 > t$, contradicting the fact that $w_1 \leq w_2 \leq w_3$.
    
\end{proof}

Lemma~\ref{vc1} combined with another standard fact shows that $\mathcal{F} = \{f_t: t \in \mathbb{R}\}$ is a VC-subgraph class:
\begin{lemma}\label{vc2}
    The class of functions $\mathcal{F} = \{f_t: t \in \mathbb{R}\}$ is a VC-subgraph class.
\end{lemma}
\begin{proof}
    Lemma 2.6.18 of \cite{vaart2023empirical} part (vi) tells us that $\{fg: g \in \mathcal{G}\}$ is a VC-subgraph class so long as the class $\mathcal{G}$ is a VC-subgraph class. Observing that $\mathcal{F} = \{fg: g \in \mathcal{G}\}$ with $\mathcal{G}$ defined as above and $f(w,z) = z$, and using the fact that $\mathcal{G}$ is a VC-subgraph class from Lemma~\ref{vc1} allows us to conclude the result.
\end{proof}

We now recall a bit more terminology.
\begin{definition}[Envelope function]
A measurable function $F: \mathcal{W} \times \mathcal{Z} \rightarrow \mathbb{R}$ is said to be an envelope function for a function class $\mathcal{F}$ if $|f| \leq F$ for all $f \in \mathcal{F}$.
\end{definition}

\begin{definition}[$P$-measurability; cf Definition 2.3.3 of \cite{vaart2023empirical}]
    A set $\mathcal{F}$ of functions, $f: \mathcal{X} \rightarrow \mathbb{R}$ on $(\mathcal{X}, \mathcal{A}, P)$ is called $P$-measurable if the map \[(X_1, \ldots, X_n) \mapsto \|\sum_{i=1}^n e_if(X_i)\|_{\mathcal{F}},\] where $\|f(\cdot)\|_{\mathcal{F}}$ means $\sup_{f\in \mathcal{F} }|f(\cdot)|$, is measurable on the completion of $(\mathcal{X}^n, \mathcal{A}^n, P^n)$ for every $n$ and every vector $(e_1, \ldots, e_n) \in \mathbb{R}^n$.
\end{definition}

Now, we define covering number and Donsker class:

\begin{definition}[Uniform entropy bound; c.f.~\cite{vaart2023empirical} Page 127]
    A class of functions $\mathcal{F}$ is said to satisfy the uniform entropy bound if \[\int_0^\infty \sup_Q\sqrt{\log N(\epsilon \|F\|_{Q,2}, \mathcal{F}, L_2(Q))}d\epsilon < \infty.\]
\end{definition}

\begin{definition}[$P$-Donsker; c.f.~\cite{vaart2023empirical} page 81]
    A class of functions $\mathcal{F}$ for which the empirical process $\sqrt{n}(\mathbb{P}_n - P)$, indexed by $\mathcal{F}$, converges weakly in $\ell^\infty(\mathcal{F})$ to a tight Borel measurable element $\mathbb{G}$ in $\ell^\infty(\mathcal{F})$ is said to be $P$-Donsker.
\end{definition}

Now, a theorem from \cite{vaart2023empirical}:
\begin{theorem}[Theorem 2.5.2 of \cite{vaart2023empirical}]
    Let $\mathcal{F}$ be a class of functions satisfying the uniform entropy bound. Furthermore, suppose that the classes \[\{f-g: f,g \in \mathcal{F}, \|f-g\|_{P,2} < \delta\}\] and \[\{(f-g) \cdot (f'-g'): f,g,f',g' \in \mathcal{F}\}\] are $P$-measurable for every $\delta > 0$. If the envelope function $F$ for $\mathcal{F}$ is square integrable, then $\mathcal{F}$ is $P$-Donsker.
\end{theorem}

This theorem (and some of the previous lemmata) easily allows us to conclude that:
\begin{lemma}\label{donsker}
    $\mathcal{F} = \{f_t: t \in \mathbb{R}\}$ is $P$-Donsker.
\end{lemma}
\begin{proof}
    First, we show the even stronger condition that $\{f_t - f_s: f_t, f_s \in \mathcal{F}\}$ is $P$-measurable. Consider the map \begin{align*}((W_1, Z_1), \ldots, (W_n,Z_n)) &\mapsto \sup_{t,s}\Big|\sum_{i=1}^ne_if_t(W_i, Z_i)-e_if_s(W_i,Z_i)\Big|\\
    &= \sup_{t,s}|\sum_{i \in [n]: W_i \in (s,t]}e_iZ_i|\end{align*} Clearly the supremum can be replaced by a supremum over $t$ and $s$ in the rationals. Since a countable supremum of measurable functions (which the inside of the supremum clearly is) is measurable, the result is measurable. The same argument shows \[((W_1, Z_1), \ldots, (W_n,Z_n)) \mapsto \sup_{s,t,s',t'}\Big|\sum_{i=1}^ne_i(f_t(W_i, Z_i) - f_s(W_i, Z_i))(f_{t'}(W_i, Z_i) - f_{s'}(W_i, Z_i))\Big|\] is measurable.

    Now, observe that $|f(w_i, z_i)| \leq |w_i|$ and $W$'s having second moment immediately implies that the envelope function $F(w,z) = w$ is square-integrable.

    Finally, notice that $\log N(\epsilon \|F\|_{Q,2}, \mathcal{F}, L_2(Q)) = 0$ for $\epsilon \geq 1$, clearly. And hence we must show that \[\int_0^1\sup_Q\sqrt{\log N(\epsilon \|F\|_{Q,2}, \mathcal{F}, L_2(Q))}d\epsilon < \infty.\]

    Theorem 2.6.7 of \cite{vaart2023empirical}, combined with our observation that $\mathcal{F}$ is a VC class, says that $\sqrt{\log N(\epsilon \|F\|_{Q,2}, \mathcal{F}, L_2(Q))} \leq O(\sqrt{\log(1/\epsilon)})$ and it is indeed true that $\int_0^1 \sqrt{\log(1/\epsilon)} < \infty$, as desired.
\end{proof}

We now recall the definition of asymptotic equicontinuity:
\begin{definition}[\cite{vaart2023empirical} page 89]
    Define the seminorm $\rho_f(f) = (P(f - Pf)^2)^{1/2}$. Then we say that the empirical process $\mathbb{G}_n$ indexed by the function class $\mathcal{F}$ is asymptotically equicontinuous if \[\lim_{\delta \downarrow 0}\limsup_{n\rightarrow \infty}P\left(\sup_{\rho_P(f-g)<\delta}|\mathbb{G}_n(f-g)| > \epsilon\right) = 0.\]
\end{definition}

Theorem 1.5.7 of \cite{vaart2023empirical} combined with the fact that $\mathcal{F}$ is $P$-Donsker (Lemma~\ref{donsker}) then immediately implies that $\mathbb{G}_n := \sqrt{n}(\mathbb{P}_n - P)$ is uniformly equicontinuous. We are finally able to prove \Cref{zero-conv}:

\zeroconv*
\begin{proof}
    First, observe that $\rho_f(f_{W_{(\lceil \alpha n\rceil),n}} - f_{q_\alpha}) = o_p(1)$. To see why, we have that \begin{align*}
        (\rho_f(f_{W_{(\lceil \alpha n\rceil),n}} - f_{q_\alpha}))^2 &\leq \int (z(I[w \leq W_{(\lceil \alpha n\rceil),n}] - I[w \leq q_\alpha])^2dP(w,z)\\
            &=\int z^2|I[w \leq W_{(\lceil \alpha n\rceil),n}] - I[w \leq q_\alpha]|dP(w,z)
    \end{align*} which easily converges to zero in probability:
    Fix any $\delta > 0$
    \begin{align}
        &\int z^2|I[w \leq W_{(\lceil \alpha n\rceil),n}] - I[w \leq q_\alpha]|dP(w,z) \nonumber\\
        &= \int z^2|I[w \leq W_{(\lceil \alpha n\rceil),n}] - I[w \leq q_\alpha]|I[|w-q_\alpha| < \delta]dP(w,z) \nonumber\\
        &+ \int z^2|I[w \leq W_{(\lceil \alpha n\rceil),n}] - I[w \leq q_\alpha]|I[|w-q_\alpha| \geq \delta]dP(w,z) \nonumber\\
        &\leq 2\int z^2I[|w-q_\alpha| < \delta]dP(w,z) + \int z^2|I[w \leq W_{(\lceil \alpha n\rceil),n}] - I[w \leq q_\alpha]|I[|w-q_\alpha| \geq \delta]dP(w,z) \label{integral-eq}
    \end{align}

    Notice that \[z^2|I[w \leq W_{(\lceil \alpha n\rceil),n}] - I[w \leq q_\alpha]|I[|w-q_\alpha| \geq \delta] \overset{a.s.}{\rightarrow} 0\] for both $w=q_\alpha - \delta$ and $w=q_\alpha+\delta$ since $W_{(\lceil \alpha n\rceil),n}\overset{a.s.}{\rightarrow} q_\alpha$. Letting \begin{align*}A &= \{\lim_{n \rightarrow \infty}z^2|I[q_\alpha - \delta \leq W_{(\lceil \alpha n\rceil),n}] - I[q_\alpha - \delta \leq q_\alpha]|I[|q_\alpha - \delta-q_\alpha| \geq \delta]=0\}\\& \cap \{\lim_{n \rightarrow \infty}z^2|I[q_\alpha + \delta \leq W_{(\lceil \alpha n\rceil),n}] - I[q_\alpha + \delta \leq q_\alpha]|I[|q_\alpha + \delta-q_\alpha| \geq \delta]=0\}\end{align*}
    denote the event that both convergences occur, a union bound tells us that $P(A) = 1$. Now, notice that for any $w' \leq q_\alpha - \delta$ we have that \begin{align*}
        &z^2|I[w' \leq W_{(\lceil \alpha n\rceil),n}] - I[w' \leq q_\alpha]|I[|w'-q_\alpha| \geq \delta] \\
        &=z^2(1 - I[w' \leq W_{(\lceil \alpha n\rceil),n}])\\
        &\leq z^2(1 - I[q_\alpha-\delta \leq W_{(\lceil \alpha n\rceil),n}])\\
        &= z^2|I[q_\alpha - \delta \leq W_{(\lceil \alpha n\rceil),n}] - I[q_\alpha - \delta \leq q_\alpha]|I[|q_\alpha - \delta-q_\alpha|\geq \delta]
    \end{align*}

    Similarly, for any $w' \geq q_\alpha+\delta$, we have that \begin{align*}
        &z^2|I[w' \leq W_{(\lceil \alpha n\rceil),n}] - I[w' \leq q_\alpha]|I[|w'-q_\alpha| \geq \delta] \\
        &=z^2I[w' \leq W_{(\lceil \alpha n\rceil),n}]\\
        &\leq z^2I[q_\alpha+\delta \leq W_{(\lceil \alpha n\rceil),n}]\\
        &\leq z^2|I[q_\alpha + \delta \leq W_{(\lceil \alpha n\rceil),n}] - I[q_\alpha + \delta \leq q_\alpha]|I[|q_\alpha + \delta-q_\alpha|\geq \delta]
    \end{align*}

    Hence, on the event $A$, we have that $z^2|I[w \leq W_{(\lceil \alpha n\rceil),n}] - I[w \leq q_\alpha]|I[|w-q_\alpha| \geq \delta] \rightarrow 0$ for all $(w,z)$ so that \[P\left(z^2|I[w \leq W_{(\lceil \alpha n\rceil),n}] - I[w \leq q_\alpha]|I[|w-q_\alpha| \geq \delta] \rightarrow 0, \forall (w,z)\right) = 1.\]
    
    Hence, dominated convergence theorem implies that \[\int z^2|I[w \leq W_{(\lceil \alpha n\rceil),n}] - I[w \leq q_\alpha]|I[|w-q_\alpha| > \delta]dP(w,z) \overset{a.s.}{\rightarrow} 0,\] Using this, from \Cref{integral-eq} we get that:  \[\limsup_n \int z^2|I[w \leq W_{(\lceil \alpha n\rceil),n}] - I[w \leq q_\alpha]|dP(w,z) \overset{a.s.}{\leq} 2\int z^2I[|w-q_\alpha| \leq \delta]dP(w,z).\]

    But notice that $\lim_{\delta \downarrow 0}[|w-q_\alpha| \leq \delta]= 0$ for almost every $w$ (in view of \Cref{as:positive-derivative}) and hence dominated convergence in turn tells us that \[\lim_{\delta \downarrow 0}2\int z^2I[|w-q_\alpha| \leq \delta]dP(w,z) = 0,\] and hence indeed $$\rho_f(f_{W_{(\lceil \alpha n\rceil),n}} - f_{q_\alpha}) = o_p(1),$$ as desired.

    Then we have that, for any $\epsilon, \delta > 0$, \begin{align*}
        &P(|\mathbb{G}_n[f_{W_{(\lceil \alpha n\rceil),n}}- f_{q_\alpha}]| > \epsilon)\\
        &= P(|\mathbb{G}_n[f_{W_{(\lceil \alpha n\rceil),n}}- f_{q_\alpha}]| > \epsilon, \rho_f(f_{W_{(\lceil \alpha n\rceil),n}} - f_{q_\alpha}) > \delta) + P(|\mathbb{G}_n[f_{W_{(\lceil \alpha n\rceil),n}}- f_{q_\alpha}]| > \epsilon, \rho_p(f_{W_{(\lceil \alpha n\rceil),n}} - f_{q_\alpha}) \leq \delta)
    \end{align*}
    The first term goes to zero by the above and the second term is at most \[P\left(\sup_{\rho_P(f-g)<\delta}|\mathbb{G}_n(f-g)| > \epsilon\right).\] Taking $n \rightarrow \infty$, then $\delta \downarrow 0$ and using the definition of uniform equicontinuity then gives the desired result.
\end{proof}

\section{Additional Material for Section \ref{hy:Base}: Main Result}\label{app:Base}

We now prove the asymptotic normality of the base estimator (and afterward present an alternative proof to the one in \cite{imai2023statistical} for the asymptotic normality of the subgroup estimator). 
For ease of presentation, we slightly adjust our notation from the main body. 
First, we characterize agents directly by their indices $W$ (instead of their covariates from which their indices are computed). 
Accordingly, we let $P'$ be an adjusted variant of the probability distribution $P$ from the main body defined over the space of indices $\mathbb{R}$ and reward functions $A\to \mathbb{R}$. 
We write $(W_{i,n},R_{i,n})\sim P'$ to denote a set of $n$ agents being sampled i.i.d. from the probability distribution $P'$.

In the policy group, we observe $(W_{i,n}, R_{i,n}(J_{i,n}))$ where $J_{i,n}$ is the binary treatment indicator variable of agent $i$. 
In the control group, we observe  $(W^0_{i,n}, R^0_{i,n}(0))$. 
For simplicity, we will sometimes write for agents in the policy group $R_{i,n}$ to denote the outcome of the reward function that we observe and $R_{i,n}^0$ for the agents in the control group.

Let \[F_{i,n} := \{Q_{\mathcal{W}_n}(W_{i,n})\leq \lceil\alpha n\rceil\},\] where $Q_{\mathcal{W}_n}(W_{i,n})$ denotes the rank of $W_{i,n}$ among $\mathcal{W}_n := \{W_{1,n}, \ldots, W_{n,n}\}$ and let $I_{F_{i,n}}$ denote the corresponding indicator variable, i.e., $I_{F_{i,n}}$ is $1$ if $W_{i,n}$ is among the $\lceil\alpha n\rceil$ lowest indices of the $n$ agents in the policy group, i.e., it receives a treatment. 
Analogously, let \[F^0_{i,n} := \{Q_{\mathcal{W}^0_n}(W^0_{i,n})\leq \lceil\alpha n\rceil\},\] where $Q_{\mathcal{W}^0_n}(W^0_{i,n})$ denotes the rank of $W^0_{i,n}$ among $\mathcal{W}^0_n := \{W^0_{1,n}, \ldots, W^0_{n,n}\}$. Similarly, define \[E_{i,n} := \{W_{i,n} \leq q_\alpha\}\] and  \[E^0_{i,n} := \{W^0_{i,n} \leq q_\alpha\}.\]

We define $\tau_n:= \mathbb{E}[R_{1,n}(1)I_{F_{1,n}}]- \mathbb{E}[R_{1,n}(0)I_{F_{1,n}}]$.

Before proceeding, we simply restate the convergence result of \cite{imai2023statistical} which shows that the difference between estimands converges at a faster-than-$\sqrt{n}$ rate
\begin{theorem}[Lemma S2 in Appendix S2 of \cite{imai2023statistical}]  \label{th:quant-vs-ranks}
    Under \Cref{regularity}, we have
    \begin{equation}
        \sqrt{n}(\tau_n - \mathbb{E}[R_{1,n}(1)I_{E_{i,n}} - R^0_{1,n}(0)I_{E_{i,n}^0}]) \rightarrow 0.
    \end{equation}
\end{theorem}

\subsection{Base estimator}
We now prove asymptotic normality for the base estimator. The original, non-rescaled version of the base estimator can be written as: 
\[\frac{1}{n}\sum_{i=1}^n(R_{i,n} - R^0_{i,n}).\]
In particular, we show that \[\frac{1}{\sqrt{n}}\sum_{i=1}^n(R_{i,n} - R^0_{i,n} - \tau_n) \overset{d}{\rightarrow} \mathcal{N}(0, \sigma^2_{\text{dm}}),\] for some $\sigma^2_{\text{dm}}$ to be specified later.

Using \Cref{th:quant-vs-ranks}, we write \begin{align}
    &\frac{1}{\sqrt{n}}\sum_{i=1}^n(R_{i,n} - R^0_{i,n} - \tau_n)\\
    &= \frac{1}{\sqrt{n}}\sum_{i=1}^n(R_{i,n}I_{F_{i,n}} - \mathbb{E}[R_{i,n}(1)I_{E_{i,n}}]) + \frac{1}{\sqrt{n}}\sum_{i=1}^n(R_{i,n}(0)(1-I_{F_{i,n}}) - \mathbb{E}[R_{i,n}(0)(1-I_{E_{i,n}})]) \\&- \frac{1}{\sqrt{n}}\sum_{i=1}^n (R^0_{i,n} - \mathbb{E}[R^0_{i,n}]) + o(1)\label{dm}
\end{align}

Under the same assumptions as \Cref{theorem4}, essentially the exact same proof of \Cref{theorem4} allows us to show that, defining  $\mu_t := \mathbb{E}[R_{i,n}(1) I[W_{i,n} \leq q_\alpha]]$, $\check{\mu}_0 := \mathbb{E}[R_{i,n}(0) I[W_{i,n} > q_\alpha]]$, \begin{align*}
    &\frac{1}{\sqrt{n}}\sum_{i=1}^n(R_{i,n}I_{F_{i,n}} - \mathbb{E}[R_{i,n}(1)I_{E_{i,n}}])\\
    &= \frac{1}{\sqrt{n}}\sum_{i=1}^n(R_{i,n}(1) I[W_{i,n} \leq q_\alpha] - \mu_t) + \mathbb{E}[R(1)|W=q_\alpha]F_W'(q_\alpha)\sqrt{n}(W_{(\lceil \alpha n\rceil),n} - q_\alpha) + o_p(1)\\
    &= \frac{1}{\sqrt{n}}\sum_{i=1}^n(R_{i,n}(1) I[W_{i,n} \leq q_\alpha] - \mu_t) -\frac{\mathbb{E}[R(1)|W=q_\alpha]F_W'(q_\alpha)}{\sqrt{n}}\sum_{i=1}^n \frac{I[W_i \leq q_\alpha]-\alpha}{F_W'(q_\alpha)} + o_p(1)\\
    &= \frac{1}{\sqrt{n}}\sum_{i=1}^n(R_{i,n}(1) I[W_{i,n} \leq q_\alpha] - \mu_t) -\frac{\mathbb{E}[R(1)|W=q_\alpha]}{\sqrt{n}}\sum_{i=1}^n (I[W_i \leq q_\alpha]-\alpha) + o_p(1)
\end{align*}

Similarly, by using a proof strategy nearly identical to \Cref{theorem4}, we obtain that 
\begin{align*}
    &\frac{1}{\sqrt{n}}\sum_{i=1}^n(R_{i,n}(0)(1-I_{F_{i,n}}) - \mathbb{E}[R_{i,n}(0)(1-I_{E_{i,n}})])\\
    &= \frac{1}{\sqrt{n}}\sum_{i=1}^n(R_{i,n}(0) I[W_{i,n} > q_\alpha] - \check{\mu}_0) - \mathbb{E}[R(0)|W=q_\alpha]F_W'(q_\alpha)\sqrt{n}(W_{(\lceil \alpha n\rceil),n} - q_\alpha) + o_p(1)\\
    &= \frac{1}{\sqrt{n}}\sum_{i=1}^n(R_{i,n}(0) I[W_{i,n} > q_\alpha] - \check{\mu}_0) +\frac{\mathbb{E}[R(0)|W=q_\alpha]F_W'(q_\alpha)}{\sqrt{n}}\sum_{i=1}^n \frac{I[W_i \leq q_\alpha]-\alpha}{F_W'(q_\alpha)} + o_p(1)\\
    &= \frac{1}{\sqrt{n}}\sum_{i=1}^n(R_{i,n}(0) I[W_{i,n} > q_\alpha] - \check{\mu}_0) +\frac{\mathbb{E}[R(0)|W=q_\alpha]}{\sqrt{n}}\sum_{i=1}^n (I[W_i \leq q_\alpha]-\alpha) + o_p(1)
\end{align*}

Hence, display~\eqref{dm} may be rewritten as 
\begin{align*}
    &\frac{1}{\sqrt{n}}\sum_{i=1}^n(R_{i,n}(1)I_{F_{i,n}} - \mathbb{E}[R_{i,n}I_{E_{i,n}}]) + \frac{1}{\sqrt{n}}\sum_{i=1}^n(R_{i,n}(0)(1-I_{F_{i,n}}) - \mathbb{E}[R_{i,n}(0)(1-I_{E_{i,n}})]) \\&- \frac{1}{\sqrt{n}}\sum_{i=1}^n (R^0_{i,n} - \mathbb{E}[R^0_{i,n}]) + o(1)\\
    &= \frac{1}{\sqrt{n}}\sum_{i=1}^n(R_{i,n}(1) I[W_{i,n} \leq q_\alpha] - \mu_t) + \frac{1}{\sqrt{n}}\sum_{i=1}^n(R_{i,n}(0) I[W_{i,n} > q_\alpha] - \check{\mu}_0)\\
    &- \frac{\mathbb{E}[R(1)|W=q_\alpha]-\mathbb{E}[R(0)|W=q_\alpha]}{\sqrt{n}}\sum_{i=1}^n (I[W_i \leq q_\alpha]-\alpha) - \frac{1}{\sqrt{n}}\sum_{i=1}^n (R^0_{i,n} - \mathbb{E}[R^0_{i,n}]) + o_p(1)
\end{align*}

Note that the last term is independent of the first three.
By the CLT, the last term converges in distribution to $\mathcal{N}(0,\mathrm{Var}(R(0)))$. 
To obtain the limiting distribution for the first three terms, we employ the multidimensional CLT. Defining $\sigma^2_t := \var(R_{i,n}(1) I[W_{i,n} \leq q_\alpha]), \check{\sigma^2_0} := \var(R_{i,n}(0) I[W_{i,n} > q_\alpha])$, we obtain that:

\begin{align*}
    &\frac{1}{\sqrt{n}}\sum_{i=1}^n\begin{pmatrix}
        R_{i,n}(1)I[W_{i,n} \leq q_\alpha] - \mu_t\\
        R_{i,n}(0)I[W_{i,n} > q_\alpha] - \check{\mu}_0\\
        -\mathbb{E}[R(1)-R(0)|W=q_\alpha](I[W_i \leq q_\alpha]-\alpha)
    \end{pmatrix} \\&\overset{d}{\rightarrow} \mathcal{N}\left(0, \Sigma\right)
\end{align*} where $\Sigma$ is

\[\begin{pmatrix}
    \sigma^2_t & -\mu_t\check{\mu}_0 & -\mathbb{E}[R(1)-R(0)|W=q_\alpha]\mu_t(1-\alpha)\\
    -\mu_t\check{\mu}_0 & \check{\sigma^2_0} & \alpha \mathbb{E}[R(1)-R(0)|W=q_\alpha]\check{\mu}_0\\
    -\mathbb{E}[R(1)-R(0)|W=q_\alpha]\mu_t(1-\alpha) & \alpha \mathbb{E}[R(1)-R(0)|W=q_\alpha]\check{\mu}_0 & \mathbb{E}[R(1)-R(0)|W=q_\alpha]^2\alpha(1-\alpha)
\end{pmatrix}\]

Hence, the continuous mapping theorem, combined with our earlier calculations, easily shows that \[\frac{1}{\sqrt{n}}\sum_{i=1}^n(R_{i,n} - R^0_{i,n} - \tau_n) \overset{d}{\rightarrow} \mathcal{N}\left(0, \sigma^2_{\text{dm}}\right)\] where 
\begin{align*}
    &\sigma^2_{\text{dm}} \\
    &= \alpha(1-\alpha)\mathbb{E}[R(1)-R(0)|W=q_\alpha]^2  \\
    & +(2\alpha\check{\mu}_0- 2(1-\alpha)\mu_t )\mathbb{E}[R(1)-R(0)|W=q_\alpha]+\sigma^2_t + \check{\sigma^2_0} -2\mu_t\check{\mu}_0+ \var(R(0))
\end{align*}
with $\mu_t := \mathbb{E}[R_{i,n}(1) I[W_{i,n} \leq q_\alpha]]$, $\check{\mu}_0 := \mathbb{E}[R_{i,n}(0) I[W_{i,n} > q_\alpha]]$, $\sigma^2_t := \var(R_{i,n}(1) I[W_{i,n} \leq q_\alpha]), \check{\sigma^2_0} := \var(R_{i,n}(0) I[W_{i,n} > q_\alpha])$

Using again the slightly more complex notation from the main body and the rescaled base estimator, we arrive at: 
\begin{theorem}\label{context1} 
   Under \Cref{second-moment} for $Z=R(0)$ and $Z=R(1)$ and \Cref{as:positive-derivative} for $\Upsilon(\mathbf{x})$ as well as \Cref{regularity}, we get:  
   $$\sqrt{n}\left(\theta^{\mathrm{base}}_{n,\alpha}(\pi)-\tau^{\mathrm{new}}_{n,\alpha}(\pi)\right)
   \overset{d}{\rightarrow} \mathcal{N}(0, \sigma^2_{\text{base}})$$ where 
   \begin{equation}
        \sigma^2_{\text{base}} 
    = \frac{1}{\alpha^2}\left(\alpha(1-\alpha)(\rho_1-\rho_0)^2  
    +(2\alpha\check{\mu}_0- 2(1-\alpha)\mu_1 )(\rho_1-\rho_0)+\sigma^2_1 + \check{\sigma^2_0} -2\mu_1\check{\mu}_0+ \mathrm{Var}(R(0))\right)
   \end{equation} 
    with $\mu_{i}=\mathbb{E}[R(i))I[\Upsilon(\mathbf{x})\leq q_{\alpha}]$, $\check{\mu_{i}}=\mathbb{E}[R(i))I[\Upsilon(\mathbf{x})> q_{\alpha}]$, $\rho_{i}=\mathbb{E}[R(i)|\Upsilon(\mathbf{x})= q_{\alpha}]$, $\sigma^2_i=\mathrm{Var}[R(i)I[\Upsilon(\mathbf{x})\leq q_\alpha]]$ and $\check{\sigma^2_i}=\mathrm{Var}[R(i)I[\Upsilon(\mathbf{x})> q_\alpha]]$ for $i\in \{0,1\}$ where $\mathbb{E}$ is taken over $(\mathbf{x},R)\sim P$.
\end{theorem}
Note that our assumptions are slightly different from those used in \citet{imai2023statistical}. They require a finite third moment of the reward (really, their proof only requires a Lyapunov condition of $(2+\delta)$-moment control) if the active (resp.\ passive) action is applied,  while we only need a finite second moment. However, we require that $F_{\Upsilon}$ has positive derivative at $q_{\alpha}$.

\subsection{Subgroup Estimator}

Under \Cref{second-moment} for $Z=R(0)$ and $Z=R(1)$ and \Cref{as:positive-derivative} for $W$ as well as \Cref{regularity}, we can show that
 \begin{equation}\label{int}\frac{1}{\sqrt{n}}\sum_{i=1}^n(R_{i,n}I_{F_{i,n}} - R^0_{i,n}I_{F^0_{i,n}} - \tau_n) \overset{d}{\rightarrow} \mathcal{N}(0, \sigma^2_{\text{asym}})\end{equation} where \begin{equation*}\sigma^2_{\text{asym}} = \alpha(1-\alpha)(\mathbb{E}[R(1)|W=q_\alpha]^2+\mathbb{E}[R(0)|W=q_\alpha]^2) - 2(1-\alpha)(\mathbb{E}[R(1)|W=q_\alpha]\mu_t+\mathbb{E}[R(0)|W=q_\alpha]\mu_c) + \sigma^2_t + \sigma^2_c\end{equation*} where $\mu_t = \mathbb{E}[R_{1,n}(1)I[W_{1,n}\leq q_\alpha]], \mu_c = \mathbb{E}[R_{1,n}(0)I[W_{1,n}\leq q_\alpha]]$, $\sigma^2_t=  \var[R_{1,n}(1)I[W_{1,n}\leq q_\alpha]], \sigma^2_c=  \var[R_{1,n}(0)I[W_{1,n}\leq q_\alpha]]$

To do so by \Cref{theorem4}  for $W$ and $R(1)$ we can conclude: 
\begin{equation}\label{eq:normdist}
    \sqrt{n}\left(\frac{1}{n}\sum_{i=1}^n R_{i,n}I_{F_{i,n}} -  \mathbb{E}[R_{i,n}(1)I_{E_{i,n}}] \right) \overset{d}{\rightarrow} \mathcal{N}(0, \sigma_t^2 - 2\mathbb{E}[R(1)|W=q_\alpha]\mu_t(1-\alpha) + \mathbb{E}[R(1)|W=q_\alpha]^2\alpha(1-\alpha)).
\end{equation}

Similarly, for $W$ and $R(0)$, we get: 
$$\sqrt{n}\left(\frac{1}{n}\sum_{i=1}^n R^0_{i,n}I_{F^ 0_{i,n}} -  \mathbb{E}[R^0_{i,n}(0)(1-I_{E^0_{i,n}})] \right) \overset{d}{\rightarrow} \mathcal{N}(0, \sigma_c^2 - 2\mathbb{E}[R(0)|W=q_\alpha]\mu_c(1-\alpha) + \mathbb{E}[R(0)|W=q_\alpha]^2\alpha(1-\alpha)).$$
Combining with \Cref{th:quant-vs-ranks} yields \Cref{int}. 
Translated to the notation from the main body, we get: 
\begin{theorem}\label{context2} 
   Under \Cref{second-moment} for $Z=R(0)$ and $Z=R(1)$ and \Cref{as:positive-derivative} for $\Upsilon(\mathbf{x})$ as well as \Cref{regularity}, we get:  
   $$\sqrt{n}\left(\theta^{\mathrm{SG}}_{n,\alpha}(\pi)-\tau^{\mathrm{new}}_{n,\alpha}(\pi)\right)
   \overset{d}{\rightarrow} \mathcal{N}(0, \sigma^2_{\text{SG}})$$ where 
   \begin{equation}
       \sigma^2_{\text{SG}}   =\frac{1}{\alpha^2}\bigg(\alpha(1-\alpha)(\rho_1^2+  \rho_0^2)  - 2(1-\alpha)(\rho_1 \mu_1  +\rho_0 \mu_0)+ \sigma^2_1+\sigma^2_0\bigg)
   \end{equation} 
    with $\mu_{i}=\mathbb{E}[R(i))I[\Upsilon(\mathbf{x})\leq q_{\alpha}]$, $\rho_{i}=\mathbb{E}[R(i)|\Upsilon(\mathbf{x})= q_{\alpha}]$ and $\sigma^2_i=\mathrm{Var}[R(i)I[\Upsilon(\mathbf{x})\leq q_\alpha]]$ for $i\in \{0,1\}$ where $\mathbb{E}$ is taken over $(\mathbf{x},R)\sim P$.
\end{theorem}

\section{Additional Material for Section \ref{hy:Base}: Variance Estimation}\label{app:Base2}

In this section, we construct variance estimators for the asymptotic variance terms obtained above. We start with the subgroup estimator whose variance expression is easier and then reuse the calculations for the base estimator.

\subsection{Subgroup Estimator}\label{SE:var}

We again start with using the less convoluted notation from the appendix and then restate the results in terms of the notation of the main body. 
Recall that the asymptotic variance in Theorem 5 is given by \begin{equation*}\sigma^2_{\text{asym}} = \alpha(1-\alpha)(\mathbb{E}[R(1)|W=q_\alpha]^2+\mathbb{E}[R(0)|W=q_\alpha]^2) - 2(1-\alpha)(\mathbb{E}[R(1)|W=q_\alpha]\mu_t+\mathbb{E}[R(0)|W=q_\alpha]\mu_c) + \sigma^2_t + \sigma^2_c\end{equation*} where $\mu_t = \mathbb{E}[R_{1,n}(1)I[W_{1,n}\leq q_\alpha]], \mu_c = \mathbb{E}[R_{1,n}(0)I[W_{1,n}\leq q_\alpha]]$, $\sigma^2_t=  \var[R_{1,n}(1)I[W_{1,n}\leq q_\alpha]], \sigma^2_c=  \var[R_{1,n}(0)I[W_{1,n}\leq q_\alpha]]$. To consistently estimate $\sigma^2_{\text{asym}}$ it suffices to consistently estimate each term above. 

\begin{itemize}
    \item $\alpha$ is known and doesn't need to be estimated
    \item $\mu_t$ can be consistently estimated by $\frac{1}{n}\sum_{i=1}^n R_{i,n}I_{F_{i,n}}$. This is a simple consequence of \Cref{theorem4} which tells us that \[\sqrt{n}\left(\frac{1}{n}\sum_{i=1}^n R_{i,n}I_{F_{i,n}} - \mu_t\right)\] converges in distribution to a Normal distribution, which by Slutsky's theorem implies that \[\frac{1}{n}\sum_{i=1}^n R_{i,n}I_{F_{i,n}} - \mu_t \overset{p}{\rightarrow} 0.\] 
    \item The same reasoning as in the previous bullet easily shows that $\mu_c$ is consistently estimated by $\frac{1}{n}\sum_{i=1}^n R^0_{i,n}I_{F^0_{i,n}}$.
    \item Now we turn to the consistent estimation of $\sigma^2_t.$ Since we know how to consistently estimate $\mu^2_t$ (since we can consistently estimate $\mu_t$) it suffices to be able to consistently estimate $\mathbb{E}[R^2_{1,n}(1)I_{E_{1,n}}].$ We claim that \[\frac{1}{n}\sum_{i=1}^n R^2_{i,n}I_{F_{i,n}} \overset{p}{\rightarrow} \mathbb{E}[R^2_{1,n}(1)I_{E_{1,n}}].\]

    We now show this claim, again following closely Example 1.5 of \cite{sen2018gentle}. First note that: \begin{align*}
        &\frac{1}{n}\sum_{i=1}^n R^2_{i,n}I_{F_{i,n}} - \mathbb{E}[R^2_{1,n}(1)I_{E_{1,n}}]\\
        &= P_n(g_{W_{(\lceil \alpha n \rceil), n}}) - P(g_{q_\alpha}), 
    \end{align*} where here $g_t(y,x) = y^2I[x \leq t]$ and $P$ and $P_n$ are with respect to the joint law of $(W_{i,n}, R_{i,n}(1))$. Above is \begin{align*}
        &= P_n(g_{W_{(\lceil \alpha n \rceil), n}}) - P(g_{q_\alpha})\\
        &= (P_n-P)(g_{W_{(\lceil \alpha n \rceil), n}}) + P(g_{W_{(\lceil \alpha n \rceil), n}})-P(g_{q_\alpha})
    \end{align*}

    Let $\varphi(t) := P(g_{t})$. Using the same argument which showed that $t \mapsto \mathbb{E}[ZI[W \leq t]]$ is differentiable at $q_\alpha$ (under \Cref{as:positive-derivative}), we see that $\varphi$ is also differentiable at $q_\alpha$ and so the delta method, combined with the asymptotic Normality of $(W_{(\lceil \alpha n \rceil), n}-q_\alpha)$ combined with Slutsky's theorem tells us that $P(g_{W_{(\lceil \alpha n \rceil), n}})-P(g_{q_\alpha}) \overset{p}{\rightarrow} 0$.

    As for the first term, notice that \[|(P_n-P)(g_{W_{(\lceil \alpha n \rceil), n}})| \leq \sup_t|(P_n-P)(g_t)|.\] We will show that this goes to zero in probability.

    Now, again for ease of readability, we recall a few elementary definitions from \cite{vaart2023empirical}:

    \begin{definition}[Brackets; Definition 2.1.6 of \cite{vaart2023empirical}]
        Fix a function class $\mathcal{G}$. Let $\ell$ and $u$ be two functions from $\mathcal{R} \times \mathcal{X} \rightarrow \mathbb{R}$ for which $\ell \leq u$ pointwise. Then $[\ell, u]$ is called a bracket and is defined to be the set of functions $g \in \mathcal{G}$ for which $ \ell \leq g \leq u$ (pointwise). An $\epsilon$-bracket is a bracket for which $\|\ell - u \| \leq \epsilon.$
    \end{definition}
    
    \begin{definition}[Bracketing number; Definition 2.1.6 of \cite{vaart2023empirical}]
        The $\epsilon$ bracketing number, $N_{[]}(\epsilon, \mathcal{F}, \|\cdot\|)$ for a function class $\mathcal{G}$ is the minimum number of $\epsilon$-brackets required to cover $\mathcal{G}$.
    \end{definition}

    \begin{definition}
        A function class $\mathcal{G}$ is called Glivenko-Cantelli if \[\sup_{g \in \mathcal{G}}|(P_n-P)(g)| \overset{p}{\rightarrow} 0\]
    \end{definition}

    Now, we recall a key theorem:
    \begin{theorem}[Theorem 2.4.1 of \cite{vaart2023empirical}]\label{gc}
        Let $\mathcal{G}$ consist of measurable functions and be such that $N_{[]}(\epsilon, \mathcal{G}, L_1(P)) < \infty$ for all $\epsilon > 0$. Then $\mathcal{G}$ is Glivenko-Cantelli.
    \end{theorem}

    Using Theorem~\ref{gc}, we obtain:
    \begin{lemma}
    We have that
        \[\sup_t|(P_n-P)(g_t)| \overset{p}{\rightarrow} 0\]
    \end{lemma}
    \begin{proof}
        By way of Theorem~\ref{gc}, all that must be shown is that $N_{[]}(\epsilon, \mathcal{G}, L_1(P)) < \infty$, where $\mathcal{G} = \{g_t: t \in \mathbb{R}\}$. 

        Observe that there exists a grid $t_1, t_2, \ldots, t_K$ for which \[\mathbb{E}[R^2(1)I[W < t_i]]-\mathbb{E}[R^2(1)I[W \leq t_{i-1}]] \leq \epsilon\] for all $i = 1, \ldots, K+1$, where $t_0 := -\infty$ and $t_{K+1} = \infty$. Then $[g_{t_{i-1}}, h_{t_i}]$, with $h_{t}(r,w) := r^2I[w < t]$ are clearly $\epsilon$-brackets and they also clearly cover all of $\mathcal{G}$.
    \end{proof}

    Therefore \[\frac{1}{n}\sum_{i=1}^n R^2_{i,n}I_{F_{i,n}} \overset{p}{\rightarrow} \mathbb{E}[R^2_{1,n}(1)I_{E_{1,n}}]\] and hence \[\frac{1}{n}\sum_{i=1}^n R^2_{i,n}I_{F_{i,n}} - (\frac{1}{n}\sum_{i=1}^n R_{i,n}I_{F_{i,n}})^2 \overset{p}{\rightarrow} \sigma^2_t.\]

    \item The exact same argument as in the last bullet point (as well as the same assumption) shows that \[\frac{1}{n}\sum_{i=1}^n (R^0_{i,n})^2I_{F^0_{i,n}} - (\frac{1}{n}\sum_{i=1}^n R^0_{i,n}I_{F^0_{i,n}})^2 \overset{p}{\rightarrow} \sigma^2_c.\]

    \item Now we show how to estimate $\mathbb{E}[R(1)|W=q_\alpha]$

    Define the estimator \[\hat{m}_n := \frac{1}{k}\sum_{i=\lceil \alpha n \rceil-k}^{\lceil \alpha n \rceil} R_{(i),n}\] and let \[m(w) := \mathbb{E}[R(1)|W=w]\] be the true conditional mean function. Let $\frac{k}{\sqrt{n} \log n} \rightarrow \infty$ with $k/n \rightarrow 0$.

    We will obtain a theorem (proved in the next section), that is a mild extension of Theorem 1 of \cite{CHENG198463}:

\begin{restatable}{lemma}{cheng}\label{cheng}
 Let $(X_i, Y_i)$ be iid draws from some distribution. Defining $\hat{m}_n$ and $m$ analogously as above (but now for the pair $(X,Y)$) suppose additionally that:
        \begin{enumerate}
            \item[(A1)] The function $m$ exists and is continuous in a closed neighborhood of $q_\alpha$. In particular, we assume that $m$ is continuous on $[L_0, U_0] \subseteq [L, U]$ for some $L_0, U_0$ such that $q_\alpha \in (L_0,U_0)$.
            \item[(A2)] We have that $\var(Y|X=x)$ is bounded by some constant $M$ for all $x \in [L_0, U_0]$.
        \end{enumerate}

        Then, \[|\hat{m}_n - m(q_\alpha)| \overset{p}{\rightarrow}0.\]
\end{restatable}
  
    Hence, we have shown that $\hat{m}_n$ is a consistent estimate for $\mathbb{E}[R(1)|W=q_\alpha]$ so long as $\frac{k}{\sqrt{n}\log n} \rightarrow \infty$ with $k/n \rightarrow 0$.
    
    \item Just as above, we have that\[\frac{1}{k}\sum_{i=\lceil \alpha n \rceil-k}^{\lceil \alpha n \rceil} R^0_{(i),n}\] is a consistent estimate for $\mathbb{E}[R(0)|W=q_\alpha]$
    
\end{itemize}

Plugging all of this together, we arrive at: 
 Then \begin{align*}
        &\hat{\sigma}^2_{\text{asym}} := \\&\alpha(1-\alpha)\Bigg[\left(\frac{1}{k}\sum_{i=\lceil \alpha n \rceil-k}^{\lceil \alpha n \rceil} R_{(i),n} \right)^2+\left(\frac{1}{k}\sum_{i=\lceil \alpha n \rceil-k}^{\lceil \alpha n \rceil} R^0_{(i),n}\right)^2\Bigg] - 2(1-\alpha)\Bigg[\left(\frac{1}{k}\sum_{i=\lceil \alpha n \rceil-k}^{\lceil \alpha n \rceil} R_{(i),n} \right)\cdot\frac{1}{n}\sum_{i=1}^n R_iI_{F_{i,n}}\\&+\left(\frac{1}{k}\sum_{i=\lceil \alpha n \rceil-k}^{\lceil \alpha n \rceil} R^0_{(i),n}\right)\cdot\frac{1}{n}\sum_{i=1}^n R^0_iI_{F^0_{i,n}}\Bigg] + \frac{1}{n}\sum_{i=1}^n R^2_{i,n}I_{F_{i,n}} - \left(\frac{1}{n}\sum_{i=1}^n R_{i,n}I_{F_{i,n}}\right)^2 + \frac{1}{n}\sum_{i=1}^n (R^0_{i,n})^2I_{F^0_{i,n}} - \left(\frac{1}{n}\sum_{i=1}^n R^0_{i,n}I_{F^0_{i,n}}\right)^2
    \end{align*}
    is a consistent estimator of $\sigma^2_{\text{asym}}$.
    
Formally, in the language of the main body, we conclude that (where $R^p_{(i)}$, respectively $R^c_{(i)}$, is the reward function of the agent with the $i$th lowest index in the policy, respectively control, arm): 
\begin{theorem}
    In addition to the assumptions made in \Cref{context2}, assume that:
    \begin{enumerate}
        \item The functions $w \mapsto \mathbb{E}[R(1)|\Upsilon(\mathbf{x})=w]$ and $w \mapsto \mathbb{E}[R(0)|\Upsilon(\mathbf{x})=w]$ are continuous in a closed neighborhood of $q_\alpha$.
        \item We have that $\var[R(1)|\Upsilon(\mathbf{x})=w]$ and $\var[R(0)|\Upsilon(\mathbf{x})=w]$ are bounded for all $w$ in these neighborhoods.
        \item  $\frac{k}{\sqrt{n}\log n} \rightarrow \infty$ with $k/n \rightarrow 0$.
    \end{enumerate}

    Then {\small \begin{align*}
        \hat{\sigma}^2_{\text{SE}} := & \frac{1}{\alpha^2}\Bigg(\alpha(1-\alpha)\Bigg[\left(\frac{1}{k}\sum_{i=\lceil \alpha n \rceil-k}^{\lceil \alpha n \rceil} R_{(i)}^p(1) \right)^2+\left(\frac{1}{k}\sum_{i=\lceil \alpha n \rceil-k}^{\lceil \alpha n \rceil} R^c_{(i)}(0)\right)^2\Bigg] - 2(1-\alpha)\Bigg[\left(\frac{1}{k}\sum_{i=\lceil \alpha n \rceil-k}^{\lceil \alpha n \rceil} R^p_{(i)}(1) \right)\cdot\frac{1}{n}\sum_{i\in \pi(\mathbf{X}_n^p,\alpha)} R_i^p(1)\\ &+\left(\frac{1}{k}\sum_{i=\lceil \alpha n \rceil-k}^{\lceil \alpha n \rceil} R^c_{(i)}(0)\right)\cdot\frac{1}{n}\sum_{i\in \pi(\mathbf{X}_n^c,\alpha)}  R^c_i(0)\Bigg] + \frac{1}{n}\sum_{i\in \pi(\mathbf{X}_n^p,\alpha)} R^p_i(1)^2 - \left(\frac{1}{n}\sum_{i\in \pi(\mathbf{X}_n^p,\alpha)} R^p_i(1)\right)^2 + \frac{1}{n}\sum_{i\in \pi(\mathbf{X}_n^c,\alpha)} R^c_i(0)^2 - \left(\frac{1}{n}\sum_{i\in \pi(\mathbf{X}_n^c,\alpha)} R^c_i(0)\right)^2 \Bigg)
    \end{align*}}
    is a consistent estimator of $\sigma^2_{\text{SE}}$. Therefore, by \Cref{context2} and Slutsky's theorem, we have that \[\frac{
    \sqrt{n}\left(\theta^{\mathrm{SG}}_{n,\alpha}(\pi)-\tau^{\mathrm{new}}_{n,\alpha}(\pi)\right)}{\hat{\sigma}_{\text{SE}}} \overset{d}{\rightarrow} \mathcal{N}(0,1)\]
\end{theorem}

\subsection{Proof of \Cref{cheng}}

We now show how to prove \Cref{cheng} which is a very mild extension of Theorem 1 of \cite{CHENG198463}; the proof follows closely the one given in \cite{CHENG198463}.
\cheng*

First we prove an analogue to Lemma 2 of \cite{CHENG198463}:
\begin{lemma}[Analogue of Lemma 2 of \cite{CHENG198463}]
     Define $V_n = k^{-1}\sum_{i=\lceil \alpha n \rceil-k}^{\lceil \alpha n \rceil} m(X_{(i),n})$. Under (A1), we have that \[|V_n - m(q_\alpha)| \overset{p}{\rightarrow} 0.\]
\end{lemma}
\begin{proof}
Proof is basically same as that of Theorem 1 in \cite{dev}. Fix $\epsilon > 0$. Choose $L'_0 \geq L_0, U'_0 \leq U_0$ close enough to $q_\alpha$ so that $|m(x)-m(x')| \leq \epsilon/2$ for all $x,x' \in [L'_0, U'_0]$. Then 
    We have that 
    \begin{align*}
        &P(|V_n - m(q_\alpha)| > \epsilon) \\
        &= P(|V_n - m(q_\alpha)| > \epsilon, X_{(\lceil \alpha n \rceil),n} \in [L'_0, U'_0] \text{ and }X_{(\lceil \alpha n \rceil-k),n} \in [L'_0, U'_0])\\
        &+ P(|V_n - m(q_\alpha)| > \epsilon, X_{(\lceil \alpha n \rceil),n} \not\in [L'_0, U'_0] \text{ or }X_{(\lceil \alpha n \rceil-k),n} \not\in [L'_0, U'_0])
    \end{align*}

    The second term goes to zero because both $X_{(\lceil \alpha n \rceil), n}$ and $X_{(\lceil \alpha n \rceil-k), n}$ converge in probability to $q_\alpha$ under the conditions of \Cref{theorem4}. 
   
    So, all that must be done is to handle $P(|V_n - m(q_\alpha)| > \epsilon, X_{(\lceil \alpha n \rceil),n} \in [L'_0, U'_0] \text{ and }X_{(\lceil \alpha n \rceil-k),n} \in [L'_0, U'_0])$ which is upper bounded as \begin{align*}
        &P(|V_n - m(q_\alpha)| > \epsilon, X_{(\lceil \alpha n \rceil),n} \in [L'_0, U'_0] \text{ and }X_{(\lceil \alpha n \rceil-k),n} \in [L'_0, U'_0])\\
        &\leq P( k^{-1}\sum_{i = \lceil \alpha n \rceil-k}^{\lceil \alpha n \rceil} |m(X_{(i),n})-m(q_\alpha)|>\epsilon, X_{(\lceil \alpha n \rceil),n} \in [L'_0, U'_0] \text{ and }X_{(\lceil \alpha n \rceil-k),n} \in [L'_0, U'_0])\\
        &\leq P(((k+1)/k)\sup_{x \in [L'_0, U'_0]}|m(x) - m(q_\alpha)| > \epsilon)\\
        &= 0
    \end{align*} for all large $k$ (hence all large $n$) since we have chosen $L'_0, U'_0$ so that $|m(x)-m(x')| \leq \epsilon/2$ for all $x,x' \in [L'_0, U'_0]$.
\end{proof}

\begin{lemma}[Lemma 3 of \cite{CHENG198463}]
     Define $\tilde{m}_n := \frac{1}{k}\sum_{i=\lceil \alpha n \rceil-k}^{\lceil \alpha n \rceil} \tilde{Y}_{(i),n}$ where $\tilde{Y}_{i,n} := Y_{i,n}I[|Y_{i,n}| \leq n^{1/2}]$ is a truncated version of $Y_{i,n}$. If $Y$ is square-integrable, then we have that \[|\tilde{m}_n - \hat{m}_n| \overset{a.s.}{\rightarrow} 0\]
\end{lemma}
\begin{proof}
    Exact same as in \cite{CHENG198463}.
\end{proof}

\begin{lemma}[Analogue to Lemma 4 of \cite{CHENG198463}]
    Assume (A1) and (A2). And define $$\tilde{V}_n := k^{-1}\sum_{i=\lceil \alpha n \rceil-k}^{\lceil \alpha n \rceil} \mathbb{E}[\tilde{Y}_{(i),n}|X_{(i),n}].$$ Then \[|V_n - \tilde{V}_n| \overset{p}{\rightarrow} 0\]
\end{lemma}
\begin{proof}
    Basically same as in \cite{CHENG198463}:
    \begin{align*}
        |V_n - \tilde{V}_n| &\leq k^{-1}\sum_{i=\lceil \alpha n \rceil-k}^{\lceil \alpha n \rceil} \mathbb{E}[|Y_{(i),n}|I[|Y_{(i),n}| > n^{1/2}]|X_{(i),n}]\\
        &\leq k^{-1}\sum_{i=\lceil \alpha n \rceil-k}^{\lceil \alpha n \rceil } n^{-1/2}\mathbb{E}[Y_{(i),n}^2|X_{(i),n}]\\
        &\leq k^{-1}\sum_{i=\lceil \alpha n \rceil-k}^{\lceil \alpha n \rceil} n^{-1/2}(\var[Y_{(i),n}|X_{(i),n}] + \mathbb{E}[Y_{(i),n}|X_{(i),n}]^2)\\
        &\leq ((k+1)/k)n^{-1/2}\max_{i = \lceil \alpha n \rceil-k, \ldots, \lceil \alpha n \rceil }(\var[Y_{(i),n}|X_{(i),n}] + \mathbb{E}[Y_{(i),n}|X_{(i),n}]^2)
    \end{align*}

    So it suffices to show that \[n^{-1/2}\max_{i = \lceil \alpha n \rceil-k, \ldots, \lceil \alpha n \rceil }(\var[Y_{(i),n}|X_{(i),n}] + \mathbb{E}[Y_{(i),n}|X_{(i),n}]^2)\] to zero. 
   
    So, we have that \begin{align*}
        &P(n^{-1/2}\max_{i = \lceil \alpha n \rceil-k, \ldots, \lceil \alpha n \rceil}(\var[Y_{(i),n}|X_{(i),n}] + \mathbb{E}[Y_{(i),n}|X_{(i),n}]^2) > \epsilon)\\
        &\leq P(n^{-1/2}\max_{i = \lceil \alpha n \rceil-k, \ldots, \lceil \alpha n \rceil }(\var[Y_{(i),n}|X_{(i),n}] + \mathbb{E}[Y_{(i),n}|X_{(i),n}]^2) > \epsilon, X_{(\lceil \alpha n \rceil),n} \in [L_0, U_0] \text{ and }X_{(\lceil \alpha n \rceil-k),n} \in [L_0, U_0])\\
        &+ P(n^{-1/2}\max_{i = \lceil \alpha n \rceil-k, \ldots, \lceil \alpha n \rceil}(\var[Y_{(i),n}|X_{(i),n}] + \mathbb{E}[Y_{(i),n}|X_{(i),n}]^2) > \epsilon, X_{(\lceil \alpha n \rceil),n} \not\in [L_0, U_0] \text{ or }X_{(\lceil \alpha n \rceil-k),n} \not\in [L_0, U_0])
    \end{align*}

    The second term goes to zero by the convergence in probability of both $X_{(\lceil \alpha n \rceil), n}$ and $X_{(\lceil \alpha n \rceil-k), n}$ to $q_\alpha$ and the first term is upper bounded as \[P(n^{-1/2}\sup_{x \in [L_0, U_0]}(\var[Y_{1,n}|X_{1,n}=x] + \mathbb{E}[Y_{1,n}|X_{1,n}=x]^2) > \epsilon) \rightarrow 0\] by employing (A1) and (A2).
\end{proof}

\begin{lemma}[Analogue to Lemma 5 of \cite{CHENG198463}]
    Assume (A1) and (A2). Then \[|\tilde{m}_n - \tilde{V}_n| \overset{a.s.}{\rightarrow} 0\]
\end{lemma}
\begin{proof}
    Basically same as in \cite{CHENG198463}:
    \begin{align*}
        &P(\tilde{m}_n - \tilde{V}_n > \epsilon)\\
        &= \mathbb{E}[P(\tilde{m}_n - \tilde{V}_n > \epsilon|X_{(\lceil \alpha n \rceil),n}, \ldots, X_{(\lceil \alpha n \rceil-k),n})]\\
        &= \mathbb{E}[P(\tilde{m}_n - \tilde{V}_n > \epsilon,  X_{(\lceil \alpha n \rceil),n} \in [L_0, U_0] \text{ and }X_{(\lceil \alpha n \rceil-k),n} \in [L_0, U_0]|X_{(\lceil \alpha n \rceil),n}, \ldots, X_{(\lceil \alpha n \rceil-k),n})]\\
        &+ \mathbb{E}[P(\tilde{m}_n - \tilde{V}_n > \epsilon,  X_{(\lceil \alpha n \rceil),n} \not\in [L_0, U_0] \text{ or }X_{(\lceil \alpha n \rceil-k),n} \not\in [L_0, U_0]|X_{(\lceil \alpha n \rceil),n}, \ldots, X_{(\lceil \alpha n \rceil-k),n})]
        \end{align*}
        Again, the second term tends to zero via the same argument made in the last proof since both \[X_{(\lceil \alpha n \rceil),n} \overset{p}{\rightarrow} q_\alpha\] and \[X_{(\lceil \alpha n \rceil-k),n} \overset{p}{\rightarrow} q_\alpha.\] Thus we need only worry about the first term, which is upper bounded by 
        \begin{align*}&\leq \sup_{x_0, \ldots, x_k \in [L_0, U_0]^{k+1}}P(\tilde{m}_n - \tilde{V}_n > \epsilon|X_{(\lceil \alpha n \rceil),n}=x_0, \ldots, X_{(\lceil \alpha n \rceil)-k,n}=x_k),
    \end{align*} which for any $\beta_n > 0$, is at least 
    \begin{align*}
        &\leq n^{-\epsilon\beta_n}\sup_{x_0, \ldots, x_k \in [L_0, U_0]^{k+1}}\mathbb{E}[\exp(\beta_n \log(n)k^{-1}\sum_{i=\lceil \alpha n \rceil-k}^{\lceil \alpha n \rceil}(\tilde{Y}_{(i),n}- \mathbb{E}[\tilde{Y}_{(i),n}|X_{(i),n}]))|X_{(\lceil \alpha n \rceil),n}=x_0, \ldots, X_{(\lceil \alpha n \rceil)-k,n}=x_k]
    \end{align*}
    Lemma 6 of \cite{CHENG198463} shows that, since $\tilde{Y}_{(\lceil \alpha n \rceil),n}, \ldots, \tilde{Y}_{(\lceil \alpha n \rceil-k),n}$ are independent draws conditional on $X_{(\lceil \alpha n \rceil),n}=x_0, \ldots, X_{(\lceil \alpha n \rceil-k),n}=x_k$ that  \[\mathbb{E}[\exp(\beta_n \log(n)k^{-1}\sum_{i=\lceil \alpha n \rceil-k}^{\lceil \alpha n \rceil}(\tilde{Y}_{(i),n}- \mathbb{E}[\tilde{Y}_{(i),n}|X_{(i),n}]))|X_{(\lceil \alpha n \rceil),n}=x_0, \ldots, X_{(\lceil \alpha n \rceil)-k,n}=x_k]\] can be upper bounded as \[\exp\left((k+1) \cdot (\beta_n \log(n)k^{-1})^2 \cdot M \cdot \frac{1+2(\beta_n \log(n)k^{-1})n^{1/2}}{2}\right)\] so long as \[\beta_n \log(n)k^{-1} \leq 1/\sqrt{n}.\]

    In view of the fact that $k/(\sqrt{n} \log n)\rightarrow \infty$, let us set \[\beta_n = \min\left((\log n)^{1/3}, \sqrt{k/(\sqrt{n}\log n)}\right).\] Then, it is clear that the condition that $\beta_n \log(n)k^{-1} \leq 1/\sqrt{n}$ is met for all large $n$, and we also have that \[\exp\left((k+1) \cdot (\beta_n \log(n)k^{-1})^2 \cdot M \cdot \frac{1+2(\beta_n \log(n)k^{-1})n^{1/2}}{2}\right) = O(1)\]

    So, for all large $n$ the above is upper bounded by, for some positive constant $C$, \[n^{-\epsilon \beta_n}C \leq Cn^{-2},\] where the last inequality is for all $n$ large enough. We conclude by the first Borel-Cantelli lemma (via the summability of the series $\sum n^{-2}$) as well as a union bound over $\epsilon = 1/\ell$, $\ell = 1, 2, \ldots$. 
\end{proof}

We conclude the theorem result by combining the four above lemmas and using the triangle inequality.

\subsection{Base Estimator}

Using the results of \Cref{SE:var}, it is easy to establish the following:
\begin{theorem}
    In addition to the assumptions made in \Cref{context1}, assume that:
    \begin{enumerate}
        \item The functions $w \mapsto \mathbb{E}[R(1)|\Upsilon(\mathbf{x})=w]$ and $w \mapsto \mathbb{E}[R(0)|\Upsilon(\mathbf{x})=w]$ are continuous in a closed neighborhood of $q_\alpha$.
        \item We have that $\var[R(1)|\Upsilon(\mathbf{x})=w]$ and $\var[R(0)|\Upsilon(\mathbf{x})=w]$ are bounded for all $w$ in these neighborhoods.
        \item  $\frac{k}{\sqrt{n}\log n} \rightarrow \infty$ with $k/n \rightarrow 0$.
    \end{enumerate}

    Then \begin{align*}
        \hat{\sigma}^2_{\text{base}} := & \frac{1}{\alpha^2}\bigg(\alpha(1-\alpha)\Bigg[\frac{1}{k}\sum_{i=\lceil \alpha n \rceil-k}^{\lceil \alpha n \rceil} R^p_{(i)}(1) - \frac{1}{k}\sum_{i=\lceil \alpha n \rceil-k}^{\lceil \alpha n \rceil} R^c_{(i)}(0)\Bigg]^2+ \\
        &\left( 2\alpha\frac{1}{n}\sum_{i\notin \pi(\mathbf{X}^c_n,\alpha)} R^c_{i}(0) - 
        2(1-\alpha)\frac{1}{n}\sum_{i\in \pi(\mathbf{X}^p_n,\alpha)} R^p_{i}(1)\right)\Bigg[\frac{1}{k}\sum_{i=\lceil \alpha n \rceil-k}^{\lceil \alpha n \rceil} R^p_{(i)}(1) - \frac{1}{k}\sum_{i=\lceil \alpha n \rceil-k}^{\lceil \alpha n \rceil} R^c_{(i)}(0)\Bigg]+\\
         & \frac{1}{n}\sum_{i\in \pi(\mathbf{X}^p_n,\alpha)} R^p_{i}(1)^2- \left(\frac{1}{n}\sum_{i\in \pi(\mathbf{X}^p_n,\alpha)} R^p_{i}(1)\right)^2 + \frac{1}{n}\sum_{i\notin \pi(\mathbf{X}^c_n,\alpha)} R^c_{i}(0)^2 - \left(\frac{1}{n}\sum_{i\notin \pi(\mathbf{X}^c_n,\alpha)} R^c_{i}(0)\right)^2-\\
         & 2\left(\frac{1}{n}\sum_{i\notin \pi(\mathbf{X}^c_n,\alpha)} R^c_{i}(0)\right)\left(\frac{1}{n}\sum_{i\in \pi(\mathbf{X}^p_n,\alpha)} R^p_{i}(1)\right)+\frac{1}{n}\sum_{i=1}^n(R^c_{i}(0)-\frac{1}{n}\sum_{i=1}^nR^c_{i}(0))^2\bigg)
    \end{align*}
    is a consistent estimator of $\sigma^2_{\text{base}}$. Therefore, by \Cref{context1} and Slutsky's theorem, we have that \[\frac{
    \sqrt{n}\left(\theta^{\mathrm{base}}_{n,\alpha}(\pi)-\tau^{\mathrm{new}}_{n,\alpha}(\pi)\right)}{\hat{\sigma}_{\text{base}}} \overset{d}{\rightarrow} \mathcal{N}(0,1)\]
\end{theorem}

\section{Asymptotic Normality of Hybrid Estimator}\label{app:Hyb}

In this section, we consider the following weighted estimator: \[\frac{1}{n}\sum_{i=1}^n R_{i,n} I_{F_{i,n}} - \frac{1}{n}\sum_{i=1}^n R^0_{i,n}I_{F^0_{i,n}} + \hat{w}_n \left(\frac{1}{n}\sum_{i=1}^n R_{i,n}(1-I_{F_{i,n}}) - \frac{1}{n}\sum_{i=1}^n R^0_{i,n}(1-I_{F^0_{i,n}})\right)\] where the (data-dependent) weight $\hat{w}_n$ is allowed to depend arbitrarily on all of the data and may take any real value; all that we will assume is that it converges in probability to some deterministic quantity $w^*$. Notice that $\hat{w}_n \equiv 0$ recovers the subgroup estimator while $\hat{w}_n \equiv 1$ recovers the base estimator.

We aim to show that the following display is asymptotically normal:
\begin{equation}\label{weighted}\sqrt{n}\left(\frac{1}{n}\sum_{i=1}^n R_{i,n} I_{F_{i,n}} - \frac{1}{n}\sum_{i=1}^n R^0_{i,n}I_{F^0_{i,n}} - \tau_n\right) + \hat{w}_n \cdot \sqrt{n}\left(\frac{1}{n}\sum_{i=1}^n R_{i,n}(1-I_{F_{i,n}}) - \frac{1}{n}\sum_{i=1}^n R^0_{i,n}(1-I_{F^0_{i,n}})\right)\end{equation}

Henceforth we will choose $\hat{w}_n$ so that it converges to some quantity $w^*$ in probability (i.e., $\hat{w}_n \overset{p}{\rightarrow} w^*$); indeed this property is satisfied both by the base and subgroup estimators. We will derive the optimal choice of $w^*$ and then say how to choose $\hat{w}_n$.

We will make the following mild assumption which ensures the existence of such an optimal $w^*$:

\begin{assumption}[Positive variance of the conditional mean]\label{as:positive-var-c-mean}
    $\var(\mathbb{E}[R_{i,n}(0)|I[W_{i,n} > q_\alpha]]) > 0$.
\end{assumption}

This assumption essentially says that, upon revealing $I[W_{i,n} > q_\alpha]$, there is still ``randomness'' left in $R_{i,n}(0)$.

Under the same assumptions as \Cref{theorem4}, essentially the exact same proof of \Cref{theorem4} allows us to show that, defining  $\mu_t := \mathbb{E}[R_{i,n}(1) I[W_{i,n} \leq q_\alpha]]$, $\check{\mu}_0 := \mathbb{E}[R_{i,n}(0) I[W_{i,n} > q_\alpha]]$, \begin{align*}
    &\frac{1}{\sqrt{n}}\sum_{i=1}^n(R_{i,n}I_{F_{i,n}} - \mathbb{E}[R_{i,n}(1)I_{E_{i,n}}])\\
    &= \frac{1}{\sqrt{n}}\sum_{i=1}^n(R_{i,n}(1) I[W_{i,n} \leq q_\alpha] - \mu_t) + \mathbb{E}[R(1)|W=q_\alpha]F_W'(q_\alpha)\sqrt{n}(W_{(\lceil \alpha n\rceil),n} - q_\alpha) + o_p(1)\\
    &= \frac{1}{\sqrt{n}}\sum_{i=1}^n(R_{i,n}(1) I[W_{i,n} \leq q_\alpha] - \mu_t) -\frac{\mathbb{E}[R(1)|W=q_\alpha]F_W'(q_\alpha)}{\sqrt{n}}\sum_{i=1}^n \frac{I[W_i \leq q_\alpha]-\alpha}{F_W'(q_\alpha)} + o_p(1)\\
    &= \frac{1}{\sqrt{n}}\sum_{i=1}^n(R_{i,n}(1) I[W_{i,n} \leq q_\alpha] - \mu_t) -\frac{\mathbb{E}[R(1)|W=q_\alpha]}{\sqrt{n}}\sum_{i=1}^n (I[W_i \leq q_\alpha]-\alpha) + o_p(1)
\end{align*}

Completely analogously: \begin{align*}
    &\frac{1}{\sqrt{n}}\sum_{i=1}^n(R^0_{i,n}I_{F^0_{i,n}} - \mathbb{E}[R^0_{i,n}(0)I_{E^0_{i,n}}])\\
    &= \frac{1}{\sqrt{n}}\sum_{i=1}^n(R^0_{i,n} I[W^0_{i,n} \leq q_\alpha] - \mathbb{E}[R^0_{i,n} I[W^0_{i,n} \leq q_\alpha]]) -\frac{\mathbb{E}[R(0)|W=q_\alpha]}{\sqrt{n}}\sum_{i=1}^n (I[W^0_i \leq q_\alpha]-\alpha) + o_p(1)
\end{align*}

Similarly, by using a proof strategy nearly identical to \Cref{theorem4}, we obtain that 
\begin{align*}
    &\frac{1}{\sqrt{n}}\sum_{i=1}^n(R_{i,n}(0)(1-I_{F_{i,n}}) - \mathbb{E}[R_{i,n}(0)(1-I_{E_{i,n}})])\\
    &= \frac{1}{\sqrt{n}}\sum_{i=1}^n(R_{i,n}(0) I[W_{i,n} > q_\alpha] - \check{\mu}_0) - \mathbb{E}[R(0)|W=q_\alpha]F_W'(q_\alpha)\sqrt{n}(W_{(\lceil \alpha n\rceil),n} - q_\alpha) + o_p(1)\\
    &= \frac{1}{\sqrt{n}}\sum_{i=1}^n(R_{i,n}(0) I[W_{i,n} > q_\alpha] - \check{\mu}_0) +\frac{\mathbb{E}[R(0)|W=q_\alpha]F_W'(q_\alpha)}{\sqrt{n}}\sum_{i=1}^n \frac{I[W_i \leq q_\alpha]-\alpha}{F_W'(q_\alpha)} + o_p(1)\\
    &= \frac{1}{\sqrt{n}}\sum_{i=1}^n(R_{i,n}(0) I[W_{i,n} > q_\alpha] - \check{\mu}_0) +\frac{\mathbb{E}[R(0)|W=q_\alpha]}{\sqrt{n}}\sum_{i=1}^n (I[W_i \leq q_\alpha]-\alpha) + o_p(1)
\end{align*}

And, completely analogously, \begin{align*}
    &\frac{1}{\sqrt{n}}\sum_{i=1}^n(R^0_{i,n}(0)(1-I_{F^0_{i,n}}) - \mathbb{E}[R^0_{i,n}(0)(1-I_{E^0_{i,n}})])\\
    &= \frac{1}{\sqrt{n}}\sum_{i=1}^n(R^0_{i,n}(0) I[W^0_{i,n} > q_\alpha] - \check{\mu}_0) +\frac{\mathbb{E}[R(0)|W=q_\alpha]}{\sqrt{n}}\sum_{i=1}^n (I[W^0_i \leq q_\alpha]-\alpha) + o_p(1)
\end{align*}

Therefore, using \Cref{th:quant-vs-ranks}, the LHS of display~\eqref{weighted} may be rewritten as 
\begin{align*}
    &\sqrt{n}\left(\frac{1}{n}\sum_{i=1}^n R_{i,n} I_{F_{i,n}} - \frac{1}{n}\sum_{i=1}^n R^0_{i,n}I_{F^0_{i,n}} - \tau_n\right) + \hat{w}_n \cdot \sqrt{n}\left(\frac{1}{n}\sum_{i=1}^n R_{i,n}(1-I_{F_{i,n}}) - \frac{1}{n}\sum_{i=1}^n R^0_{i,n}(1-I_{F^0_{i,n}})\right)\\
    &= \frac{1}{\sqrt{n}}\sum_{i=1}^n(R_{i,n}(1) I[W_{i,n} \leq q_\alpha] - \mu_t) -\frac{\mathbb{E}[R(1)|W=q_\alpha]}{\sqrt{n}}\sum_{i=1}^n (I[W_i \leq q_\alpha]-\alpha)\\
    &-\left(\frac{1}{\sqrt{n}}\sum_{i=1}^n(R^0_{i,n} I[W^0_{i,n} \leq q_\alpha] - \mathbb{E}[R^0_{i,n} I[W^0_{i,n} \leq q_\alpha]]) -\frac{\mathbb{E}[R(0)|W=q_\alpha]}{\sqrt{n}}\sum_{i=1}^n (I[W^0_i \leq q_\alpha]-\alpha)\right)\\
    &+ \hat{w}_n\Big(\frac{1}{\sqrt{n}}\sum_{i=1}^n(R_{i,n}(0) I[W_{i,n} > q_\alpha] - \check{\mu}_0) +\frac{\mathbb{E}[R(0)|W=q_\alpha]}{\sqrt{n}}\sum_{i=1}^n (I[W_i \leq q_\alpha]-\alpha)\\
    &- \left(\frac{1}{\sqrt{n}}\sum_{i=1}^n(R^0_{i,n}(0) I[W^0_{i,n} > q_\alpha] - \check{\mu}_0) +\frac{\mathbb{E}[R(0)|W=q_\alpha]}{\sqrt{n}}\sum_{i=1}^n (I[W^0_i \leq q_\alpha]-\alpha)\right) \Big) + o_p(1)
\end{align*}

It is not hard to see that \begin{align*}
    &\frac{1}{\sqrt{n}}\sum_{i=1}^n(R_{i,n}(0) I[W_{i,n} > q_\alpha] - \check{\mu}_0) +\frac{\mathbb{E}[R(0)|W=q_\alpha]}{\sqrt{n}}\sum_{i=1}^n (I[W_i \leq q_\alpha]-\alpha)\\
    &- \left(\frac{1}{\sqrt{n}}\sum_{i=1}^n(R^0_{i,n}(0) I[W^0_{i,n} > q_\alpha] - \check{\mu}_0) +\frac{\mathbb{E}[R(0)|W=q_\alpha]}{\sqrt{n}}\sum_{i=1}^n (I[W^0_i \leq q_\alpha]-\alpha)\right)
\end{align*} converges in distribution to a Normal distribution (indeed we will show it's joint asymptotic Normality with the first term shortly), and hence Slutsky's theorem easily shows that the previous display is asymptotically equal to the same thing with $\hat{w}_n$ replaced by $w^*$:
\begin{align*}
    &\frac{1}{\sqrt{n}}\sum_{i=1}^n(R_{i,n}(1) I[W_{i,n} \leq q_\alpha] - \mu_t) -\frac{\mathbb{E}[R(1)|W=q_\alpha]}{\sqrt{n}}\sum_{i=1}^n (I[W_i \leq q_\alpha]-\alpha)\\
    &-\left(\frac{1}{\sqrt{n}}\sum_{i=1}^n(R^0_{i,n} I[W^0_{i,n} \leq q_\alpha] - \mathbb{E}[R^0_{i,n} I[W^0_{i,n} \leq q_\alpha]]) -\frac{\mathbb{E}[R(0)|W=q_\alpha]}{\sqrt{n}}\sum_{i=1}^n (I[W^0_i \leq q_\alpha]-\alpha)\right)\\
    &+ w^*\Big(\frac{1}{\sqrt{n}}\sum_{i=1}^n(R_{i,n}(0) I[W_{i,n} > q_\alpha] - \check{\mu}_0) +\frac{\mathbb{E}[R(0)|W=q_\alpha]}{\sqrt{n}}\sum_{i=1}^n (I[W_i \leq q_\alpha]-\alpha)\\
    &- \left(\frac{1}{\sqrt{n}}\sum_{i=1}^n(R^0_{i,n}(0) I[W^0_{i,n} > q_\alpha] - \check{\mu}_0) +\frac{\mathbb{E}[R(0)|W=q_\alpha]}{\sqrt{n}}\sum_{i=1}^n (I[W^0_i \leq q_\alpha]-\alpha)\right) \Big) + o_p(1)
\end{align*}

 Combining terms from treatment group into one term and from control group into another, we rewrite the above as 
 \begin{align*}
    &\Bigg[\frac{1}{\sqrt{n}}\sum_{i=1}^n(R_{i,n}(1) I[W_{i,n} \leq q_\alpha] - \mu_t) -\frac{\mathbb{E}[R(1)-w^* R(0)|W=q_\alpha]}{\sqrt{n}}\sum_{i=1}^n (I[W_i \leq q_\alpha]-\alpha)\\
    &+ w^*\cdot\frac{1}{\sqrt{n}}\sum_{i=1}^n(R_{i,n}(0) I[W_{i,n} > q_\alpha] - \check{\mu}_0)\Bigg]\\
    &-\Bigg[\frac{1}{\sqrt{n}}\sum_{i=1}^n(R^0_{i,n} I[W^0_{i,n} \leq q_\alpha] - \mathbb{E}[R^0_{i,n} I[W^0_{i,n} \leq q_\alpha]]) -\frac{\mathbb{E}[R(0)-w^*R(0)|W=q_\alpha]}{\sqrt{n}}\sum_{i=1}^n (I[W^0_i \leq q_\alpha]-\alpha)\\
    &+ w^*\cdot \frac{1}{\sqrt{n}}\sum_{i=1}^n(R^0_{i,n}(0) I[W^0_{i,n} > q_\alpha] - \check{\mu}_0)\Bigg]
\end{align*}

As the two bracketed terms are independent, it suffices to show their asymptotic Normality separately and then to sum their variances.

As for the first term, define $\sigma^2_t := \var(R_{i,n}(1) I[W_{i,n} \leq q_\alpha]), \check{\sigma^2_0} := \var(R_{i,n}(0) I[W_{i,n} > q_\alpha])$, we obtain that:

\begin{align*}
    &\frac{1}{\sqrt{n}}\sum_{i=1}^n\begin{pmatrix}
        R_{i,n}(1)I[W_{i,n} \leq q_\alpha] - \mu_t\\
        w^*R_{i,n}(0)I[W_{i,n} > q_\alpha] - w^*\check{\mu}_0\\
        -\mathbb{E}[R(1)-w^*R(0)|W=q_\alpha](I[W_i \leq q_\alpha]-\alpha)
    \end{pmatrix} \\&\overset{d}{\rightarrow} \mathcal{N}\left(0, \Sigma\right)
\end{align*} where $\Sigma$ is

\[\begin{pmatrix}
    \sigma^2_t & -w^*\mu_t\check{\mu}_0 & -\mathbb{E}[R(1)-w^*R(0)|W=q_\alpha]\mu_t(1-\alpha)\\
    -w^*\mu_t\check{\mu}_0 & w^*{}^2\check{\sigma^2_0} & \alpha w^*\mathbb{E}[R(1)-w^*R(0)|W=q_\alpha]\check{\mu}_0\\
    -\mathbb{E}[R(1)-w^*R(0)|W=q_\alpha]\mu_t(1-\alpha) & \alpha w^*\mathbb{E}[R(1)-w^*R(0)|W=q_\alpha]\check{\mu}_0 & \mathbb{E}[R(1)-w^*R(0)|W=q_\alpha]^2\alpha(1-\alpha)
\end{pmatrix}\]

As for the second term, we obtain that 

\begin{align*}
    &\frac{1}{\sqrt{n}}\sum_{i=1}^n\begin{pmatrix}
        R^0_{i,n}I[W^0_{i,n} \leq q_\alpha] - \mu_c\\
        w^*R^0_{i,n}(0)I[W^0_{i,n} > q_\alpha] - w^*\check{\mu}_0\\
        -\mathbb{E}[R(0)-w^*R(0)|W=q_\alpha](I[W^0_i \leq q_\alpha]-\alpha)
    \end{pmatrix} \\&\overset{d}{\rightarrow} \mathcal{N}\left(0, \Sigma'\right)
\end{align*} where $\Sigma'$ is

\[\begin{pmatrix}
    \sigma^2_c & -w^*\mu_c\check{\mu}_0 & -\mathbb{E}[R(0)-w^*R(0)|W=q_\alpha]\mu_c(1-\alpha)\\
    -w^*\mu_c\check{\mu}_0 & w^*{}^2\check{\sigma^2_0} & \alpha w^*\mathbb{E}[R(0)-w^*R(0)|W=q_\alpha]\check{\mu}_0\\
    -\mathbb{E}[R(0)-w^*R(0)|W=q_\alpha]\mu_c(1-\alpha) & \alpha w^*\mathbb{E}[R(0)-w^*R(0)|W=q_\alpha]\check{\mu}_0 & \mathbb{E}[R(0)-w^*R(0)|W=q_\alpha]^2\alpha(1-\alpha)
\end{pmatrix}\]

The continuous mapping theorem (along with the independence that we observed earlier) then says that the asymptotic variance of our estimator is equal to the sum of each term in the above matrices which is
\begin{align*}
    &w^*{}^2\Bigg[2\alpha(1-\alpha)\mathbb{E}[R(0)|W=q_\alpha]^2 + 2\check{\sigma^2_0} - 4\alpha \mathbb{E}[R(0)|W=q_\alpha]\check{\mu}_0\Bigg]\\
    &+ w^*\Bigg[-2(\mu_t+\mu_c)\check{\mu}_0 + 2(\mu_t+\mu_c)\mathbb{E}[R(0)|W=q_\alpha](1-\alpha) + 2\alpha\check{\mu}_0\Big(\mathbb{E}[R(1)|W=q_\alpha] \\ &+ \mathbb{E}[R(0)|W=q_\alpha]\Big) - 2\alpha(1-\alpha)\mathbb{E}[R(0)|W=q_\alpha]\left(\mathbb{E}[R(1)|W=q_\alpha]+\mathbb{E}[R(0)|W=q_\alpha]\right)\Bigg]\\
    &+ \Bigg[\sigma^2_t + \sigma^2_c - 2(1-\alpha)\left(\mathbb{E}[R(1)|W=q_\alpha]\mu_t + \mathbb{E}[R(0)|W=q_\alpha]\mu_c\right) + \alpha(1-\alpha)\left(\mathbb{E}[R(1)|W=q_\alpha]^2 + \mathbb{E}[R(0)|W=q_\alpha]^2\right)\Bigg]
\end{align*}

Now, write \begin{align*}
    A &= \Bigg[2\alpha(1-\alpha)\mathbb{E}[R(0)|W=q_\alpha]^2 + 2\check{\sigma^2_0} - 4\alpha \mathbb{E}[R(0)|W=q_\alpha]\check{\mu}_0\Bigg]
\end{align*}
and \begin{align*}
    B &= \Bigg[-2(\mu_t+\mu_c)\check{\mu}_0 + 2(\mu_t+\mu_c)\mathbb{E}[R(0)|W=q_\alpha](1-\alpha) + 2\alpha\check{\mu}_0\Big(\mathbb{E}[R(1)|W=q_\alpha] \\ &+ \mathbb{E}[R(0)|W=q_\alpha]\Big) - 2\alpha(1-\alpha)\mathbb{E}[R(0)|W=q_\alpha]\left(\mathbb{E}[R(1)|W=q_\alpha]+\mathbb{E}[R(0)|W=q_\alpha]\right)\Bigg]
\end{align*}

So long as we can show that $A > 0$, it is quite clear that differentiating wrt $w^*$ and setting equal to zero, we see that the optimal choice of $w^*$ is then $\frac{-B}{2A}$. Furthermore, it is quite clear how to consistently estimate $w^*$ since we have shown how to consistently estimate each term in $A$ and $B$.

To see that $A > 0$, consider a coupling to our problem of random variables $(\tilde{W}_{i,n}, \tilde{R}_{i,n}(0), \tilde{R}_{i,n}(1))$ and $(\tilde{W}^0_{i,n}, \tilde{R}^0_{i,n}(0), \tilde{R}^0_{i,n}(1))$ where we set \[\tilde{W}_{i,n} = W_{i,n}, \tilde{W}^0_{i,n} = W^0_{i,n}, \tilde{R}_{i,n}(0) = R_{i,n}(0), \tilde{R}^0_{i,n}(0) = R_{i,n}^0(0), \tilde{R}^0_{i,n}(1) = R_{i,n}^0(1),\] but $\tilde{R}_{i,n}(1) \equiv 0$ (i.e., so everything is the same except that we set $\tilde{R}_{i,n}(1) \equiv 0$). It is clear that in this data-generating process, the covariance matrix for \[\begin{pmatrix}
        \tilde{R}_{i,n}(1)I[\tilde{W}_{i,n} \leq q_\alpha] \\
        w^*\tilde{R}_{i,n}(0)I[\tilde{W}_{i,n} > q_\alpha] - w^*\check{\mu}_0\\
        -\mathbb{E}[-w^*R(0)|W=q_\alpha](I[\tilde{W}_i \leq q_\alpha]-\alpha)
    \end{pmatrix}\] is $w^*{}^2\tilde{\Sigma}$
where $\tilde{\Sigma}$ is

\[\begin{pmatrix}
    0 & 0 & 0\\
    0& \check{\sigma^2_0} & \alpha \mathbb{E}[-R(0)|W=q_\alpha]\check{\mu}_0\\
    0 & \alpha \mathbb{E}[-R(0)|W=q_\alpha]\check{\mu}_0 & \mathbb{E}[-R(0)|W=q_\alpha]^2\alpha(1-\alpha)
\end{pmatrix}\]

Noticing that $A = 2 (0,1,1) \tilde{\Sigma} (0,1,1)^\top$ it suffices to show that the bottom right submatrix \[\begin{pmatrix}
\check{\sigma^2_0} & \alpha \mathbb{E}[-R(0)|W=q_\alpha]\check{\mu}_0\\
\alpha \mathbb{E}[-R(0)|W=q_\alpha]\check{\mu}_0 & \mathbb{E}[-R(0)|W=q_\alpha]^2\alpha(1-\alpha)
\end{pmatrix}\] is (strictly) positive definite (observe that, because it is a covariance matrix, it is automatically positive semi-definite). In case $\mathbb{E}[R(0)|W=q_\alpha] = 0$ it is already immediate that $A$ is strictly positive, since $\check{\sigma^2}_0$ is. Thus, consider the case when $\mathbb{E}[R(0)|W=q_\alpha] \neq 0$. Then the diagonal terms of the above matrix are non-zero and by computing the determinant and rearranging, we see that if the determinant is zero, then \[\text{Corr}\left(\tilde{R}_{i,n}(0)I[\tilde{W}_{i,n} > q_\alpha], -\mathbb{E}[-R(0)|W=q_\alpha]I[\tilde{W}_i \leq q_\alpha]\right) = \pm 1,\] which occurs if and only if $\tilde{R}_{i,n}(0)I[\tilde{W}_{i,n} > q_\alpha]$ and $-\mathbb{E}[-R(0)|W=q_\alpha]I[\tilde{W}_i \leq q_\alpha]$ are related in an affine manner, which is not the case per \Cref{as:positive-var-c-mean}. Therefore, the matrix is psd with non-zero determinant, hence positive-definite and $A$ is strictly positive as desired.

\paragraph{Variance estimation}
Define \[C = \Bigg[\sigma^2_t + \sigma^2_c - 2(1-\alpha)\left(\mathbb{E}[R(1)|W=q_\alpha]\mu_t + \mathbb{E}[R(0)|W=q_\alpha]\mu_c\right) + \alpha(1-\alpha)\left(\mathbb{E}[R(1)|W=q_\alpha]^2 + \mathbb{E}[R(0)|W=q_\alpha]^2\right)\Bigg].\] Then with the above choice of $w^*$, the asymptotic variance shown in the previous section is equal to \[\frac{B^2 - 4AC}{-4A}.\] Again, we have already shown how to estimate each term already and hence the variance can be consistently esimated.

\begin{remark}[Optimality]
    Notice that for the base estimator we have $\hat{w}_n \equiv w^* = 1$ and for the subgroup estimator, they have $\hat{w}_n \equiv w^* = 0$. This means that the above theory can be used to applied to those estimators and in particular, shows that our choice of $\hat{w}_n$ is asymptotically always at least as good (and will indeed will result in a \emph{strictly} smaller confidence interval (asymptotically) as compared to both of these approaches so long as $w^* \neq 0$ or $1$).
\end{remark}

\subsection{Putting it Together}
Summarizing and translating to the notation of the main body we arrive at the following. 
Recall that, for any consistent estimator $\hat{w}_n$ of some quantity $w$, we define the \emph{hybrid estimator} as  $$\theta^{\mathrm{hyb}}_{n,\alpha,\hat{w}}(\pi):=(1-\hat{w}_n)\cdot \theta^{\mathrm{SG}}_{n,\alpha}(\pi)+\hat{w}_n\cdot \theta^{\mathrm{base}}_{n,\alpha}(\pi).$$

Let (where $\mathbb{E}$ is taken over $(\mathbf{x},R)\sim P$):
\begin{align*}
    A &= \Bigg[2\alpha(1-\alpha)\mathbb{E}[R(0)|\Upsilon(\mathbf{x})=q_\alpha]^2 + 2\check{\sigma^2_0} - 4\alpha \mathbb{E}[R(0)|\Upsilon(\mathbf{x})=q_\alpha]\check{\mu}_0\Bigg]
\end{align*}
and \begin{align*}
    B &= \Bigg[-2(\mu_t+\mu_c)\check{\mu}_0 + 2(\mu_t+\mu_c)\mathbb{E}[R(0)|\Upsilon(\mathbf{x})=q_\alpha](1-\alpha) + 2\alpha\check{\mu}_0\Big(\mathbb{E}[R(1)|\Upsilon(\mathbf{x})=q_\alpha] \\ &+ \mathbb{E}[R(0)|\Upsilon(\mathbf{x})=q_\alpha]\Big) - 2\alpha(1-\alpha)\mathbb{E}[R(0)|\Upsilon(\mathbf{x})=q_\alpha]\left(\mathbb{E}[R(1)|\Upsilon(\mathbf{x})=q_\alpha]+\mathbb{E}[R(0)|\Upsilon(\mathbf{x})=q_\alpha]\right)\Bigg]
\end{align*} and 
\[C = \Bigg[\sigma^2_t + \sigma^2_c - 2(1-\alpha)\left(\mathbb{E}[R(1)|\Upsilon(\mathbf{x})=q_\alpha]\mu_t + \mathbb{E}[R(0)|\Upsilon(\mathbf{x})=q_\alpha]\mu_c\right) + \alpha(1-\alpha)\left(\mathbb{E}[R(1)|\Upsilon(\mathbf{x})=q_\alpha]^2 + \mathbb{E}[R(0)|\Upsilon(\mathbf{x})=q_\alpha]^2\right)\Bigg],\] 
with $\mu_{i}=\mathbb{E}[R(i))I[\Upsilon(\mathbf{x})\leq q_{\alpha}]$, $\check{\mu_{i}}=\mathbb{E}[R(i))I[\Upsilon(\mathbf{x})> q_{\alpha}]$,  $\sigma^2_i=\mathrm{Var}[R(i)I[\Upsilon(\mathbf{x})\leq q_\alpha]]$ and $\check{\sigma^2_i}=\mathrm{Var}[R(i)I[\Upsilon(\mathbf{x})> q_\alpha]]$ for $i\in \{0,1\}$
\begin{theorem}\label{context3} 
   Under \Cref{second-moment} for $Z=R(0)$ and $Z=R(1)$ and \Cref{as:positive-derivative} for $\Upsilon(\mathbf{x})$ as well as \Cref{regularity}, for any sequence $\hat{w}_n \overset{p}{\rightarrow} w$, we get:  
   $$\sqrt{n}\left(\theta^{\mathrm{hyb}}_{n,\alpha,\hat{w}}(\pi)-\tau^{\mathrm{new}}_{n,\alpha}(\pi)\right)
   \overset{d}{\rightarrow} \mathcal{N}(0, \sigma^2_{\mathrm{hyb}(w)})$$ where 
   \begin{equation}
       \sigma^2_{\mathrm{hyb}(w)}   = \frac{1}{\alpha^2}(w^2A+wB+C)
   \end{equation} 
   We have that $w^*:={\frac{-B}{2A}=\argmin_{w\in \mathbb{R}} \sigma^2_{\mathrm{hyb}(w)}}$. 
   If we additionally, assume that:
    \begin{enumerate}
        \item The functions $w \mapsto \mathbb{E}[R(1)|\Upsilon(\mathbf{x})=w]$ and $w \mapsto \mathbb{E}[R(0)|\Upsilon(\mathbf{x})=w]$ are continuous in a closed neighborhood of $q_\alpha$.
        \item We have that $\var[R(1)|\Upsilon(\mathbf{x})=w]$ and $\var[R(0)|\Upsilon(\mathbf{x})=w]$ are bounded for all $w$ in these neighborhoods.
        \item  $\frac{k}{\sqrt{n}\log n} \rightarrow \infty$ with $k/n \rightarrow 0$,
    \end{enumerate} using the consistent estimators for each term of $A$, $B$, $C$ from \Cref{app:Base2}, we can construct a sequence $\hat{w}^*_n$ for which $\hat{w}^*_n \overset{p}{\rightarrow} w^*$ and also we can derive a consistent estimate $\hat{\sigma}^2_{\mathrm{hyb}(w^*)}$ of $\sigma^2_{\mathrm{hyb}(w^*)}$
\end{theorem}

\end{document}